\documentclass[letterpaper]{article} 
\usepackage{arxiv}
\usepackage[hyphens]{url}  
\usepackage{graphicx} 
\usepackage{natbib}  
\usepackage{caption} 
\usepackage{amsmath,amssymb,amsthm,mathtools}
\usepackage{booktabs}
\usepackage{array}
\usepackage{algorithm}
\usepackage{algorithmic}
\usepackage{enumitem}
\usepackage{alphalph}

\graphicspath{{figures/}}

\newcommand{\placeholdergraphic}[2][]{%
  \IfFileExists{figures/#2.pdf}{\includegraphics[#1]{#2}}{%
  \IfFileExists{figures/#2.png}{\includegraphics[#1]{#2}}{%
    \fbox{\begin{minipage}{0.90\linewidth}\centering\vspace{0.4in}
    Placeholder for figure \texttt{\detokenize{#2}}\\
    Add the corresponding file under \texttt{figures/} for the camera-ready build.
    \vspace{0.4in}\end{minipage}}%
  }}%
}

\newtheorem{theorem}{Theorem}
\newtheorem{lemma}[theorem]{Lemma}
\newtheorem{corollary}[theorem]{Corollary}
\newtheorem{proposition}[theorem]{Proposition}
\newtheorem{remark}[theorem]{Remark}
\newtheorem{definition}{Definition}

\newcommand{\E}{\mathbb{E}}
\newcommand{\Prob}{\mathbb{P}}
\newcommand{\Reg}{\mathrm{Reg}}
\newcommand{\OPT}{\mathrm{OPT}}
\newcommand{\R}{\mathbb{R}}
\newcommand{\calB}{\mathcal{B}}
\newcommand{\calE}{\mathcal{E}}
\newcommand{\calG}{\mathcal{G}}
\DeclareMathOperator*{\argmax}{arg\,max}

\title{Discrepancy-Rounded Fair Bandits with Static and Time-Varying Exposure Floors}
\author {
    Ibne Farabi Shihab\textsuperscript{\rm 1}\equalcontrib\corresponding,
    Joyanta Jyoti Mondal\textsuperscript{\rm 2}\equalcontrib,
    Anuj Sharma\textsuperscript{\rm 3}
}
\affiliations {
    \textsuperscript{\rm 1}Department of Computer Science, Iowa State University, USA\\
    \textsuperscript{\rm 2}Department of Computer and Information Sciences, University of Delaware, USA\\
    \textsuperscript{\rm 3}Department of Civil, Construction \& Environmental Engineering, Iowa State University, USA \\
    joyanta@udel.edu, \{ishihab, anujs\}@iastate.edu
}
\date{}

\begin{document}
\maketitle

\begin{abstract}
Minimum-exposure constraints arise in recommendation, content curation, and regulated allocation when each provider, arm, or group must receive guaranteed exposure inside a period rather than only in aggregate. We study stochastic bandits with exact exposure floors and show that the right object is a rounding problem: a fractional fair schedule is realized as integral pulls, and the exposure error is exactly a discrepancy vector. The main contribution is a blockwise model with time-varying floors. BDQ-UCB satisfies every block floor deterministically and has fair regret governed by the nonmandatory budget \(R\), not the horizon \(T\), with high-probability regret \(O(\sqrt{KR\log(KT)})\). A MOSS residual variant attains \(O(\sqrt{KR})\), and a matching lower bound gives the minimax rate \(\Theta(\sqrt{KR})\), even with positive mandatory exposure; a kl-UCB\(^{++}\) residual rule adds instance-dependent optimality. The formulation becomes essential for overlapping group floors: per-arm rounding can violate a group constraint by \(\Omega(s)\) in the group size, whereas Beck--Fiala null-space rounding meets every group floor within the block budget with violation below the arm degree \(t\), and composes with UCB at the same \(R\)-parametrized regret. For learned group plans, we close disjoint systems at \(\widetilde\Theta(\sqrt{KT})\), give a dual-ledger decomposition explaining why naive index rules fail under overlap, and prove a plan-sampling rule that is pathwise feasible under an initial cover-slack condition and attains a conditional \(\widetilde O(\sqrt{KT})\) guarantee, leaving the condition-free overlap rate open. Experiments on synthetic floors, MovieLens-100k genre exposure, and deployment stress tests show exact feasibility without penalty tuning and regret competitive with tuned Lagrangian baselines.
\end{abstract}

\section{Introduction}

Fairness-constrained bandits arise when arms are providers, sellers, content sources, treatments, or protected groups that must receive minimum exposure. A classical stochastic bandit concentrates nearly all pulls on the empirically best arm, but in many allocation systems even a lower-reward arm is contractually, legally, or ethically entitled to a minimum number of opportunities.

For a fixed global exposure floor there is a simple, well-studied solution: give every arm its required pulls, then run a standard bandit algorithm, as Fair-MAB \citep{patil2020fair,patil2021fair} formalizes. Our starting point is different: we read the fairness layer as a \emph{rounding} layer, where the learner realizes a fractional exposure plan through integral pulls and the resulting exposure gap is exactly a discrepancy vector. This is almost trivial for a single global floor, but becomes structurally useful once floors vary over time and genuinely necessary once they overlap, when a single pull can credit several constraints at once. Figure~\ref{fig:pipeline} shows the pipeline: within each period, round a fractional fair plan to integral pulls that meet the floor exactly, then spend the remaining rounds learning reward. The discrepancy rounding leaves behind determines feasibility; the number of nonmandatory rounds determines regret.

\begin{figure}[t]
\centering
\placeholdergraphic[width=\columnwidth]{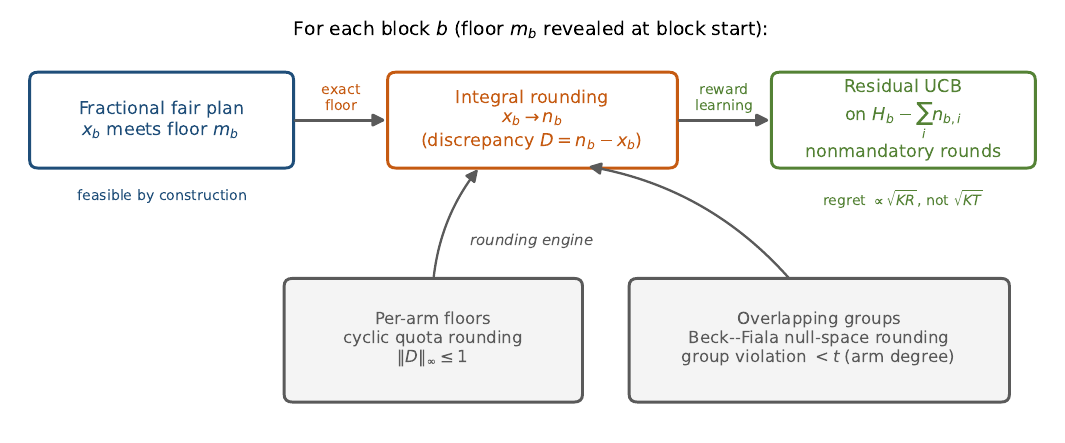}
\caption{The discrepancy-rounding view of fair bandits. Each period rounds a fractional fair plan to integral pulls (the exposure gap is a discrepancy vector) and spends the remaining rounds on residual UCB. Per-arm floors need only cyclic quota rounding; overlapping group floors need genuine set-discrepancy rounding, where the reduction becomes the algorithm itself rather than merely descriptive language.}
\label{fig:pipeline}
\end{figure}

This lens organizes the paper. Our central object is the \emph{blockwise} model, in which the horizon is split into periods, each carrying its own integer floor vector that must be satisfied inside that period rather than only in aggregate. The blockwise algorithm BDQ-UCB (Blockwise Discrepancy-Quota Upper Confidence Bound) executes each block's quota exactly and runs UCB only on the rounds the floor leaves free, so its fair regret is governed by the total nonmandatory budget
\[
    R=\sum_{b=1}^B\Bigl(H_b-\sum_i m_{b,i}\Bigr),
\]
not by the horizon \(T\). This is the right complexity parameter in a strong sense: a matching lower bound and a MOSS-based residual rule pin the minimax rate at \(\Theta(\sqrt{KR})\), even when the mandatory exposure is a constant fraction of the horizon, and a kl-UCB\(^{++}\) residual rule recovers the exact instance-dependent constant. The mandatory pulls, being common to every fair policy, are free; all unavoidable exploration lives in the residual budget. Algorithm names are systematic: the suffix is the residual index rule (UCB, MOSS, or kl-UCB\(^{++}\)); DQ-UCB is the static base, the B prefix marks the blockwise version, and the Group-, D-, OG-, and P- prefixes mark the fixed-, disjoint-, one-shot-, and sampled-plan group algorithms.

The reduction becomes necessary when floors are placed over overlapping groups. Rounding a group-feasible fractional plan with a per-arm rule can miss a group floor by \(\Omega(s)\) in the group size, whereas Beck--Fiala null-space rounding, a genuine discrepancy algorithm, meets every group floor within the reserved block budget using group slack strictly below the arm degree \(t\), independent of the group size and the number of groups. Composed with UCB, this yields Group-BDQ-UCB at the same residual-budget regret for a fixed group plan. The final question is whether the plan itself can be learned, so the benchmark becomes the best group-fair plan rather than a given one. Here the answer is sharp for disjoint groups, where a per-pull covering rule closes the price of plan adaptivity at \(\widetilde\Theta(\sqrt{KT})\) against a matching lower bound, and partial under arbitrary overlap, where linear-programming duality identifies the exact obstruction and a plan-sampling rule attains a conditional \(\widetilde O(\sqrt{KT})\) guarantee whose one open hypothesis we state explicitly. The full complexity picture across settings is summarized in Table~\ref{tab:glance} (Appendix~\ref{app:results-summary}).

We calibrate the claims explicitly. The static result is a rederivation of fixed-floor exposure fairness, not a new rate over Fair-MAB; the modeling contribution is the blockwise floor model, whose floors are period-specific, revealed before each period, and satisfied inside it. The separation result rules out final-count and block-independent surrogates but not deadline-aware dynamic penalties, which we include as baselines. The framework's value is exact feasibility without penalty calibration and regret governed by the nonmandatory budget \(R\).
\section{Related Work}

The closest prior work is the Fair-MAB framework \citep{patil2020fair,patil2021fair}, which requires every arm to receive a prescribed fraction of the pulls at every round up to an additive tolerance and measures regret against a fairness-aware comparator. Our static result is a discrepancy-based rederivation of this guarantee rather than a rate improvement, and because the quota is front-loaded, DQ-UCB meets the same anytime floor with tolerance one (Remark~\ref{rem:anytime}), but its purpose is to make the rounding layer explicit and carry it to the blockwise time-varying floors Fair-MAB does not capture. Other fairness criteria for bandits are complementary rather than directly comparable: meritocracy \citep{joseph2016fairness}, merit-based exposure \citep{wang2021fairness}, Nash social welfare \citep{barman2022fairness}, max-min fairness \citep{harada2025bandit}, fairness under uncertainty \citep{lee2026fairness}, and the fairness--performance frontier \citep{wilms2026fairness}, whose Pareto viewpoint parallels our closed-form analysis.

Nearer to exact exposure are works with explicit group or contextual constraints (fair contextual bandits \citep{chen2020fair}, bilevel group exposure with within-group merit \citep{pokhriyal2024bilevel}, hypergeometric finite-pool ranking floors \citep{cartier2025hyperfair}, and graph-structured multi-regularizer fairness \citep{zhou2025online}) and works that relax the single stationary learner through interacting agents \citep{manupriya2025multiagent,xu2025fair,krishnamurthy2026creator} or slow non-stationarity \citep{shaarad2020regret}, the latter adjacent to our blockwise model. A separate line enforces constraints through budgets, queues, or duality (bandits with knapsacks \citep{badanidiyuru2013bandits}, concave-reward convex-constraint extensions \citep{agrawal2014bandits}, virtual-queue fairness in combinatorial sleeping bandits \citep{li2019combinatorial}, modified-index minimum-rate fairness \citep{claure2020multi}, and dual mirror descent for online allocation \citep{balseiro2020dual}) but these certify only asymptotic or average feasibility, whereas exact per-period feasibility is what our model makes primitive; a deadline-aware Lagrangian represents this penalty-based family in our experiments. None of these enforces time-varying block-level floors deterministically without penalty tuning; the combination of exact blockwise feasibility, a tuning-free design, and a matching minimax lower bound is what is specific to this work.

The analysis draws on standard bandit machinery: finite-time upper confidence bound (UCB) \citep{auer2002finite,han2024ucb}; the minimax construction via the Kullback--Leibler (KL) chain rule and Pinsker's inequality \citep{bubeck2012regret,lattimore2020bandit}; its Bretagnolle--Huber form for the instance-dependent lower bound; and the MOSS \citep{audibert2009minimax} and kl-UCB\(^{++}\) \citep{menard2017minimax} indices behind our rate-optimal residual rules. Allocation under time-varying demand is complementary (floors there are unknown, ours are revealed per block \citep{lyu2023online}) and multi-objective and preference/risk mixtures \citep{davoodi2025stochastic,tatli2025preference,tatli2025risk} share the mixture structure of our Pareto curve. Finally, the rounding engine is classical discrepancy theory (Spencer's theorem \citep{spencer1985six}, Banaszczyk balancing \citep{banaszczyk1998balancing}, Bansal's constructive method \citep{bansal2010constructive}, dependent rounding \citep{gandhi2006dependent}, matroid-friendly rounding \citep{bansal2016approximation}, and online vector balancing \citep{bansal2020online,altschuler2025threshold,bednorz2024gramschmidt}), from which we use the Beck--Fiala theorem as the central tool for overlapping floors and to which we connect the harder online results in the open adaptive-planning question.

\section{Problem Setup and the Discrepancy Reduction}
\label{sec:model}

We study a stochastic \(K\)-armed bandit with horizon \(T\). Arm \(i\) produces independent rewards in \([0,1]\) with mean \(\mu_i\). Let
\[
    i^\star \in \argmax_{i\in[K]} \mu_i,
    \qquad
    \mu_\star = \mu_{i^\star},
    \qquad
    \Delta_i = \mu_\star-\mu_i,
\]
so that \(\Delta_{i^\star}=0\) and \(\Delta_i\ge 0\) for every arm \(i\). Notation is summarized in Table~\ref{tab:notation}, and all proofs are deferred to the technical appendix.

For a target exposure fraction \(\delta\in[0,1/K]\), the integral floor is \(m=\lfloor\delta T\rfloor\), and a policy is \(m\)-fair if its final pull counts satisfy \(N_i(T)\ge m\) for every arm \(i\in[K]\) (equivalently, the empirical exposure of every arm is at least \(m/T\ge\delta-1/T\)). The best \(m\)-fair allocation, the induced fair pseudo-regret, and the fair-regret identity \(\widehat\Reg_m(T)=\sum_{i\neq i^\star}\Delta_i\bigl(N_i(T)-m\bigr)\), which shows fair regret is a gap-weighted count of pulls beyond the quota, are the static (\(B=1\)) specialization of the blockwise objects of the next section; we defer their formal statements to Appendix~\ref{app:static} (Lemma~\ref{lem:static-id}). This gap-weighted count is what the discrepancy view will make actionable.

\subsection{The Exposure--Discrepancy Identity}
\label{sec:reduce}

We first record that the difference between a fractional fair schedule and an integral pull sequence is exactly a discrepancy vector. A fractional allocation at time \(t\) is a vector \(x_t\in\R^K\) with \(x_{t,i}\ge0\) and \(\sum_i x_{t,i}=1\), the intended cumulative fractional exposure of arm \(i\) is \(S_i(T)=\sum_{t=1}^T x_{t,i}\), and a deterministic pull \(A_t\) corresponds to the standard basis vector \(e_{A_t}\). Define the discrepancy vector
\[
    D_T
    =
    \sum_{t=1}^T \bigl(e_{A_t}-x_t\bigr)
    =
    N(T)-S(T),
\]
where \(N(T)=(N_1(T),\dots,N_K(T))\) and \(S(T)=(S_1(T),\dots,S_K(T))\).

\begin{proposition}[Exposure--discrepancy identity]
\label{thm:reduce}
Fix any fractional schedule \(x_1,\dots,x_T\). For any pull sequence \(A_1,\dots,A_T\), the following hold.
\begin{enumerate}[label=(\roman*),leftmargin=*]
\item The exposure error of each arm is exactly the corresponding coordinate of the discrepancy vector:
\[
    N_i(T)-S_i(T)=D_{T,i}.
\]
Hence, if \(S_i(T)\ge m+B\) for every \(i\) and \(\|D_T\|_\infty\le B\), then the integral pull sequence is \(m\)-fair.
\item The reward gap relative to the fractional schedule is a weighted discrepancy:
\[
    \sum_{t=1}^T \mu^\top x_t
    -
    \sum_{t=1}^T \mu_{A_t}
    =
    -\mu^\top D_T.
\]
\item Therefore, minimizing reward loss subject to a fairness floor is a constrained discrepancy-minimization problem: one chooses one vector from \(\{e_1-x_t,\dots,e_K-x_t\}\) at each time \(t\), controls the terminal \(\ell_\infty\) discrepancy for fairness, and controls the weighted discrepancy \(-\mu^\top D_T\) for regret.
\end{enumerate}
\end{proposition}

The identity is elementary, its proof is telescoping plus linearity, and we deliberately label it a proposition rather than a theorem. Its value is organizational: it fixes the objects the algorithms manipulate, and it is what will later let a genuine discrepancy algorithm, Beck--Fiala rounding, solve a fairness problem that per-arm rounding provably cannot. In the static algorithm in Appendix~\ref{app:static} we use a simple constructive rounding schedule for the mandatory quota part; in the blockwise extension the same rounding idea is applied separately to each block.

\section{Blockwise Exposure Floors}
\label{sec:blockwise}

The main extension is blockwise fairness. The horizon is partitioned into \(B\) consecutive blocks \(\calB_1,\ldots,\calB_B\), where block \(b\) has length \(H_b\), and each block has its own integer floor vector
\[
    m_b=(m_{b,1},\ldots,m_{b,K})\in\mathbb{Z}_{\ge0}^K,
    \qquad
    \sum_{i=1}^K m_{b,i}\le H_b.
\]
A policy is blockwise fair if \(N_{b,i}\ge m_{b,i}\) for every block \(b\) and arm \(i\), where \(N_{b,i}\) is the number of pulls of arm \(i\) inside block \(b\). This model captures time-varying exposure contracts, rotating provider guarantees, periodic protected-group targets, or demand-dependent exposure requirements, and it cannot be represented by a single global minimum fraction.

\begin{definition}[Best blockwise-fair allocation]
The best blockwise-fair allocation gives the floor \(m_{b,i}\) to every arm in each block and allocates the residual \(H_b-\sum_j m_{b,j}\) pulls of that block to a best arm. Its value is
\[
    \OPT_{\mathbf m}(T)=T\mu_\star-
    \sum_{b=1}^B\sum_{i\neq i^\star}m_{b,i}\Delta_i,
\]
and the realized blockwise fair pseudo-regret is
\(\widehat\Reg_{\mathbf m}(T)=\OPT_{\mathbf m}(T)-\sum_{t=1}^T\mu_{A_t}\).
\end{definition}

\begin{lemma}[Blockwise fair-regret identity]
\label{lem:block-id}
For every pull sequence,
\[
    \widehat\Reg_{\mathbf m}(T)
    =
    \sum_{i\neq i^\star}\Delta_i
    \left(N_i(T)-\sum_{b=1}^B m_{b,i}\right).
\]
If the sequence is blockwise fair, then \(\widehat\Reg_{\mathbf m}(T)\ge0\).
\end{lemma}

\subsection{BDQ-UCB Algorithm and Guarantees}

BDQ-UCB applies discrepancy rounding separately in each block. At the start of block \(b\) it executes any deterministic schedule containing exactly \(m_{b,i}\) copies of arm \(i\), realizing the block floor with terminal discrepancy zero relative to the block target, and then runs UCB for the remaining \(H_b-\sum_i m_{b,i}\) rounds using all observations so far. We use the standard optimistic convention that an arm with zero observations has UCB index \(+\infty\) during residual rounds (equivalently, unseen arms are sampled before the empirical-mean index is used); this is needed because early blocks may have \(m_{b,i}=0\) for some arm, and the additive \(K\) term in the regret bounds is its initialization cost. Pseudocode is Algorithm~\ref{alg:bdqucb}.

\begin{algorithm}
\caption{BDQ-UCB}
\label{alg:bdqucb}
\begin{algorithmic}[1]
\STATE \textbf{Input:} arms $K$, blocks $B$, block lengths $(H_b)$, floors $(m_{b,i})$, failure prob.\ $\eta\in(0,1)$
\STATE \textbf{Output:} pull sequence across all blocks
\STATE $L \leftarrow \log(2KT/\eta)$, where $T = \sum_b H_b$
\FOR{$b = 1,\dots,B$}
    \STATE Pull each arm $i\in[K]$ exactly $m_{b,i}$ times (any fixed order; count determines fairness)
    \FOR{each residual round in block $b$}
        \STATE For each arm $i$, set $U_i\leftarrow +\infty$ if $N_i=0$ and $U_i\leftarrow\widehat{\mu}_i+\sqrt{2L/N_i}$ otherwise.
        \STATE Pull $A \in \argmax_{i\in[K]} U_i$, \quad $N_i$ = total pulls of arm $i$ so far
    \ENDFOR
\ENDFOR
\end{algorithmic}
\end{algorithm}

Let
\[
    R=\sum_{b=1}^B\left(H_b-\sum_{i=1}^K m_{b,i}\right)
\]
be the total number of nonmandatory rounds.

\begin{theorem}[Blockwise exact fairness and regret]
\label{thm:blockwise}
Assume rewards are independent and supported in \([0,1]\). BDQ-UCB satisfies:
\begin{enumerate}[label=(\roman*),leftmargin=*]
\item Deterministically, \(N_{b,i}\ge m_{b,i}\) for every block \(b\) and arm \(i\).
\item With probability at least \(1-\eta\),
\begin{align}
      \widehat\Reg_{\mathbf m}(T)
      &\le
      K+
      \sum_{i:\Delta_i>0}
      \min\left\{R\Delta_i,\,\frac{8L}{\Delta_i}\right\}, \notag\\
      &\qquad L=\log\left(\frac{2KT}{\eta}\right).
  \end{align}
\item With probability at least \(1-\eta\),
\[
    \widehat\Reg_{\mathbf m}(T)
    \le K+4\sqrt{2KRL}.
\]
\end{enumerate}
\end{theorem}

Taking \(B=1\), \(H_1=T\), and \(m_{1,i}=m\) for every arm recovers the static single-floor guarantee (Theorem~\ref{thm:main}, Appendix~\ref{app:static}) as the special case of Theorem~\ref{thm:blockwise}, up to replacing \(R=T-Km\) by the looser bound \(T\) (Corollary~\ref{cor:static-special}, Appendix~\ref{app:static-special}).

BDQ-UCB extends unchanged to the setting where \(m_b\) is revealed only at the start of block \(b\), without knowledge of future floors. Because feasibility is block-local and the UCB concentration event does not depend on future floors, all realized block constraints hold and the regret bounds of Theorem~\ref{thm:blockwise} apply conditional on the realized floor sequence; Appendix~\ref{app:online-floors-proof} makes this precise.

\subsection{Lower Bounds and the Optimal Rate}
\label{sec:block-lower}

This section pins the complexity of blockwise fair bandits from both sides and in both regimes: the residual budget \(R\), not the total horizon \(T\), is the right parameter, in the minimax sense and in the instance-dependent sense. Mandatory pulls are matched by the comparator, so the unavoidable exploration cost lives in the nonmandatory rounds.

\begin{theorem}[Minimax lower bound]
\label{thm:block-lower}
There is a universal constant \(c>0\) such that, for every \(K\ge2\) and every residual budget \(R\ge K\), there exists a blockwise-fair instance with total nonmandatory budget \(R\) for which every blockwise-fair policy \(\pi\) satisfies
\[
    \sup_{\mu\in[0,1]^K}\Reg_{\mathbf m}^{\pi}(T)
    \ge
    c\sqrt{KR}.
\]
Consequently, BDQ-UCB is minimax optimal for blockwise fair regret up to logarithmic factors.
\end{theorem}

The logarithmic gap between Theorems~\ref{thm:blockwise} and~\ref{thm:block-lower} is removable: a MOSS index \citep{audibert2009minimax} on residual observations gives BDQ-MOSS with expected blockwise fair regret at most \(C\sqrt{KR}+K\), so the minimax rate is \(\Theta(\sqrt{KR})\) for \(R\ge K\) (Theorem~\ref{thm:bdqmoss}). BDQ-MOSS discards mandatory observations, which makes the reduction black-box; the UCB variant uses every observation, carries the gap-dependent guarantee of Theorem~\ref{thm:blockwise}(ii), and is the one we run. The instance-dependent characterization, the positive-mandatory-exposure lower bound, and the kl-UCB\(^{++}\) variant are in Appendix~\ref{app:blockwise-extra}.

The blockwise model is also strictly more expressive than any single global fraction: a global final-count constraint fixes total exposure but not \emph{when} an arm receives it, and a block-independent per-block surrogate enforces timing only by over-serving blocks where no exposure is due. Proposition~\ref{prop:separation} makes this precise on a two-arm rotating-floor family: no global final count encodes the constraints, and any block-independent bound either violates a floor or pays \(\Omega(\alpha\Delta T)\) extra regret, while BDQ-UCB is exactly feasible at \(O(\sqrt{KR\log KT})\). The separation is scoped to those two surrogate classes; deadline-aware dynamic penalties escape it once tuned, which is why our experiments carry a deadline-aware Lagrangian as a calibration baseline.

\section{Overlapping Group Floors and Set-Discrepancy Rounding}
\label{sec:groupfloor}

The constructions so far use per-arm floors, which cyclic quota rounding (Lemma~\ref{lem:balanced}) already satisfies exactly. We now turn to the setting that justifies the discrepancy view as more than a vocabulary: \emph{overlapping} group floors, where a single pull credits several groups at once. Let \(\calG\) be a collection of arm subsets (groups), and require that within each block every group \(g\in\calG\) receive at least \(f_g\) pulls in aggregate, that is, \(\sum_{i\in g}N_{b,i}\ge f_g\). Such constraints arise when a provider belongs to several protected categories simultaneously (a film is both independent and foreign; a seller sits in several promotional tiers). Each arm belongs to at most \(t\) groups, the maximum \emph{arm degree}.

Per-arm rounding is no longer enough here. Rounding a group-feasible fractional plan \(x\) to integral pulls incurs a group exposure error \(\bigl|\sum_{i\in g}(n_i-x_i)\bigr|=\bigl|\sum_{i\in g}D_i\bigr|\), which is exactly a signed-sum set discrepancy of the per-arm discrepancy vector \(D=n-x\) over the set system \(\calG\). Bounding group violation is therefore a set-discrepancy problem, and this is where the identity of Proposition~\ref{thm:reduce} becomes essential, because the classical Beck--Fiala theorem \citep{beckfiala1981} bounds exactly this quantity. The failure of group-blind rounding is not hypothetical.

\begin{proposition}[Naive rounding fails on group floors]
\label{prop:naive-group}
There is a group-floor instance with groups of size \(s\) on which independent per-arm nearest-integer rounding (the group-blind analogue of cyclic quota rounding) violates some group floor by \(\Omega(s)\).
\end{proposition}

One might hope to escape by rounding every coordinate up instead, but that inflates the mandatory phase by up to one pull per arm and can exceed the rounds reserved for it; the real problem is meeting every group floor \emph{within the block budget}, and that is what the discrepancy algorithm delivers.

\begin{theorem}[Group-fair rounding via Beck--Fiala]
\label{thm:groupbf}
Let every arm belong to at most \(t\) groups of \(\calG\), and let \(x\in\R_{\ge0}^K\) be a fractional block plan whose group totals satisfy \(\sum_{i\in g}x_i\ge f_g+t\) for every \(g\in\calG\) and whose ceiling budget satisfies \(\sum_i\lceil x_i\rceil\le C\), where \(C\) is the number of rounds reserved for the mandatory phase of the block. Beck--Fiala null-space rounding outputs, in polynomial time, an integral allocation \(n\) with \(n_i\in\{\lfloor x_i\rfloor,\lceil x_i\rceil\}\) for every arm (hence \(\sum_i\lfloor x_i\rfloor\le\sum_i n_i\le\sum_i\lceil x_i\rceil\le C\)) and with
\[
    \Bigl|\sum_{i\in g}(n_i-x_i)\Bigr| < t
    \qquad\text{for every } g\in\calG,
\]
independent of the group size \(s\) and of \(|\calG|\). Consequently every group floor \(f_g\) is met within the reserved budget.
\end{theorem}

The rounding composes with learning exactly as in the per-arm case, giving a group-fair bandit policy rather than a rounding statement alone.

\begin{corollary}[Group-BDQ-UCB]
\label{cor:groupbdq}
In each block \(b\), let \(x_b\) be any group-feasible fractional plan with slack \(t\) whose ceiling budget satisfies \(\sum_i\lceil x_{b,i}\rceil\le H_b\), let \(n_b\) be its Beck--Fiala rounding, execute the \(n_{b,i}\) mandatory pulls of each arm, and run UCB on the remaining \(R_b=H_b-\sum_i n_{b,i}\) rounds of the block. Then every group floor is met in every block deterministically, and with probability at least \(1-\eta\) the fair regret relative to the comparator that executes the same mandatory allocations \((n_b)_{b\le B}\) and assigns every residual round to a best arm is at most \(K+4\sqrt{2KRL}\), with \(R=\sum_b R_b\) and \(L=\log(2KT/\eta)\).
\end{corollary}

The proof, in the appendix, is the residual argument of Theorem~\ref{thm:blockwise} verbatim: the mandatory pulls are comparator-matched by construction, so only residual pulls of suboptimal arms contribute. Relative to the fractional plan itself, the mandatory phase additionally changes reward by the weighted discrepancy \(-\mu^\top(n_b-x_b)\), which the per-coordinate containment bounds by the number of fractional coordinates in the block; this is the unavoidable price of integrality, not a learning cost. Replacing the residual rule by the MOSS index as in Theorem~\ref{thm:bdqmoss} yields Group-BDQ-MOSS with expected group-fair regret \(C\sqrt{KR}+K\) by the same embedded-game reduction.

The contrast between the two rounding rules is sharp, and it is what makes the discrepancy bridge essential rather than cosmetic. On a row/column set system over an \(a\times b\) grid (every arm in \(t=2\) groups, group size \(s=\max(a,b)\)) with the adversarial half-integral plan of Proposition~\ref{prop:naive-group}, Beck--Fiala group violation stays below \(t=2\) at every scale while naive nearest rounding grows linearly in \(s\), exactly as the construction predicts (Table~\ref{tab:groupdisc}, Appendix~\ref{app:experiments-to-add}). When \(t\ll s\) the discrepancy algorithm is provably and unboundedly better than per-arm rounding.

\section{Learning the Group Plan}
\label{sec:learned-plan}

So far the fractional plan has been an input: Corollary~\ref{cor:groupbdq} guarantees regret only against the comparator that executes the \emph{same} plan. Beyond off-the-shelf rounding lies the harder question of whether the plan itself can be learned, so that the benchmark becomes the best group-fair plan rather than the given one. We answer with an optimistic planning layer around the same Beck--Fiala engine, a regret guarantee against the per-block fractional optimum, and a lower bound showing that this stronger benchmark carries an unavoidable new cost.

We answer with OG-BDQ-UCB, an optimistic planner that in each block solves a linear program over a slacked plan polytope (group floors raised by the arm degree \(t\), budget reduced by \(2K\) for rounding and initialization), rounds the solution with the same Beck--Fiala engine, and spends the reserve on the residual UCB rule. Against \(\OPT^{\mathrm{ad}}\), the best slack-feasible per-block plan with its reserve on a best arm, it is exactly feasible and satisfies \(\OPT^{\mathrm{ad}}-V\le 3KB+4H_{\max}\sqrt{2LB}\) with probability at least \(1-\eta\) (Theorem~\ref{thm:ogbdq}, Appendix~\ref{app:learned-plan}); the built-in slack costs only \(O(B)\) against the unslacked optimum under uniform Slater-type margins (Proposition~\ref{prop:slater-gap}). More interesting is that this stronger benchmark is genuinely more expensive: in every \(R\)-parametrized result above, forced pulls cancel because the comparator matches them, but when the comparator instead places the forced mass optimally \emph{within} each group, forcing exposure onto arms whose ordering is still unresolved is itself an exploration cost.

\begin{proposition}[The price of plan adaptivity]
\label{prop:adaptivity-lower}
For every \(B\ge1\) and every even \(H\ge2\) there is a family of blockwise group-floor instances (\(K=2B+1\) arms, \(B\) blocks of length \(H\), \emph{disjoint} groups of size two (arm degree \(t=1\)), group floor \(H/2\) on one fresh group per block) on which every blockwise-group-fair policy \(\pi\) satisfies
\[
    \sup_{\mathrm{family}}\ \bigl[\OPT^{\mathrm{frac}}-\E[V^\pi]\bigr]\ \ge\ c\,B\sqrt H
\]
for a universal constant \(c>0\), where \(\OPT^{\mathrm{frac}}\) is the sum of per-block fractional group-fair optima. Since \(T=BH\), the bound is \(c\sqrt B\cdot\sqrt T\), while the same instances admit a fixed-plan comparator against which Corollary~\ref{cor:groupbdq} achieves regret \(O(\sqrt{KRL})\).
\end{proposition}

Each block introduces a fresh pair of arms whose within-pair ordering is unknown and must absorb half the block as forced exposure; an Assouad-type averaging over independent sign patterns, with a Bretagnolle--Huber two-point bound per block (cf.\ \citealp{lattimore2020bandit}), gives the result, and a Markov selection step keeps the per-block information cost bounded without restricting the regime (proof in the appendix). Because the groups are disjoint, the cost is attributable entirely to learning the plan, not to the rounding.

For disjoint systems, in fact, the gap closes entirely. No rounding is needed there, the per-block optimum decomposes across groups, and a per-pull optimistic covering rule turns each group's forced mass into its own embedded bandit. D-BDQ-UCB executes, in each block \(b\), exactly \(f_{b,g}\) covering pulls for every group \(g\) (each selected as the index-maximizing member \(\argmax_{i\in g}U_i\), with indices updated after every pull and computed from all observations) and spends the surplus \(H_b-\sum_g f_{b,g}\) rounds by the global UCB rule.

\begin{theorem}[Disjoint groups: the gap closes]
\label{thm:disjoint}
Let the groups be pairwise disjoint and let every block satisfy \(\sum_g f_{b,g}\le H_b\). D-BDQ-UCB meets every group floor in every block exactly and deterministically, and with probability at least \(1-\eta\) its realized mean value \(V\) satisfies
\[
\begin{aligned}
    \OPT^{\mathrm{frac}}-V
    &\ \le\ 4K+4\sqrt{2L}\Bigl(\sum_{g\in\calG}\sqrt{|g|F_g}+\sqrt{KR'}\Bigr)\\
    &\ \le\ 4K+8\sqrt{KTL},
\end{aligned}
\]
where \(F_g=\sum_b f_{b,g}\) is the group's total floor and \(R'=T-\sum_g F_g\) the total surplus. On the family of Proposition~\ref{prop:adaptivity-lower} the bound is \(O(\sqrt L\cdot B\sqrt H)\), matching the lower bound up to \(\sqrt L\) in every regime; with singleton groups it recovers Theorem~\ref{thm:blockwise}, whose strong and plan-matched comparators coincide.
\end{theorem}

The price of plan adaptivity is therefore \(\widetilde\Theta(\sqrt{KT})\) whenever the groups are disjoint, attained by a general algorithm rather than one tailored to the lower bound's family, and with singleton groups it recovers the blockwise guarantee of Theorem~\ref{thm:blockwise}. What remains genuinely open is the overlapping case \(t\ge2\), where a single pull credits several groups at once and the per-block optimum no longer decomposes across groups; there the one-shot planner of Theorem~\ref{thm:ogbdq} matches the lower bound only when \(H=\Theta(B)\).

The obstruction is algorithm design, not accounting, and linear-programming duality makes it precise: weak duality turns blockwise group-fair regret into an exact per-pull ledger whose two nonnegative columns charge suboptimal pulls and pulls that overshoot an already-satisfied group (Proposition~\ref{prop:dual-ledger}). The disjoint covering rule controls both columns, but under overlap it provably fails, incurring \(\Omega(T)\) regret when one arm lies in two groups whose floors a single pull can discharge at once, a multi-coverage value the dual prices encode but a within-group index cannot see (Proposition~\ref{prop:greedy-fails}); any correct overlap algorithm must be LP-aware. Our plan-sampling algorithm P-BDQ-UCB re-solves the optimistic residual program after every pull and adds a terminal cover-rounding guard that makes feasibility pathwise under an initial cover-slack condition; under two explicit one-step descent inequalities, which hold automatically on disjoint systems, it attains \(\widetilde O(\sqrt{KT})+\widetilde O(B(1/\sigma^2+K/\sigma))\) against the fractional optimum. Removing that descent condition for arbitrary overlap is the paper's one open problem. The ledger, the algorithm, the guard, the descent condition, and the conditional theorem are developed in full in Appendix~\ref{app:learned-plan}.

\section{Experiments}
\label{sec:experiments}

The experiments validate three claims on synthetic and semi-real data: exact feasibility by construction, the \(R\)-parametrized regret envelope, and the advantage of parameter-free quota construction over penalty-based methods. Every reported run lies below the explicit regret cap of Theorem~\ref{thm:blockwise} (a loose sanity check, not a rate verification; all runs use \(\eta=1/T\)), while the rate evidence is the normalized-regret trend of Figure~\ref{fig:regret-vs-T} in the appendix.

The static single-floor case is a clean rederivation; its experiments are in Appendix~\ref{app:static-exp} (Table~\ref{tab:static-all}, Figure~\ref{fig:pareto}), where DQ-UCB meets the floor at every horizon inside the \(O(\sqrt{KT\log T})\) envelope while every floor-feasible Lagrangian penalty incurs strictly positive regret. The main text focuses on the blockwise and group settings.

\begin{table}[t]
\centering\small
\resizebox{\columnwidth}{!}{%
\begin{tabular}{lrrr}
\toprule
method & blockwise regret \(\pm\) s.e.m. & block viol. & min margin \\
\midrule
BDQ-UCB & \(55.7\pm3.9\) & \(0.0\) & \(0.0\) \\
Global quota+UCB & \(2.2\pm1.2\) & \(2174.4\) & \(-120.0\) \\
Lagrangian, \(\lambda=0.5\) & \(48.0\pm2.9\) & \(6.2\) & \(-3.0\) \\
Lagrangian, \(\lambda=2\) & \(56.0\pm3.4\) & \(0.0\) & \(0.0\) \\
Lagrangian, \(\lambda=8\) & \(50.9\pm3.5\) & \(0.0\) & \(0.0\) \\
BDQ-MOSS & \(67.5\pm4.3\) & \(0.0\) & \(0.0\) \\
Queue pacing, \(V=1\) & \(590.8\pm4.4\) & \(0.0\) & -- \\
\bottomrule
\end{tabular}%
}
\caption{Synthetic blockwise quotas (\(K{=}6\), \(B{=}12\), \(H{=}600\), \(16\) seeds). Block violation \(=\) total missing mandatory pulls; min margin \(=\) smallest \(N_{b,i}-m_{b,i}\); s.e.m.\ \(=\) standard error of the mean. Only BDQ-UCB, BDQ-MOSS, and the \(\lambda\ge2\) Lagrangian are blockwise feasible; see text. Semi-real MovieLens-100k counterpart: Table~\ref{tab:blockwise-movielens} (Appendix~\ref{app:experiments-to-add}).}
\label{tab:blockwise-all}
\vspace{-5mm}
\end{table}

\subsection{Blockwise Time-Varying Quotas}

We use \(K=6\), \(B=12\) blocks, block length \(H=600\), and Bernoulli means \((0.70,0.57,0.53,0.49,0.45,0.41)\). Each block has a nonuniform rotating floor: every arm receives a base quota of \(0.03H\); one nonbest arm receives an additional \(0.17H\); another receives an additional \(0.07H\). The deadline-aware Lagrangian baseline selects, at each round of block \(b\), the arm maximizing the UCB index plus \(\lambda\,(m_{b,i}-N_{b,i}(t))_+/(H_b-\tau_b(t)+1)\), where \(N_{b,i}(t)\) is the pull count of arm \(i\) inside the current block and \(\tau_b(t)\) is the round's position within the block, so the penalty scales the remaining per-arm deficit by the block's remaining rounds.
Table~\ref{tab:blockwise-all} compares BDQ-UCB with a global-quota method that satisfies only the aggregate floor and with this Lagrangian (block violation counts the total missing mandatory pulls); Figure~\ref{fig:blockwise} in Appendix~\ref{app:experiments-to-add} visualizes the same comparison.

The blockwise experiment shows the modeling difference and the value of parameter-free feasibility. The global guaranteed-pull method meets aggregate exposure but violates the block constraints by \(2174\) pulls, so its low apparent regret is an artifact of infeasibility. A deadline-aware Lagrangian, once its penalty is large enough, is both feasible and competitive with BDQ-UCB on regret; we do not claim BDQ-UCB dominates it. BDQ-UCB's advantage is that exact blockwise feasibility holds deterministically on every instance with no penalty to tune, whereas the Lagrangian's feasibility depends on a per-instance penalty search, with small penalties still violating the floor. We tested the regret claim directly: a sweep over the mandatory fraction found \emph{no} regime where BDQ-UCB strictly beats the best feasible penalty, the tuned penalty being modestly ahead with the gap shrinking as the floor tightens (Table~\ref{tab:regret-regime}, Appendix~\ref{app:experiments-to-add}). The contribution is parameter-free exact feasibility, not a regret advantage. The remaining two rows are consistent: BDQ-MOSS is exactly feasible but pays a premium for discarding mandatory observations, and a virtual-queue pacing baseline in the style of \citet{li2019combinatorial} is feasible only at its most conservative setting \(V=1\) at much higher regret, its \(V\)-sweep lowering regret only by breaking feasibility, the same tuning trap as the Lagrangian.

\subsection{Semi-Real MovieLens Genre Benchmark}
\label{sec:movielens-protocol}

We instantiate one semi-real benchmark based on MovieLens-100k \citep{harper2015movielens} genre exposure. Arms are the \(K=18\) high-level genre providers, each with a reward mean \(\widehat\mu_i\) estimated from historical ratings (protocol in Appendix~\ref{app:escale-detail}). Blocks represent recommendation periods (\(B=12\), \(H=600\)); a rotating contract elevates one or two genres per block for higher minimum exposure while the rest receive a small base floor. Each method is run with Bernoulli rewards calibrated by \(\widehat\mu_i\), and the comparator is the best blockwise-fair allocation under those calibrated means (\(16\) seeds): the genre means are real, the reward draws calibrated Bernoulli. Results appear in Table~\ref{tab:blockwise-movielens} (Appendix~\ref{app:experiments-to-add}) and match the synthetic pattern.

To drop the calibrated-Bernoulli assumption entirely, we rerun the \(K{=}18\) MovieLens benchmark with rewards drawn non-parametrically from each genre's \emph{actual} empirical rating distribution: at each pull of genre \(i\) we sample a real observed normalized rating, making the reward process the true one rather than a two-point calibration (\(16\) seeds). The picture is unchanged: BDQ-UCB is exactly feasible (\(0\) block violation) at blockwise regret \(102.4\pm1.3\), the small penalty \(\lambda{=}0.5\) violates the floors (\(34.6\) missing pulls), and the feasible penalties \(\lambda\in\{2,8\}\) tie it on regret (\(101.6\)--\(101.8\); Table~\ref{tab:realreplay}, Appendix~\ref{app:experiments-to-add}). Exact feasibility and the feasible-Lagrangian pattern therefore hold under the real reward distribution, not only calibrated draws.

\subsection{Stress Test: Penalty Tuning at Scale}
\label{sec:stress}

The blockwise table above exhibits the tuning problem on two instances (\(\lambda=0.5\) violates, \(\lambda\ge2\) is feasible); the stronger claim is that no single penalty works \emph{across} instances. A \(24\)-instance battery (arm counts \(K\in\{18,64,256\}\), three gap scales, three floor tightnesses, heterogeneous block lengths; \(16\) seeds each) confirms it: over the grid \(\lambda\in\{0.25,\dots,32\}\) feasibility climbs monotonically (\(3,4,6,8,9,16,18,22\) of \(24\)) but no bounded penalty clears more than \(22/24\), the smallest feasible \(\lambda\) spans a \(128\times\) range, and two tight small-gap large-\(K\) instances admit no feasible \(\lambda\le32\). Only the \(\lambda\to\infty\) hard schedule and BDQ-UCB reach all \(24\) (regret \(59.0\) and \(58.5\)); on the oracle-feasible subset the tuned per-instance oracle edges BDQ-UCB on regret (\(60.6\) vs \(61.3\)) at the cost of a \(192\)-run sweep, so the advantage is again tuning-free exact feasibility, not lower regret. Full setup, Table~\ref{tab:escale}, and Figures~\ref{fig:escale-lambda} and~\ref{fig:escale-intervals} are in Appendix~\ref{app:experiments-to-add}.

A group-floor arm exercises the planning layer on real structure: the \(54\) MovieLens genre\(\times\)popularity cells (\(18\) genres \(\times\) \(3\) tiers), with the \(18\) genre unions and \(3\) tier unions as overlapping groups (arm degree \(t=2\)). Learning the plan pays: OG-BDQ-UCB improves on the fixed-plan Group-BDQ-UCB by about \(8\%\) (\(1083\) vs \(1174\) against the per-block fractional optimum), while a tuned group-Lagrangian is again lower-regret when feasible; the mild overlap keeps all three feasible, so this instance isolates the value of \emph{learning} the plan rather than the Beck--Fiala necessity of the adversarial construction. A disjoint control (genre-only floors, \(t=1\)) runs D-BDQ-UCB, exactly feasible with no rounding as Theorem~\ref{thm:disjoint} predicts; a scaling sweep on the lower-bound family of Proposition~\ref{prop:adaptivity-lower} confirms its sharpest prediction, D-BDQ-UCB's normalized regret staying flat while the one-shot planner's grows and is \(2.5\times\) larger in the block-sparse regime (Table~\ref{tab:egroup}, Figure~\ref{fig:dbdq-scaling}, Appendix~\ref{app:experiments-to-add}).

\subsection{Deployment-Facing Stress Tests}
\label{sec:deployment}

Six further studies (Appendix~\ref{app:experiments-to-add}, Table~\ref{tab:deployment}), all over \(16\) seeds and reusing the same primitives, stress deployment-relevant axes: off-policy value recovery (N1), overlapping group floors (N2), an adaptively tuned online penalty (N3), a LinUCB contextual layer (N4), delayed feedback (N5), and a mandatory-sample placement ablation (N6). Because feasibility is a scheduling property it survives every one, with only regret ever degrading: the off-policy value is recovered to within \(0.002\) mean absolute error by IPS, SNIPS, and doubly-robust estimators, a group-blind policy misses group floors by \(34.3\) pulls per block where group-aware rounding stays exact, and even the adaptively tuned online penalty still pays \(1.25\) block violations while it ramps up.


\section{Conclusion}

We treated fair exposure as rounding: the gap between a fractional fair schedule and an integral pull sequence is a discrepancy vector, feasibility is its control, and reward loss its weighted form. This single identity organizes the paper. It recovers the static Fair-MAB guarantee from rounding alone, but pays off in the blockwise model, where BDQ-UCB meets every period's floor deterministically and incurs regret only on the nonmandatory budget \(R\), the right complexity parameter in both the minimax (\(\Theta(\sqrt{KR})\)) and instance-dependent senses. Under overlap, Beck--Fiala null-space rounding meets every group floor within the block budget where per-arm rounding fails, and learning the plan exposes a new lower bound, settled at \(\widetilde\Theta(\sqrt{KT})\) for disjoint systems and attained conditionally under overlap. Experiments on synthetic, MovieLens, and deployment settings confirm exact feasibility without penalty tuning; the overlapping learned-plan rate without our descent condition is the main open problem, with adversarially adaptive floors and matroid-like feasibility systems the natural next targets for this discrepancy bridge.

\bibliography{aaai2027}

\appendix
\setcounter{secnumdepth}{2}
\renewcommand{\thesection}{\AlphAlph{\value{section}}}

\section{Notation}
\label{app:notation}

\begin{table}[h]
\centering
\small
\begin{tabular}{cl}
\toprule
Symbol & Meaning \\
\midrule
\(K\) & number of arms \\
\(T\) & time horizon \\
\(\mu_i\) & mean reward of arm \(i\) \\
\(i^\star\) & index of a best arm \\
\(\Delta_i\) & suboptimality gap \(\mu_\star - \mu_i\) \\
\(m\) & integral fairness floor (static case) \\
\(B,\,H_b\) & number of blocks; length of block \(b\) \\
\(m_{b,i}\) & floor for arm \(i\) in block \(b\) \\
\(N_i(T)\) & total pull count of arm \(i\) through round \(T\) \\
\(R\) & total nonmandatory rounds \(\sum_b(H_b - \sum_i m_{b,i})\) \\
\(S(T)\) & cumulative fractional exposure vector \\
\(D_T\) & discrepancy vector \(N(T) - S(T)\) \\
\(\calG,\,t\) & group system; maximum arm degree \\
\(f_g\) & aggregate floor of group \(g\) \\
\(\OPT_m(T)\) & value of the best fair allocation \\
\(\Reg_m(T)\) & expected fair pseudo-regret \\
\(\calE\) & high-probability concentration event \\
\(L\) & \(\log(2KT/\eta)\) \\
\(\eta\) & failure probability \\
\bottomrule
\end{tabular}
\caption{Notation used throughout the paper.}
\label{tab:notation}
\end{table}

\section{Summary of Results}
\label{app:results-summary}

Table~\ref{tab:glance} collects the paper's guarantees in one place, grouping each setting with its upper bound, matching lower bound, and the algorithm that attains the rate. Reading down the table traces the paper's arc: the residual budget \(R\) governs the blockwise settings and the horizon \(T\) governs the learned-plan settings, and in every row the upper and lower bounds match up to logarithmic factors except where noted.

\begin{table}[h]
\centering\footnotesize
\setlength{\tabcolsep}{4pt}
\begin{tabular}{@{}p{0.20\linewidth}p{0.31\linewidth}p{0.12\linewidth}p{0.21\linewidth}@{}}
\toprule
Setting & Upper bound & Lower bound & Algorithm \\
\midrule
Static, plan-matched & \(O(\sqrt{KT\log KT})\) & (Fair-MAB) & DQ-UCB (Thm.~\ref{thm:main}) \\
Blockwise, minimax & \(O(\sqrt{KR})\) & \(\Omega(\sqrt{KR})\) & BDQ-MOSS (Thm.~\ref{thm:bdqmoss}) \\
Blockwise, instance & \((1{+}o(1))\log R\) \(\textstyle\sum_i \Delta_i/\mathrm{kl}_i\) & matching & BDQ-KL (Cor.~\ref{cor:bdqkl}) \\
Groups, fixed plan & \(O(\sqrt{KR})\), exact floors & \(-\) & Group-BDQ-UCB (Cor.~\ref{cor:groupbdq}) \\
Learned plan, disjoint & \(\widetilde{O}(\sqrt{KT})\) & \(\Omega(\sqrt{KT})\) & D-BDQ-UCB (Thm.~\ref{thm:disjoint}) \\
Learned plan, any \(\calG\) & \(\widetilde{O}(\sqrt{KT})\) \({+}\,O_\sigma(BK{+}BL)\)\textsuperscript{\dag} & \(\Omega(\sqrt{KT})\) & P-BDQ-UCB (Thm.~\ref{thm:pbdq}) \\
\bottomrule
\end{tabular}
\caption{Complexity picture across settings; upper and lower bounds match except where noted. Here \(\mathrm{kl}_i=\mathrm{kl}(\mu_i,\mu_\star)\) and \(R\) is the nonmandatory budget. Plan-matched benchmarks execute the same mandatory schedule; learned-plan benchmarks target the per-block fractional group-fair optimum. \textsuperscript{\dag}Conditional on the sampled-plan descent condition and initial cover slack; feasibility of P-BDQ-UCB needs only the slack condition.}
\label{tab:glance}
\end{table}

\section{Static Exposure Floors and Pareto Frontier}
\label{app:static}

\subsection{Static Setup: Floor, Best-Fair Allocation, and the Fair-Regret Identity}
\label{sec:static-setup}

The static single-floor objects specialized in the main text are as follows.

\begin{definition}[Integral fairness floor]
For a target exposure fraction \(\delta\in[0,1/K]\), define the integral floor
\[
    m = \lfloor \delta T \rfloor.
\]
A policy is \(m\)-fair if its final pull counts \(N_i(T)\) satisfy
\[
    N_i(T)\ge m
    \qquad\text{for every } i\in[K].
\]
Equivalently, the empirical exposure of every arm is at least \(m/T\ge\delta-1/T\). If \(K\lceil \delta T\rceil\le T\), one may replace \(m=\lfloor\delta T\rfloor\) by \(m=\lceil\delta T\rceil\) everywhere below to obtain the literal floor \(N_i(T)\ge\delta T\).
\end{definition}

The best fair allocation is the allocation that maximizes expected reward subject to the same integral floor. Since all rewards are stationary and independent, the best fair allocation gives exactly \(m\) pulls to every suboptimal arm and all remaining pulls to a best arm; a formal proof is in the appendix.

\begin{definition}[Best fair allocation and fair pseudo-regret]
Let
\[
    n_i^\star(m)=
    \begin{cases}
        T-(K-1)m, & i=i^\star,\\
        m, & i\neq i^\star.
    \end{cases}
\]
The value of the best fair allocation is
\[
    \OPT_m(T)
    = \sum_{i=1}^K n_i^\star(m)\mu_i
    = T\mu_\star - m\sum_{i\neq i^\star}\Delta_i.
\]
For a policy \(\pi\), its fair pseudo-regret is
\[
    \Reg_m^\pi(T)
    =
    \OPT_m(T)
    -
    \E_\pi\!\left[\sum_{t=1}^T \mu_{A_t}\right],
\]
where \(A_t\) is the arm pulled at time \(t\). The realized fair pseudo-regret is
\[
    \widehat{\Reg}_m^\pi(T)
    =
    \OPT_m(T)
    -
    \sum_{t=1}^T \mu_{A_t}.
\]
\end{definition}

\begin{lemma}[Fair-regret identity]
\label{lem:static-id}
For every pull sequence \(A_1,\dots,A_T\),
\[
    \widehat{\Reg}_m(T)
    =
    \sum_{i\neq i^\star} \Delta_i\bigl(N_i(T)-m\bigr).
\]
Consequently, if the policy is \(m\)-fair, then \(\widehat{\Reg}_m(T)\ge 0\).
\end{lemma}

\subsection{DQ-UCB Algorithm and Guarantees}
\label{sec:dqucb-all}

The mandatory part of a fair bandit schedule requires \(m\) pulls of every arm, and the following elementary discrepancy construction gives an integral schedule with exactly those counts and uniformly small prefix imbalance.

\begin{lemma}[Balanced quota rounding]
\label{lem:balanced}
Consider the length-\(Km\) sequence that pulls arms cyclically,
\[
    1,2,\dots,K,\;1,2,\dots,K,\;\ldots,\;1,2,\dots,K,
\]
with exactly \(m\) full cycles, and let \(N_i(t)\) be the number of pulls of arm \(i\) among the first \(t\le Km\) pulls. Then \(N_i(Km)=m\) for every \(i\), and for every prefix \(t\le Km\),
\[
    \left|N_i(t)-\frac{t}{K}\right|\le 1
    \qquad\text{for every } i.
\]
\end{lemma}

Lemma~\ref{lem:balanced} is a constructive discrepancy bound: the integral sequence has exact terminal quota and at most one unit of prefix discrepancy relative to the uniform fractional quota schedule. DQ-UCB builds on it in two phases. The discrepancy quota phase uses the balanced schedule to pull every arm exactly \(m\) times, enforcing the fairness floor deterministically. The optimistic residual phase then runs UCB \citep{auer2002finite} on the remaining \(T-Km\) pulls; since the quota is already satisfied, this phase can focus purely on reward learning. When \(m=0\) the fairness constraint is vacuous, and the algorithm pulls each arm once for initialization before running UCB; we assume \(T\ge K\) in that case. Pseudocode is given as Algorithm~\ref{alg:dqucb}, with ties in the index broken toward the lowest arm index.

\begin{algorithm}
\caption{DQ-UCB}
\label{alg:dqucb}
\begin{algorithmic}[1]
\STATE \textbf{Input:} arms $K$, horizon $T$, floor $m\in\{0,\dots,\lfloor T/K\rfloor\}$, failure prob.\ $\eta\in(0,1)$
\STATE \textbf{Output:} pull sequence $A_1,\dots,A_T$
\STATE $s \leftarrow \max(m,\,1)$; \quad $L \leftarrow \log(2KT/\eta)$
\FOR{$j = 1,\dots,s$}
    \FOR{$i = 1,\dots,K$}
        \STATE Pull arm $i$
    \ENDFOR
\ENDFOR
\FOR{$t = Ks+1,\dots,T$}
    \STATE Pull $A_t \in \displaystyle\argmax_{i\in[K]}\!\left\{\widehat{\mu}_i(t) + \sqrt{2L/N_i(t)}\right\}$
\ENDFOR
\end{algorithmic}
\end{algorithm}

\begin{theorem}[Exact fairness and fair regret]
\label{thm:main}
Assume rewards are independent and supported in \([0,1]\). Let \(m\in\{0,1,\dots,\lfloor T/K\rfloor\}\), and if \(m=0\) assume \(T\ge K\). Then DQ-UCB satisfies the following.
\begin{enumerate}[label=(\roman*),leftmargin=*]
\item \emph{Exact fairness.} Deterministically,
\[
    N_i(T)\ge m
    \qquad\text{for every } i\in[K].
\]
\item \emph{Gap-dependent high-probability bound.} With probability at least \(1-\eta\),
\begin{align}
      \widehat{\Reg}_m(T)
      &\le
      K
      +
      \sum_{i:\Delta_i>0}
      \min\left\{
          T\Delta_i,\,
          \frac{8L}{\Delta_i}
      \right\}, \notag\\
      &\qquad L=\log\!\left(\frac{2KT}{\eta}\right).
  \end{align}
\item \emph{Gap-free high-probability bound.} With probability at least \(1-\eta\),
\[
    \widehat{\Reg}_m(T)
    \le
    K
    +
    4\sqrt{2KT L}.
\]
\item \emph{Expected regret.} Taking \(\eta=1/T\),
\[
    \Reg_m(T)
    =
    O\!\left(\sqrt{KT\log(KT)}+K\right).
\]
\end{enumerate}
\end{theorem}

\begin{proof}[Proof sketch]
Fairness is deterministic: the quota phase pulls every arm exactly \(m\) times. For regret, the standard UCB concentration event \(\calE\) (Hoeffding's inequality, union bound over \(KT\) pairs) holds with probability \(1-\eta\). On \(\calE\), any suboptimal arm \(i\) with gap \(\Delta_i\) can be selected at most \(1+8L/\Delta_i^2\) times during the residual phase. Summing via Lemma~\ref{lem:static-id} gives the gap-dependent bound; an \(\varepsilon\)-split optimized at \(\varepsilon=\sqrt{8KL/T}\) gives the gap-free bound. The full proof is in Appendix~\ref{app:main-proof}.
\end{proof}

\begin{remark}[Anytime fairness]
\label{rem:anytime}
Because the quota is front-loaded, DQ-UCB is fair at every round, not only at the horizon: for all \(t\le T\) and all \(i\), \(N_i(t)\ge \delta t-1\). During the quota phase, Lemma~\ref{lem:balanced} gives \(N_i(t)\ge t/K-1\ge\delta t-1\) since \(\delta\le1/K\); afterwards \(N_i(t)\ge m=\lfloor\delta T\rfloor\ge\delta t-1\). DQ-UCB therefore also meets the anytime fairness requirement of \citet{patil2021fair} with tolerance one.
\end{remark}

\begin{corollary}[Realized reward regret]
\label{cor:realized}
Let \(R_T=\sum_{t=1}^T X_{A_t}\) be the realized reward. With probability at least \(1-2\eta\),
\[
    \OPT_m(T)-R_T
    \le
    K+4\sqrt{2KT\log(2KT/\eta)}
    +
    \sqrt{2T\log(1/\eta)}.
\]
\end{corollary}

\subsection{Fairness--Regret Pareto Frontier}
\label{sec:pareto}

The tradeoff between the floor level and achievable reward is a property of the underlying allocation problem, independent of the learning algorithm, and it admits a closed form. For a probability vector \(p\in\R^K\) (\(p_i\ge0\), \(\sum_i p_i=1\)), define the fairness level \(\phi(p)=\min_i p_i\) and the expected reward \(V(p)=\sum_i p_i\mu_i\); an allocation \(p\) \emph{dominates} \(q\) if \(\phi(p)\ge\phi(q)\) and \(V(p)\ge V(q)\) with at least one strict inequality.

\begin{theorem}[Closed-form static Pareto frontier]
\label{thm:pareto}
Assume the best arm \(i^\star\) is unique. For every fairness level \(\beta\in[0,1/K]\), define
\[
    p_i^\beta
    =
    \begin{cases}
        1-(K-1)\beta, & i=i^\star,\\
        \beta, & i\neq i^\star.
    \end{cases}
\]
Then the Pareto frontier is exactly the curve
\(\{(\phi(p^\beta),V(p^\beta)):\beta\in[0,1/K]\}\).
Moreover \(\phi(p^\beta)=\beta\) and
\[
    V(p^\beta)
    =
    \mu_\star
    -
    \beta\sum_{i\neq i^\star}\Delta_i,
\]
so the unconstrained regret per round at fairness level \(\beta\) is \(\beta\sum_{i\neq i^\star}\Delta_i\).
\end{theorem}

\begin{corollary}[Algorithmic approach to the frontier]
\label{cor:frontier}
For any integral floor \(m\), DQ-UCB targets the point \(\beta=m/T\) on the Pareto frontier; by Theorem~\ref{thm:main}, its per-round reward gap relative to \(V(p^{m/T})=\mu_\star-(m/T)\sum_{i\neq i^\star}\Delta_i\) vanishes at rate \(O(\sqrt{K\log(KT)/T})\).
\end{corollary}

\subsection{Static Experiments}
\label{app:static-exp}

We evaluate the exact DQ-UCB policy: a deterministic round-robin quota of \(m=\lfloor\delta T\rfloor\) pulls per arm, followed by pure UCB on the residual budget. Because the quota phase meets the floor deterministically, realized fair regret is always nonnegative, and no slack convention or terminal repair is needed. All static experiments use synthetic Bernoulli fair-bandit instances on CPU, reporting means over \(16\) seeds; the comparator is the best fair allocation (\(m\) pulls on every arm and the remaining budget on the best arm), and the baseline is the Lagrangian-penalty policy that selects the arm maximizing the UCB index plus \(\lambda(\delta-N_i(t)/t)\), swept over \(\lambda\in\{0.5,2,8\}\). We report fair pseudo-regret \(\widehat{\Reg}_m(T)\) and the realized minimum exposure fraction \(\min_i N_i(T)/T\). We also ran the Fair-MAB algorithm of \citet{patil2021fair} on the same instances: it is anytime-feasible throughout and statistically ties DQ-UCB, with both at approximately zero fair regret on these easy instances, consistent with Remark~\ref{rem:anytime}; we report it in prose rather than as additional table rows.

The first experiment (E1) fixes \(K=5\) and \(\delta=0.1\) and varies the horizon; Table~\ref{tab:static-all}(a) reports fair regret against the best-fair oracle and the realized minimum exposure, and Figure~\ref{fig:regret-vs-T} in the appendix shows the trend across horizons. The second (E2) fixes \(K=5\) and \(T=3.2\times10^4\) and sweeps the floor \(\delta\); Table~\ref{tab:static-all}(b) and Figure~\ref{fig:pareto} report the resulting frontier, where the best feasible Lagrangian is the lowest-regret penalty setting that still meets the floor. The zero-regret entries for DQ-UCB mean that on this particular instance the residual UCB phase matched the best fair comparator after the mandatory quota was satisfied (after \(m\) quota samples per arm, the best arm's index remains above every suboptimal index for the rest of the horizon) and should be read as an easy-instance diagnostic rather than a universal theorem. The third (E3) fixes \(T=3.2\times10^4\) and \(\delta=0.1\) and varies \(K\) and the reward-gap structure (Table~\ref{tab:static-all}(c)). The \(K=10\) rows report zero fair regret because \(\delta=1/K\) at \(K=10\): the floor \(m=\lfloor\delta T\rfloor\) makes the mandatory quota \(Km=T\) exhaust the horizon, leaving residual budget \(R=0\), so no optimistic phase runs and the allocation coincides with the best fair comparator. The \(K=5\) rows have \(R=T-Km=\tfrac12 T>0\), and the hard instance (gap \(0.05\)) accrues the largest normalized regret, still within the \(O(\sqrt{KT\log T})\) envelope.

\begin{figure}[t]
\centering
\placeholdergraphic[width=\columnwidth]{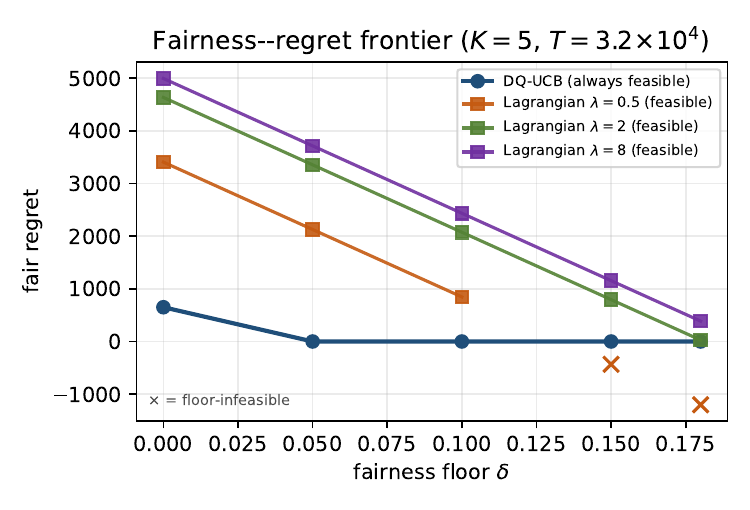}
\caption{E2: Fairness--regret Pareto frontier (\(K=5\), \(T=3.2\times10^4\)). DQ-UCB achieves zero fair regret at every floor level \(\delta\) on this instance; the best feasible Lagrangian incurs positive regret ranging from \(2126\) at \(\delta=0.05\) to \(26.8\) at \(\delta=0.18\).}
\label{fig:pareto}
\end{figure}

\begin{table*}[t]
\centering
\small
\textit{(a) E1: Fair regret vs.\ horizon (\(K=5\), \(\delta=0.1\))}\\[2pt]
\begin{tabular}{rrrrr}
\toprule
\(T\) & fair regret \(\pm\) s.e.m. & \(\min_i N_i/T\) & \(\Reg/T\) &
\(\Reg/\sqrt{KT\log T}\) \\
\midrule
\(2{,}000\)   & \(31.5 \pm 3.9\) & \(0.104\) & \(0.0157\) & \(0.114\) \\
\(8{,}000\)   & \(0.0 \pm 0.0\)  & \(0.100\) & \(0.000\)  & \(0.000\) \\
\(32{,}000\)  & \(0.0 \pm 0.0\)  & \(0.100\) & \(0.000\)  & \(0.000\) \\
\(128{,}000\) & \(0.0 \pm 0.0\)  & \(0.100\) & \(0.000\)  & \(0.000\) \\
\bottomrule
\end{tabular}
\par\vspace{4pt}
\textit{(b) E2: Fairness--regret frontier (\(K=5\), \(T=3.2\times10^4\))}\\[2pt]
\begin{tabular}{rrrr}
\toprule
\(\delta\) & DQ-UCB regret & \(\min_i N_i/T\) & best feasible Lagrangian regret \\
\midrule
\(0.05\) & \(0.0\) & \(0.050\) & \(2126.4\) \\
\(0.10\) & \(0.0\) & \(0.100\) & \(846.4\)  \\
\(0.15\) & \(0.0\) & \(0.150\) & \(794.8\)  \\
\(0.18\) & \(0.0\) & \(0.180\) & \(26.8\)   \\
\bottomrule
\end{tabular}
\par\vspace{4pt}
\textit{(c) E3: Robustness (\(T=3.2\times10^4\), \(\delta=0.1\))}\\[2pt]
\begin{tabular}{rlrrr}
\toprule
\(K\) & instance & gap & fair regret \(\pm\) s.e.m. & \(\min_i N_i/T\) \\
\midrule
\(5\)  & easy & \(0.20\) & \(0.0 \pm 0.0\)     & \(0.100\) \\
\(5\)  & hard & \(0.05\) & \(197.7 \pm 12.2\)  & \(0.114\) \\
\(10\) & easy & \(0.20\) & \(0.0 \pm 0.0\)     & \(0.100\) \\
\(10\) & hard & \(0.05\) & \(0.0 \pm 0.0\)     & \(0.100\) \\
\bottomrule
\end{tabular}
\caption{Static DQ-UCB experiments (\(16\) seeds; s.e.m.\ denotes the standard error of the mean). (a) Floor met at every horizon; fair regret per round diminishes to zero by \(T=8{,}000\), where the exact zeros arise because, after the \(m\) quota samples per arm, the best arm's index remains above every suboptimal index for the remainder of the horizon on this instance. (b) At every \(\delta>0\), DQ-UCB attains zero fair regret on this instance (an easy-instance diagnostic, not a universal claim) while all floor-feasible Lagrangian settings incur positive regret. (c) Floor met at all \((K,\text{gap})\); the largest normalized regret stays within the \(O(\sqrt{KT\log T})\) envelope.}
\label{tab:static-all}
\end{table*}

\section{Additional Blockwise Complexity Results}
\label{app:blockwise-extra}

\subsection{BDQ-MOSS: Closing the Minimax Rate}

Replace the residual selection rule of BDQ-UCB by the MOSS index of \citet{audibert2009minimax}, computed on residual observations only: writing \(N^{\mathrm{res}}_i\) and \(\widehat\mu^{\mathrm{res}}_i\) for the pull count and empirical mean of arm \(i\) over residual rounds alone, the rule pulls each arm once during the first \(K\) residual rounds and thereafter selects
\[
\argmax_{i\in[K]}\left\{\widehat\mu^{\mathrm{res}}_i+\sqrt{\frac{\max\bigl(\log\bigl(R/(K N^{\mathrm{res}}_i)\bigr),\,0\bigr)}{N^{\mathrm{res}}_i}}\right\}.
\]
We call the resulting algorithm BDQ-MOSS.

\begin{theorem}[BDQ-MOSS closes the minimax rate]
\label{thm:bdqmoss}
BDQ-MOSS satisfies the deterministic blockwise fairness guarantee of Theorem~\ref{thm:blockwise}(i), and its expected blockwise fair regret satisfies
\[
    \Reg_{\mathbf m}(T)\ \le\ C\sqrt{KR}+K
\]
for a universal constant \(C\). Combined with Theorem~\ref{thm:block-lower}, the minimax blockwise fair regret is \(\Theta(\sqrt{KR})\) for \(R\ge K\), and BDQ-MOSS attains it up to a universal constant.
\end{theorem}

\begin{proof}[Proof idea]
Because the residual rule reads only residual history, the \(R\) residual rounds form a standard \(K\)-armed bandit played by MOSS; by Lemma~\ref{lem:block-id} the fair regret equals that embedded pseudo-regret, which MOSS bounds by \(C\sqrt{KR}\) \citep{audibert2009minimax}. A doubling trick handles unknown \(R\) (Appendix~\ref{app:bdqmoss-proof}).
\end{proof}

\subsection{Positive Mandatory Exposure and Instance-Dependent Rates}

The zero-floor reduction invites the objection that the lower bound says nothing about instances with genuine mandatory exposure. It does, once the mandatory mass is placed carefully.

\begin{proposition}[Positive mandatory exposure does not remove the residual lower bound]
\label{prop:positive-floor-lower}
For every \(K\ge2\), \(R\ge K\), and any integer mandatory budget \(M\ge0\), there is a blockwise-fair instance with total mandatory exposure \(M\) and total residual budget \(R\) on which every blockwise-fair policy suffers minimax fair regret \(\Omega(\sqrt{KR})\).
\end{proposition}

The construction, given in full in Appendix~\ref{app:positive-floor-proof}, places the entire mandatory mass on a reference arm whose reward distribution is identical across all instances in the lower-bound family: a first block of length \(M\) carries the floor vector \((M,0,\ldots,0)\), and arm \(1\) is Bernoulli\((1/2)\) under every instance while the alternatives shift only one of the remaining arms. The forced pulls of arm \(1\) then carry no information (they contribute zero KL divergence between instances) and are matched by the comparator, so by Lemma~\ref{lem:block-id} they cancel from fair regret; the residual block of length \(R\) reproduces the standard construction over the remaining arms, giving \(\Omega(\sqrt{(K-1)R})=\Omega(\sqrt{KR})\). The point of the proposition is that mandatory exposure, however large, cannot buy the learner out of the \(\Omega(\sqrt{KR})\) residual cost.

The minimax results leave open whether easy instances are easier in the residual budget. They are, but only down to the classical logarithmic barrier, and again it is \(R\) that sets the clock. Call a blockwise-fair policy \emph{uniformly fast} on the template family of Proposition~\ref{prop:positive-floor-lower} (arm \(1\) Bernoulli\((1/2)\) in every instance, all means in \((0,1)\) with \(\mu_\star<1\)) if its expected fair regret is \(o(R^a)\) for every \(a\in(0,1]\) on every instance of the family, as \(R\to\infty\) with the mandatory budget \(M=M(R)\) growing at most polynomially in \(R\). BDQ-UCB with \(\eta=1/T\) is uniformly fast by Theorem~\ref{thm:blockwise}(ii), so the class is nonempty.

\begin{theorem}[Instance-dependent lower bound]
\label{thm:instance-lower}
Let \(\pi\) be any uniformly fast blockwise-fair policy. Then for every instance \(\nu\) of the template family and every arm \(i\ne1\) with \(\Delta_i>0\),
\[
    \liminf_{R\to\infty}\ \frac{\E_\nu\bigl[N^{\mathrm{res}}_i\bigr]}{\log R}\ \ge\ \frac{1}{\mathrm{kl}(\mu_i,\mu_\star)},
\]
and consequently
\[
    \liminf_{R\to\infty}\ \frac{\Reg_{\mathbf m}(T)}{\log R}\ \ge\ \sum_{i\ne1:\,\Delta_i>0}\frac{\Delta_i}{\mathrm{kl}(\mu_i,\mu_\star)}.
\]
\end{theorem}

The proof, in Appendix~\ref{app:instance-lower-proof}, is a change of measure via the Bretagnolle--Huber inequality against the instance that raises arm \(i\) just above \(\mu_\star\); the mandatory pulls contribute zero divergence, so the information budget is again \(R\). Together with the per-arm term \(8\log(2KT/\eta)/\Delta_i\) of Theorem~\ref{thm:blockwise}(ii), the theorem shows that logarithmic-in-\(R\) growth with \(1/\Delta_i\) scaling is necessary, up to the constant relating \(\mathrm{kl}(\mu_i,\mu_\star)\) to \(\Delta_i^2\). The constant itself is attainable through the same embedded-game reduction.

\begin{corollary}[BDQ-KL: both regimes simultaneously]
\label{cor:bdqkl}
For Bernoulli rewards, let BDQ-KL replace the residual rule of BDQ-UCB by kl-UCB\(^{++}\) \citep{menard2017minimax}, run on residual observations only with horizon parameter \(R\). BDQ-KL retains the deterministic blockwise fairness guarantee, its expected fair regret is at most \(C\sqrt{KR}+K\) for a universal constant \(C\), and on every fixed instance
\[
    \limsup_{R\to\infty}\ \frac{\Reg_{\mathbf m}(T)}{\log R}\ \le\ \sum_{i\ne i^\star}\frac{\Delta_i}{\mathrm{kl}(\mu_i,\mu_\star)} .
\]
\end{corollary}

BDQ-KL is uniformly fast, and its per-arm residual counts meet the bound of Theorem~\ref{thm:instance-lower} with equality for every arm the theorem constrains, so the plan-matched blockwise complexity is closed in both regimes, with the exact constant in the instance-dependent one (proof in Appendix~\ref{app:bdqkl-proof}). The lower bounds, Theorem~\ref{thm:bdqmoss}, and Corollary~\ref{cor:bdqkl} together say that blockwise fairness costs nothing beyond the classical price of learning on the residual budget: \(\Theta(\sqrt{KR})\) in the worst case, \((1+o(1))\log R\sum_i\Delta_i/\mathrm{kl}(\mu_i,\mu_\star)\) on fixed instances, with the mandatory budget \(M\) appearing in neither.

\section{Online Block-Start Floors}
\label{app:online-floors-proof}

\begin{theorem}[Online block-start floors]
\label{thm:online-floors}
Suppose that before block \(b\) begins, the learner observes \(m_b\), with \(\sum_i m_{b,i}\le H_b\), but does not know future floor vectors \(m_{b+1},\ldots,m_B\). Online-BDQ-UCB, which executes the observed mandatory schedule for block \(b\) and then runs UCB on the residual rounds of that block, satisfies all realized block constraints exactly. Conditional on the realized floor sequence, the regret bounds of Theorem~\ref{thm:blockwise} hold with
\[
    R=\sum_{b=1}^B\left(H_b-\sum_i m_{b,i}\right).
\]
\end{theorem}

\begin{proof}
Feasibility is block-local: once \(m_b\) is revealed, the algorithm schedules exactly \(m_{b,i}\) pulls of each arm before the block's residual UCB rounds, meeting the block constraint. The UCB concentration event depends only on sample counts and does not require knowledge of future floors. Conditional on the realized sequence \((m_1,\ldots,m_B)\), the proof of Theorem~\ref{thm:blockwise} applies verbatim.
\end{proof}

\section{Learning the Group Plan: Extended Analysis}
\label{app:learned-plan}

This appendix develops in full the learned-plan material summarized in Section~\ref{sec:learned-plan}: the one-shot planner OG-BDQ-UCB, the benchmark-slack calibration, the dual ledger and its overlap counterexample, and the plan-sampling algorithm P-BDQ-UCB together with its feasibility guarantee, its descent condition, and the conditional regret theorem. Proofs of the statements collected here appear in the proof sections that follow.

\subsection{The One-Shot Planner OG-BDQ-UCB}

OG-BDQ-UCB (optimistic-plan Group-BDQ-UCB) runs in each block \(b\) as follows: pull every arm once (\(K\) initialization rounds); form the clipped indices \(U_i=\min\{1,\,\widehat\mu_i+\sqrt{2L/N_i}\}\) from all observations so far; solve the linear program that maximizes \(\sum_i U_iy_i\) over the \emph{plan polytope}
\[
\begin{aligned}
    P_b=\Bigl\{y\ge0:\ &\textstyle\sum_{i\in g}y_i\ge f_{b,g}+t\ \ \forall g\in\calG,\\
    &\textstyle\sum_i y_i=H_b-2K\Bigr\},
\end{aligned}
\]
giving the plan \(y_b\); round \(y_b\) by Beck--Fiala to \(n_b\) and execute it; and spend the at most \(2K\) leftover rounds of the block by the UCB rule. The plan polytope is assumed nonempty, that is, the slack floors are coverable within \(H_b-2K\); the \(2K\) reserve absorbs the ceiling budget of the rounding and the initialization. The benchmark is the strongest one available to any per-block policy with the same rounding slack, \(\OPT^{\mathrm{ad}}=\sum_{b=1}^B\bigl[\max_{y\in P_b}\langle\mu,y\rangle+2K\mu_\star\bigr]\), the best slack-feasible fractional plan per block with its reserve spent on a best arm.

\begin{theorem}[OG-BDQ-UCB: learning the plan]
\label{thm:ogbdq}
OG-BDQ-UCB meets every group floor in every block deterministically, and with probability at least \(1-\eta\) its realized mean value \(V\) satisfies
\[
    \OPT^{\mathrm{ad}}-V\ \le\ 3KB+4H_{\max}\sqrt{2LB},
\]
where \(H_{\max}=\max_b H_b\) and \(L=\log(2KT/\eta)\). With equal blocks \(H_b=T/B\), the bound reads \(3KB+4\sqrt2\,T\sqrt{L/B}\).
\end{theorem}

The proof is in Appendix~\ref{app:ogbdq-proof}: feasibility is Theorem~\ref{thm:groupbf} applied to \(y_b\), and for regret, block-\(b\) planning has at least \(b\) samples per arm, so optimism and LP optimality bound the per-block gap by \(\min(1,2\sqrt{2L/b})H_b\); summing \(\sum_b b^{-1/2}\le2\sqrt B\) with \(3K\) per block for integrality and initialization gives the bound.

We first record that the rounding slack built into the plan polytope of OG-BDQ-UCB costs little whenever the floors admit a Slater-type interior.

\begin{proposition}[Comparator gap under a margin]
\label{prop:slater-gap}
Write \(V^{\mathrm{frac}}_b\) for the unslacked fractional group-fair optimum of block \(b\) (floors \(f_{b,g}\), budget exactly \(H_b\)) and \(V^\star_b=\max_{y\in P_b}\langle\mu,y\rangle+2K\mu_\star\) for the block-\(b\) term of \(\OPT^{\mathrm{ad}}\). Suppose block \(b\) admits a plan \(z\ge0\) with \(\sum_{i\in g}z_i\ge f_{b,g}+\sigma_b H_b\) for every group and \(\sum_i z_i\le(1-\kappa_b)H_b\), for margins \(\sigma_b,\kappa_b\in(0,1)\). Then
\[
    V^{\mathrm{frac}}_b-V^\star_b\ \le\ \max\Bigl\{\frac{t}{\sigma_b},\,\frac{2K}{\kappa_b}\Bigr\},
\]
so under uniform margins the total benchmark gap is \(O(B)\), independent of the block lengths.
\end{proposition}

The remaining development addresses the open overlapping case. Write \(d_g(\tau)\) for group \(g\)'s remaining deficit in its block just before round \(\tau\). Linear-programming duality turns blockwise group-fair regret into an exact per-pull ledger, for every group system.

\begin{proposition}[Dual ledger]
\label{prop:dual-ledger}
Fix a block \(b\) and any \(\omega_b\in\R\), \(\lambda_b\ge0\) that are dual feasible for the block's fractional program, meaning \(\omega_b-\sum_{g\ni i}\lambda_{b,g}\ge\mu_i\) for every arm \(i\), and define the reduced cost \(r_i=\omega_b-\mu_i-\sum_{g\ni i}\lambda_{b,g}\ge0\). Every blockwise-group-fair pull sequence satisfies
\[
    V^{\mathrm{frac}}_b-\sum_{\tau\in\calB_b}\mu_{A_\tau}
    \ \le\
    \sum_{\tau\in\calB_b}\Bigl[\,r_{A_\tau}\ +\!\!\sum_{g\ni A_\tau:\,d_g(\tau)=0}\!\!\lambda_{b,g}\Bigr],
\]
with equality when \((\omega_b,\lambda_b)\) is dual optimal, in which case arms in the support of an optimal fractional plan have zero reduced cost. Every pull's contribution is nonnegative.
\end{proposition}

The proof, in the appendix, is weak duality plus an exact-crediting identity: blockwise feasibility drives every deficit from \(f_{b,g}\) to zero, so each group is credited its price on exactly \(f_{b,g}\) rounds. Theorem~\ref{thm:disjoint} is the disjoint instantiation: there the optimal dual is explicit (\(\omega_b=\mu_\star\) and \(\lambda_{b,g}=\mu_\star-\mu^\star_g\)), the reduced cost of \(i\in g\) is the within-group gap \(\widetilde\Delta_i\), and D-BDQ-UCB controls both ledger columns, bounding reduced costs through the covering selection inequality and zeroing the waste column by stopping each group's covering pulls at exactly \(f_{b,g}\). That covering rule does not survive overlap:

\begin{proposition}[The disjoint rule fails under overlap]
\label{prop:greedy-fails}
There is a four-arm instance with two overlapping groups (arm degree \(t=2\)) on which any rule whose covering pulls select an index-maximizing member of some deficient group and whose surplus pulls select the global index argmax (D-BDQ-UCB applied verbatim) is exactly feasible yet, on the concentration event, suffers regret at least \(T/20-CL\) against \(\OPT^{\mathrm{frac}}\), for a universal constant \(C\).
\end{proposition}

The construction, proved in the appendix, is two floors sharing one arm whose mean sits slightly below its groupmates': the optimal cover accepts the lower mean because one unit of that arm's mass satisfies both floors at once, a multi-coverage value that the dual prices \(
\sum_{g\ni i}\lambda_g
\) encode and that a within-group index comparison cannot see. On per-arm floors there is no multi-coverage, which is why hard mandatory-first schedules remain competitive there (cf.\ the E-Scale stress battery in the experiments); under overlap, any correct algorithm must be LP-aware. LP-awareness alone is still not enough for our proof: the sampled plan must also satisfy explicit one-step cover and value-descent inequalities. We state these conditions openly rather than hiding them inside an invalid concavity argument.

The algorithm, P-BDQ-UCB (plan-sampling BDQ-UCB), re-solves the optimistic residual program after every pull and lets the plan choose the arm. Fix a measurable tie-breaking rule once and for all: among multiple LP optima or minimum covers, choose the lexicographically first extreme-point solution. At round \(\tau\) of block \(b\), with \(r(\tau)\) rounds and integer deficits \(d_g(\tau)\) remaining, the algorithm first computes the fractional cover slack
\[
\begin{aligned}
    S(\tau)&=r(\tau)-\mathrm{mc}(d(\tau)),\\
    \mathrm{mc}(d)&=\min\Bigl\{\textstyle\sum_i y_i:
    y\ge0,\ \sum_{i\in g}y_i\ge d_g\ \forall g\in\calG\Bigr\}.
\end{aligned}
\]
The terminal guard fires at most once per block: at the first round with \(S(\tau)\le 2K\), P-BDQ-UCB commits to the lexicographically first minimum fractional cover \(w(\tau)\) of \(d(\tau)\), executes the rounded cover \(\lceil w(\tau)\rceil\) over the following rounds, and spends any leftover rounds of the block by the index argmax. At every earlier round it solves
\[
\begin{aligned}
    \widehat y(\tau)\in\argmax\Bigl\{\textstyle\sum_iU_i(\tau)y_i:\ y\ge0,\
    &\sum_{i\in g}y_i\ge d_g(\tau)\ \forall g\in\calG,\\
    &\sum_iy_i=r(\tau)\Bigr\}
\end{aligned}
\]
with clipped indices \(U_i=\min\{1,\widehat\mu_i+\sqrt{2L/(N_i\vee1)}\}\) on all observations, and pulls \(A_\tau\sim\widehat y(\tau)/r(\tau)\). Thus the algorithm is defined independently of the analytical condition below: it samples from the optimistic plan until the guard fires. Proposition~\ref{prop:pbdq-feas} shows that the block-start slack condition \(H_b-\mathrm{mc}(f_b)>2K\) alone makes the guard's rounded cover fit deterministically; without it, P-BDQ-UCB is a heuristic and no guarantee is claimed. This is not a Beck--Fiala discrepancy layer; it is a terminal cover-rounding guard. The guard is motivated by a simple integrality obstruction: with groups \(\{a,b\},\{b,c\},\{c,a\}\), unit floors, two remaining rounds, and a surplus arm, the fractional program has slack \(0.5\) that no integral sequence can safely spend, since a surplus pull leaves one round against a fractional cover of \(1.5\). This example motivates the terminal guard; it does not rule out an unconditional \(\widetilde O(\sqrt{KT})\) plus block-additive integrality guarantee for overlap, which remains open.

The guard alone already yields an unconditional feasibility guarantee for arbitrary overlapping systems, with no discrepancy-rounding layer and no descent assumption.

\begin{proposition}[Pathwise feasibility of plan sampling]
\label{prop:pbdq-feas}
Suppose every block satisfies the initial cover-slack condition \(H_b-\mathrm{mc}(f_b)>2K\). Then P-BDQ-UCB satisfies every group floor in every block, pathwise and deterministically: each pull lowers \(\mathrm{mc}\) by at most one, so the guard fires no later than \(r=2K\) and at a state with \(S(\tau_c)>2K-1\), where the rounded minimum cover \(\lceil w(\tau_c)\rceil\) has size at most \(\mathrm{mc}(d(\tau_c))+K\le r(\tau_c)-K+1\) and fits in the remaining budget.
\end{proposition}

Turning to regret, for a deficit vector \(d\) write
\[
    d^+(A)_g = \max\{d_g-\mathbf 1\{A\in g\},0\}.
\]
For \(r\ge1\), define the true residual fractional value
\[
    \Psi_\mu(r,d)=\max\Bigl\{\langle\mu,y\rangle:
    y\ge0,
    \textstyle\sum_{i\in g}y_i\ge d_g\ \forall g,
    \sum_i y_i=r\Bigr\}.
\]
The following condition is the exact place where overlap is hard. It replaces the invalid argument that \(\mathrm{mc}\) is concave: \(\mathrm{mc}\) is convex as a function of deficits, since it is the pointwise maximum of feasible covering-dual linear forms.

\begin{definition}[Sampled-plan descent condition]
\label{def:sampled-cover-margin}
Fix \(\sigma\in(0,1)\) and \(\eta\in(0,1)\), set \(L=\log(2KT/\eta)\), and let \(\calE_L\) be the concentration event of the proof of Theorem~\ref{thm:main}, on which every index used by the algorithm satisfies
\[
\begin{aligned}
    \mu_i&\le U_i(\tau)\le \mu_i+\rho_i(\tau),\\
    \rho_i(\tau)&=\min\{1,2\sqrt{2L/(N_i(\tau)\vee1)}\}.
\end{aligned}
\]
A block satisfies the sampled-plan descent condition with margin \(\sigma\) at level \(\eta\) if its initial deficits satisfy \(H_b-\mathrm{mc}(f_b)\ge\sigma H_b\), and, on \(\calE_L\), at every pre-guard state generated by P-BDQ-UCB with remaining budget \(r\), deficits \(d\), sampling distribution \(p_i=\widehat y_i/r\), and radii \(\rho_i\), the following two one-step inequalities hold:
\begin{align}
    \E_{A\sim p}\bigl[\mathrm{mc}(d^+(A))\bigr]
    &\le \frac{r-1}{r}\,\mathrm{mc}(d),
    \label{eq:cert-cover}\\
    \E_{A\sim p}\bigl[\Psi_\mu(r-1,d^+(A))+\mu_A\bigr]
    &\ge \Psi_\mu(r,d)-\E_{A\sim p}[\rho_A].
    \label{eq:cert-value}
\end{align}
\end{definition}

Condition~\eqref{eq:cert-cover} is an online-cover contraction condition; Condition~\eqref{eq:cert-value} is an analytical true-value descent condition and is not computable without the unknown means. The definition is therefore a trajectory-level sufficient condition, not an implementable certificate. It is nonetheless non-vacuous:

\begin{proposition}[The descent condition holds for disjoint systems]
\label{prop:pbdq-disjoint}
Let the groups be pairwise disjoint. Then at every pre-guard state, inequality~\eqref{eq:cert-cover} holds unconditionally, and inequality~\eqref{eq:cert-value} holds on \(\calE_L\). Consequently, every disjoint instance whose blocks satisfy \(\sum_gf_{b,g}\le(1-\sigma)H_b\) and \(\sigma H_b>2K\) meets all hypotheses of Theorem~\ref{thm:pbdq}.
\end{proposition}

Under arbitrary overlap, by contrast, the two inequalities are real assumptions rather than consequences of LP optimism, as the following examples show.

\begin{remark}[Why the descent condition is nontrivial]
\label{rem:pbdq-counterexamples}
Two small examples explain why the condition is stated explicitly rather than proved. First, take groups \(g_1=\{a,c\}\) and \(g_2=\{b,c\}\), deficits \(d=(1,1)\), and remaining budget \(r=2\). Then \(\mathrm{mc}(d)=1\), because arm \(c\) covers both groups; yet a plan that loads \(a\) and \(b\) (exactly what optimism produces when \(a\) and \(b\) are fresh and \(c\) is not) is feasible, and sampling from it leaves \(\E[\mathrm{mc}(d^+(A))]=1\) in either branch, whereas \eqref{eq:cert-cover} demands \(1/2\). The cover-contraction inequality is thus a genuine restriction on which optimistic plans arise, not a consequence of Jensen's inequality: \(\mathrm{mc}\) is convex in \(d\), so Jensen runs the wrong way. Second, the radius-free descent \(\E[\Psi_\mu(r-1,d^+(A))]\ge\frac{r-1}{r}\Psi_\mu(r,d)\), which a homogeneity-plus-concavity argument would try to prove and which would deliver \eqref{eq:cert-value} through optimism, fails at basis-change kinks of the residual program: with the same two groups, a fourth arm \(x\) outside both, means \((0.6,0.6,0.3,1.0)\) for \((a,b,c,x)\), deficits \((5,5)\), and \(r=10\), the covering plan that loads \(a\) and \(b\) gives \(\E[\Psi_\mu(9,d^+(A))]=5.80\) while the homogeneity target is \(0.9\cdot6.5=5.85\). The failure occurs at a change of the optimal fractional basis, not from statistical error, so \eqref{eq:cert-value} cannot be derived from convexity or homogeneity of \(\Psi_\mu\) alone and is assumed with its optimism allowance \(\E_{A\sim p}[\rho_A]\).
\end{remark}

\begin{theorem}[Plan sampling under the descent condition]
\label{thm:pbdq}
Fix \(\eta\in(0,1)\) and \(L=\log(2KT/\eta)\). Suppose every block satisfies the sampled-plan descent condition with margin \(\sigma\) at level \(\eta\) (Definition~\ref{def:sampled-cover-margin}) together with the block-start slack condition
\[
    \sigma H_b>2K
    \qquad\text{for every block } b.
\]
Then, in addition to the pathwise feasibility guaranteed by Proposition~\ref{prop:pbdq-feas}, the expected regret of P-BDQ-UCB against the per-block fractional group-fair optimum satisfies
\[
    \OPT^{\mathrm{frac}}-\E[V]
    \le
    K+4\sqrt{2KTL}+2\eta T
    +CB\Bigl(\frac{L}{\sigma^2}+\frac{K}{\sigma}\Bigr),
\]
for a universal constant \(C\). With \(\eta=1/T\), the expected regret is
\[
    \widetilde O(\sqrt{KT})+
    \widetilde O\!\left(B\left(\frac{1}{\sigma^2}+\frac{K}{\sigma}\right)\right).
\]
\end{theorem}

\begin{proof}[Proof sketch]
The proof is in the appendix. Feasibility is Proposition~\ref{prop:pbdq-feas}, whose hypothesis follows from the margin clause and \(\sigma H_b>2K\); in particular the guard fires strictly after the block start and its rounded cover fits deterministically. For regret, on the concentration event \(\calE_L\), \eqref{eq:cert-cover} makes \(S(\tau)/r(\tau)\) a bounded-increment submartingale starting at least \(\sigma\); Azuma's inequality keeps the guard from firing until \(r=O(L/\sigma^2+K/\sigma)\). On the same event, \eqref{eq:cert-value} gives the per-pull inequality
\[
    \E[\Psi_\mu(\tau)-\Psi_\mu(\tau{+}1)-\mu_{A_\tau}\mid\mathcal F_\tau]
    \le
    \E[\rho_{A_\tau}\mid\mathcal F_\tau],
\]
which telescopes across sampled pulls. Summing the confidence radii over actually pulled arms gives \(K+4\sqrt{2KTL}\), and the committed tail contributes only the guard length. The concentration failure and slack-concentration failure contribute \(2\eta T\).
\end{proof}

Theorem~\ref{thm:pbdq} should be read as a conditional overlap theorem. It improves over the one-shot rate of Theorem~\ref{thm:ogbdq} only on trajectories where the sampled-plan descent condition holds (automatic in the disjoint case by Proposition~\ref{prop:pbdq-disjoint}) and it identifies the exact per-step inequalities needed to make LP-aware plan sampling work. D-BDQ-UCB remains the LP-free deterministic rule for disjoint systems, where no overlap pricing is needed. The overlapping case without the descent condition remains open: the triangle-with-surplus example shows why a terminal guard is needed, while Remark~\ref{rem:pbdq-counterexamples} shows why the natural optimistic-plan proof cannot be closed by a generic convexity or basis-stability argument.

\section{Extensions and Discussion}
\label{app:extensions}

\subsection{Terminal Repair for Slack-Rounded Policies}
\label{app:repair}

The main algorithms in the paper are exactly fair. However, some discrepancy rounding methods are designed to maintain small discrepancy at every prefix and may end with a small terminal deficit. The following lemma shows how to repair such a sequence.

\begin{lemma}[Terminal repair]
\label{lem:repair}
Let \(A_1,\dots,A_T\) be any preliminary pull sequence with counts \(N_i(T)\), and suppose \(N_i(T)\ge m-B\) for every \(i\). Then there exists another pull sequence \(\widetilde{A}_1,\dots,\widetilde{A}_T\) with counts \(\widetilde{N}_i(T)\ge m\) for every \(i\) such that the two sequences differ in at most \(KB\) positions. Consequently, because all rewards have means in \([0,1]\), the expected reward changes by at most \(KB\).
\end{lemma}

\begin{proof}
Define the deficit of arm \(i\) by \(d_i=(m-N_i(T))_+\). Since \(N_i(T)\ge m-B\), we have \(d_i\le B\) for every \(i\), and therefore \(d=\sum_{i=1}^K d_i\le KB\). The total number of pulls is fixed at \(T\ge Km\), so the total deficit among arms below \(m\) is matched by at least \(d\) surplus pulls among arms strictly above their required levels. Choose \(d_i\) positions currently assigned to surplus arms (never reducing a surplus arm below \(m\)) and change those positions to arm \(i\), for each deficient arm \(i\). After all changes, every arm has count at least \(m\). The number of changed positions is exactly \(d\le KB\), and changing one pull can change the expected reward by at most \(1\), so the total expected reward changes by at most \(KB\).
\end{proof}

\subsection{Alternative View: Fractional Frontier and Integral Rounding}
\label{app:fractional}

The continuous Pareto frontier in Theorem~\ref{thm:pareto} is expressed in terms of exposure fractions. For finite \(T\), the integral floor \(m\) induces the fraction \(\beta_T=m/T\). The best fair integral allocation is
\[
    n_i^\star(m)
    =
    \begin{cases}
        T-(K-1)m, & i=i^\star,\\
        m, & i\neq i^\star,
    \end{cases}
\]
and dividing by \(T\) gives the exposure vector
\[
    \frac{n_i^\star(m)}{T}
    =
    \begin{cases}
        1-(K-1)\beta_T, & i=i^\star,\\
        \beta_T, & i\neq i^\star.
    \end{cases}
\]
Thus the finite-horizon fair comparator lies exactly on the continuous Pareto frontier at \(\beta=\beta_T\), and the only difference between a desired fraction \(\delta\) and the implemented fraction \(\beta_T\) is the unavoidable integer-rounding error \(0\le \delta-\beta_T < 1/T\).

\subsection{Avoiding Penalty Tuning}
\label{app:penalty}

A Lagrangian approach would choose actions using an objective of the form
\[
    \text{reward estimate}
    -
    \lambda\cdot \text{fairness violation}.
\]
The difficulty is that the correct \(\lambda\) depends on the unknown reward gaps, the horizon, and the desired floor. If \(\lambda\) is too small, the floor may be violated; if it is too large, the policy may over-explore low-reward arms. DQ-UCB and BDQ-UCB avoid this tuning problem by combinatorially enforcing the quota before reward optimization begins; the only statistical tuning parameter is the usual UCB confidence radius.

\section{Details of the Deployment-Facing Studies}
\label{app:experiments-to-add}

The six studies summarized in the main text (Table~\ref{tab:deployment}) extend the clean stochastic and genre-exposure validation toward deployment-level coverage. We record their exact configurations here; all are reproduced over \(16\) seeds on CPU, reusing the same BDQ-UCB, blockwise-Lagrangian, global-quota, and block-floor routines as the main experiments.

\emph{Logged off-policy replay with propensities (N1).} A soft logging policy (arm \(0\) favored, \(\varepsilon=0.5\)) generates \(40{,}000\) pulls with known propensities; we evaluate the blockwise-fair target allocation off-policy via IPS, SNIPS, and a doubly-robust estimator (per-arm empirical model). All three match the on-policy value \(0.595\) (MAE \(0.0021\), \(0.0021\), \(0.0019\)).

\emph{Overlapping group floors (N2).} \(K=8\) arms, four overlapping groups (sliding arm windows), a per-block floor of \(0.20H\) on each group; a single pull credits every group its arm belongs to. The max-group-deficit rounding rule is feasible (\(0\) group violation); a group-blind per-arm/UCB policy violates by \(34.3\) pulls per block. This is the regime where genuine discrepancy rounding matters beyond cyclic quota scheduling. Table~\ref{tab:groupdisc} isolates the same separation on a controlled adversarial family: the half-integral plan of Proposition~\ref{prop:naive-group} with fractional part \(0.48\), over eight instances differing only in coordinate permutation and seed, where Beck--Fiala stays below \(t=2\) (\(1.04\to1.48\)) while naive rounding grows linearly in \(s\) (\(1.92\to11.52\), slope \(0.48\), correlation \(1.00\)); Figure~\ref{fig:groupdisc} plots the same separation. Because the plan is adversarial the table is a worst-case demonstration; a random-fractional-part variant shows the same qualitative separation with smaller constants: Beck--Fiala stays flat below \(t\) (\(0.93\to1.55\)) while naive rounding grows with \(s\) (\(0.86\to2.45\), correlation \(0.99\)).

\begin{table}[t]
\centering
\small
\begin{tabular}{lrrrr}
\toprule
grid & \(K\) & \(s\) & Beck--Fiala viol. & naive viol. \\
\midrule
\(2{\times}4\)  & \(8\)   & \(4\)  & \(1.04\) & \(1.92\)  \\
\(3{\times}6\)  & \(18\)  & \(6\)  & \(1.12\) & \(2.88\)  \\
\(4{\times}10\) & \(40\)  & \(10\) & \(1.20\) & \(4.80\)  \\
\(6{\times}16\) & \(96\)  & \(16\) & \(1.32\) & \(7.68\)  \\
\(8{\times}24\) & \(192\) & \(24\) & \(1.48\) & \(11.52\) \\
\bottomrule
\end{tabular}
\caption{Overlapping group floors (row/column set system over an \(a\times b\) grid; arm degree \(t=2\); group size \(s=\max(a,b)\)). Maximum group-exposure violation of Beck--Fiala null-space rounding versus naive nearest-integer rounding on the adversarial half-integral plan (fractional part \(0.48\)) of Proposition~\ref{prop:naive-group}, over eight instances differing only in permutation and seed. Beck--Fiala stays below the \(t=2\) bound of Theorem~\ref{thm:groupbf} at every scale, while naive rounding grows linearly in \(s\) (slope \(0.48\), corr.\ \(1.00\)): the discrepancy algorithm is the one that makes group floors achievable within the block budget.}
\label{tab:groupdisc}
\end{table}

\begin{figure}[t]
\centering
\placeholdergraphic[width=0.92\columnwidth]{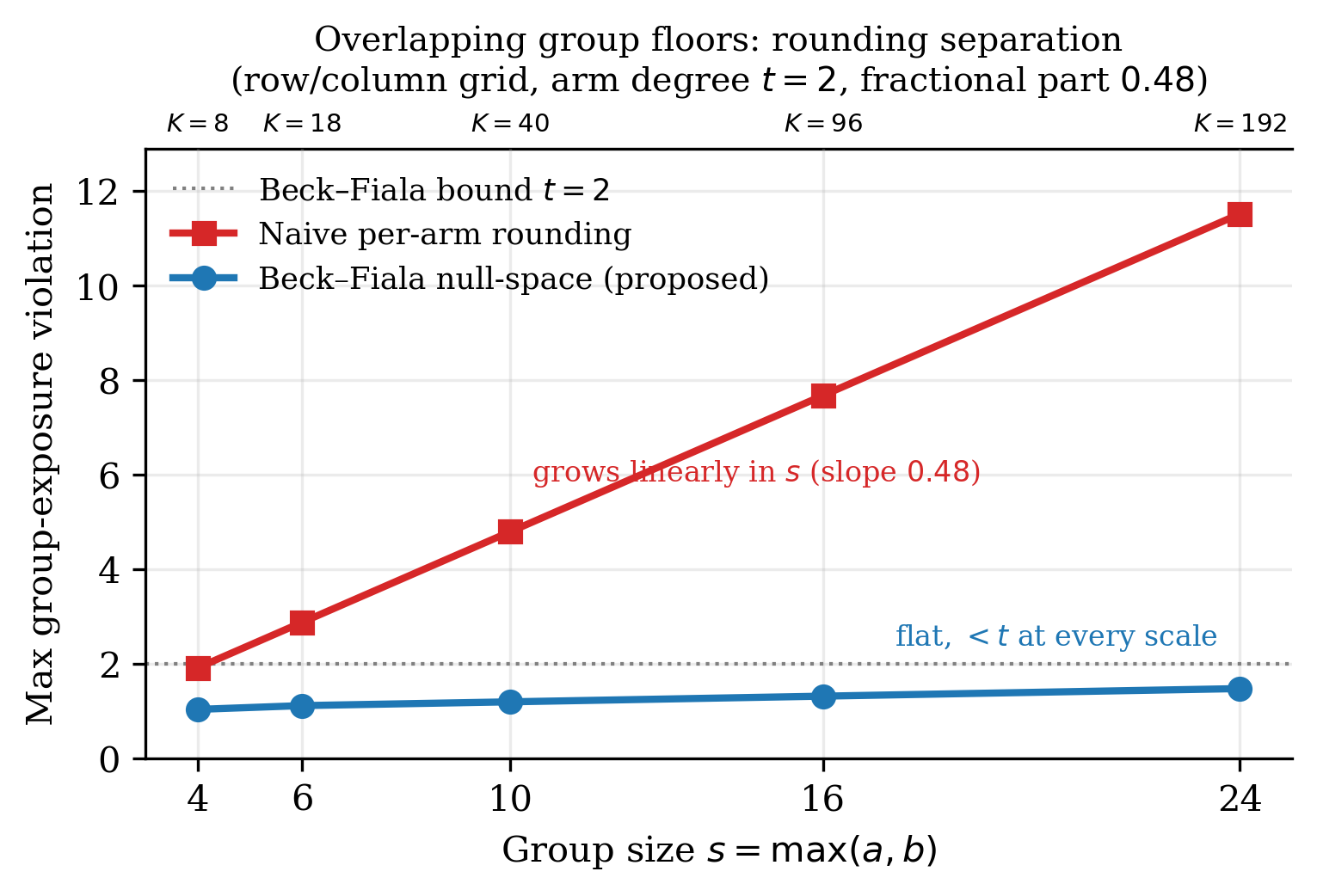}
\caption{Visualization of Table~\ref{tab:groupdisc}: maximum group-exposure violation versus group size \(s\) on the adversarial half-integral plan (row/column set system over an \(a\times b\) grid, arm degree \(t=2\), fractional part \(0.48\)). Beck--Fiala null-space rounding stays flat below the \(t=2\) bound of Theorem~\ref{thm:groupbf} at every scale (top axis reports the arm count \(K\) per grid), while naive per-arm rounding grows linearly in \(s\) (slope \(0.48\)). When \(t\ll s\) the discrepancy algorithm is unboundedly better than per-arm rounding.}
\label{fig:groupdisc}
\end{figure}

\emph{Adaptive online penalty (N3).} A deadline-aware blockwise Lagrangian that doubles \(\lambda\) after any infeasible block, starting from \(\lambda_0=0.5\). It reaches competitive regret (\(52.2\)) but incurs \(1.25\) violations on average during the ramp-up; BDQ-UCB has \(0\). Online tuning removes the offline grid but not the transient infeasibility.

\emph{Contextual provider exposure (N4).} LinUCB over \(d=4\) provider features with blockwise per-provider floors enforced by block-quota rounding; the residual rounds run LinUCB. Block feasibility is exact (\(0\) violation), so feasibility-by-construction composes with contextual learning.

\emph{Delayed and rolling blocks (N5).} First, reward feedback delayed by \(d\in\{0,50,200\}\) rounds: block floors stay exactly met (feasibility is scheduling, independent of feedback timing) while regret grows \(47.4\to69.1\to122.4\) (Table~\ref{tab:n5delay}). Second, rolling-window floors (\(0.08H\) per trailing window of length \(H\)): an eager deadline-aware rule keeps the worst trailing-window violation at most \(1\) pull.

\emph{Mandatory-sample ablation (N6).} At fixed residual budget, mandatory mass placed on hard-to-distinguish near-best arms (informative) versus a dummy worst arm (uninformative) gives residual regret \(188.1\pm7.7\) vs \(189.7\pm3.2\). The direction is consistent with informative mandatory pulls helping, confirming the worst-case \(R\)-only bound is pessimistic when the mandatory exposure happens to be informative; the effect is within seed noise on this instance and is reported as such.

Three quantities summarized in prose above are tabulated here in full: the residual-budget sweep behind the regret comparison of the synthetic blockwise experiment (Table~\ref{tab:regret-regime}), the real-rating MovieLens replay of the genre benchmark (Table~\ref{tab:realreplay}), and the delayed-feedback sweep of study N5 (Table~\ref{tab:n5delay}).

\begin{table}[t]
\centering\small
\begin{tabular}{rrrr}
\toprule
mand.\ frac. & \(R\) & BDQ-UCB & best feasible Lagr.\ (\(\lambda\)) \\
\midrule
\(0.03\) & \(4680\) & \(142.1\) & \(134.6\ (\lambda{=}8)\) \\
\(0.06\) & \(3384\) & \(70.3\)  & \(64.5\ (\lambda{=}4)\)  \\
\(0.10\) & \(1656\) & \(14.4\)  & \(12.8\ (\lambda{=}1)\)  \\
\(0.13\) & \(360\)  & \(0.24\)  & \(0.12\ (\lambda{=}4)\)  \\
\bottomrule
\end{tabular}
\caption{Residual-budget sweep on the synthetic blockwise instance (\(K{=}6\), \(B{=}12\), \(16\) seeds): the mandatory base fraction and the induced total residual budget \(R\) versus blockwise regret of BDQ-UCB and of the best \emph{feasible} Lagrangian over a six-point penalty grid (winning \(\lambda\) in parentheses). The tuned penalty is modestly ahead at every setting and the gap ($7.5\to5.8\to1.6\to0.12$) shrinks as the floor tightens and \(R\) falls; BDQ-UCB never strictly wins on regret, consistent with the parameter-free-feasibility (not lower-regret) scoping.}
\label{tab:regret-regime}
\end{table}

\begin{table}[t]
\centering\small
\begin{tabular}{lrr}
\toprule
method & blockwise regret \(\pm\) s.e.m. & block violation \\
\midrule
BDQ-UCB & \(102.4\pm1.3\) & \(0.0\) \\
Lagrangian, \(\lambda=0.5\) & \(100.7\pm1.3\) & \(34.6\) \\
Lagrangian, \(\lambda=2\) & \(101.6\pm1.3\) & \(0.0\) \\
Lagrangian, \(\lambda=8\) & \(101.8\pm1.1\) & \(0.0\) \\
\bottomrule
\end{tabular}
\caption{Real-rating MovieLens-100k replay (\(K{=}18\), \(B{=}12\), \(H{=}600\), \(16\) seeds): rewards are drawn non-parametrically from each genre's actual normalized \(1\)--\(5\) ratings rather than calibrated Bernoulli draws. BDQ-UCB is exactly feasible; the small penalty \(\lambda{=}0.5\) violates the floors (\(34.6\) missing pulls) while the feasible penalties tie it on regret. The calibrated-Bernoulli pattern of Table~\ref{tab:blockwise-movielens} persists under the true reward distribution.}
\label{tab:realreplay}
\end{table}

\begin{table}[t]
\centering\small
\begin{tabular}{lrr}
\toprule
feedback delay & blockwise regret \(\pm\) s.e.m. & block violation \\
\midrule
\(d=0\)   & \(47.4\pm3.4\)  & \(0.0\) \\
\(d=50\)  & \(69.1\pm3.6\)  & \(0.0\) \\
\(d=200\) & \(122.4\pm6.8\) & \(0.0\) \\
\bottomrule
\end{tabular}
\caption{Delayed-feedback sweep of study N5 (\(16\) seeds): block feasibility is exact at every delay (feasibility is a scheduling property, independent of feedback timing) while only regret degrades. The companion rolling-window floor (trailing window \(H\), floor \(0.08H{=}48\)) is kept to a worst-window violation of at most one pull by the eager deadline-aware rule.}
\label{tab:n5delay}
\end{table}

\begin{table}[t]
\centering
\small
\begin{tabular}{p{0.20\linewidth}p{0.62\linewidth}}
\toprule
Study & Outcome \\
\midrule
N1 Off-policy replay & Fair value recovered: MAE IPS/SNIPS \(0.0021\), DR \(0.0019\) (best) vs \(V_{\text{true}}{=}0.595\). \\
N2 Overlapping groups & Group-aware rounding \(0\) violation; group-blind UCB violates \(34.3\)/block. \\
N3 Adaptive penalty & Removes offline grid but pays \(1.25\) early violations; BDQ-UCB \(0\). \\
N4 Contextual (LinUCB) & Block-quota rounding + LinUCB: \(0\) block violation (exact feasibility). \\
N5 Delayed \& rolling & Feasibility unaffected by delay (\(0\) viol., \(d\le200\)); rolling-window viol.\ \(\le1\). \\
N6 Mandatory ablation & Informative mandatory pulls directionally lower residual regret (\(188.1\) vs \(189.7\), within seed noise): \(R\)-only bound pessimistic. \\
\bottomrule
\end{tabular}
\caption{Six deployment-facing studies (\(16\) seeds, CPU), reusing the BDQ-UCB, blockwise-Lagrangian, and block-floor primitives. N3 is a clean negative for the adaptive penalty baseline (transient infeasibility); N6's effect is directional only and within seed noise.}
\label{tab:deployment}
\end{table}

Two supporting figures accompany the extended experiments. Figure~\ref{fig:escale-lambda} plots the E-Scale feasibility--regret tradeoff against the penalty \(\lambda\), and Figure~\ref{fig:dbdq-scaling} plots the disjoint-family scaling test underlying Theorem~\ref{thm:disjoint}.

\begin{table}[t]
\centering\small
\resizebox{\columnwidth}{!}{%
\begin{tabular}{lrrc}
\toprule
method & regret vs.\ frac.\ opt. & group viol. & feasible \\
\midrule
OG-BDQ-UCB (learns plan)        & \(1083.4\pm7.6\) & \(0.0\) & yes \\
Group-BDQ-UCB (fixed plan)      & \(1174.3\pm0.8\) & \(0.0\) & yes \\
Group-Lagrangian (\(\lambda=4\))  & \(739.1\pm4.9\)  & \(0.0\) & yes (tuned) \\
Group-Lagrangian (\(\lambda=16\)) & \(741.5\pm7.7\)  & \(0.0\) & yes (tuned) \\
\midrule
\multicolumn{4}{l}{\emph{Disjoint control (genre-only floors, \(t=1\)):}}\\
D-BDQ-UCB (Thm.~\ref{thm:disjoint}) & \(1042.8\pm2.8\) & \(0.0\) & yes \\
\bottomrule
\end{tabular}%
}
\caption{E-Group on \(54\) MovieLens genre\(\times\)popularity-tier cells; \(18\) genre-union \(+\) \(3\) tier-union groups, arm degree \(t=2\) (top block); regret against the per-block fractional group-fair optimum. OG-BDQ-UCB learns a better plan than the fixed-plan comparator; the tuned group-Lagrangian is lower-regret when feasible. Beck--Fiala \emph{necessity} is isolated by the adversarial group construction (Table~\ref{tab:groupdisc}), not by this mild-overlap real instance. The disjoint control uses genre-only floors (\(t=1\)); D-BDQ-UCB is exactly feasible with no rounding, matching Theorem~\ref{thm:disjoint}.}
\label{tab:egroup}
\end{table}

\begin{figure*}[t]
\centering
\placeholdergraphic[width=\textwidth]{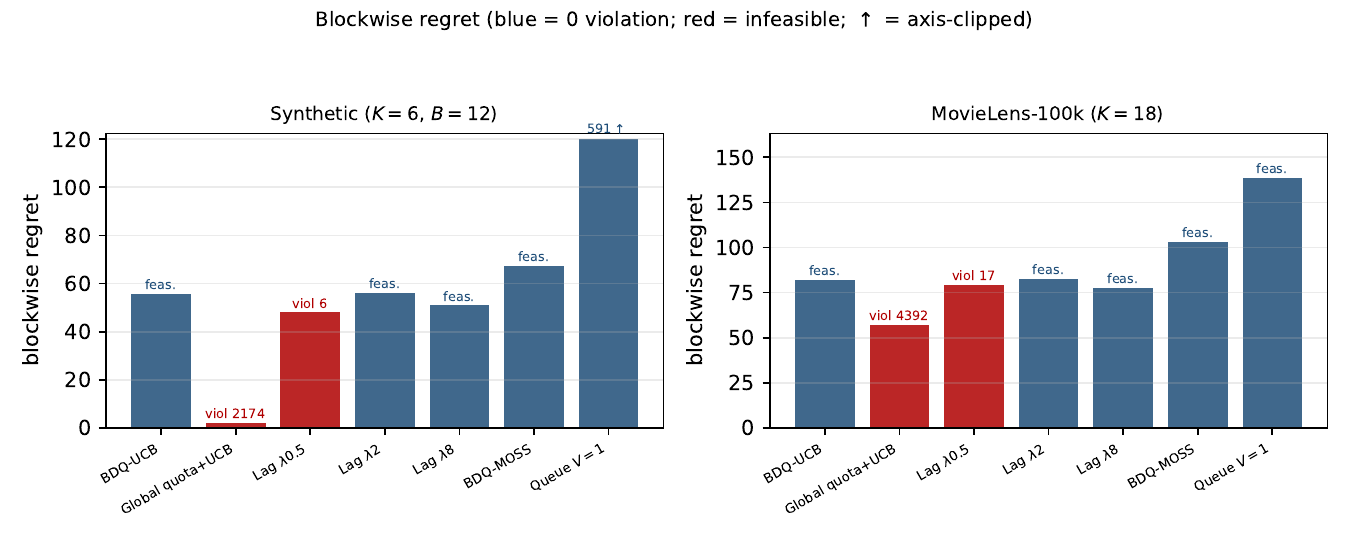}
\caption{Blockwise regret for each method (Table~\ref{tab:blockwise-all}; 16 seeds). Blue bars mark methods with zero block violations; red bars mark infeasible methods, annotated with their total number of missing mandatory pulls. BDQ-UCB achieves exact feasibility in both settings; Global quota+UCB has low apparent regret but violates 2174 (synthetic) and 4392 (MovieLens) mandatory pulls; the Lagrangian is feasible only at \(\lambda\ge2\) and requires per-instance penalty tuning.}
\label{fig:blockwise}
\end{figure*}

\begin{figure}[t]
\centering
\placeholdergraphic[width=0.82\columnwidth]{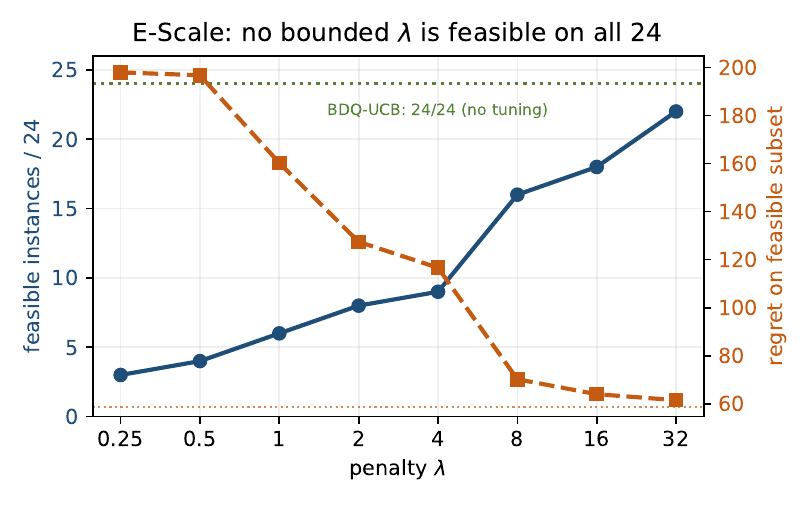}
\caption{E-Scale feasibility--regret tradeoff versus the penalty \(\lambda\) (Table~\ref{tab:escale}). Raising \(\lambda\) buys feasibility (blue, left axis: \(3\to22\) of \(24\)) only by trading away regret on the feasible subset (orange, right axis: \(198\to61\)); no bounded \(\lambda\) reaches all \(24\), which BDQ-UCB attains without tuning (green line).}
\label{fig:escale-lambda}
\end{figure}

\begin{figure}[t]
\centering
\placeholdergraphic[width=0.80\columnwidth]{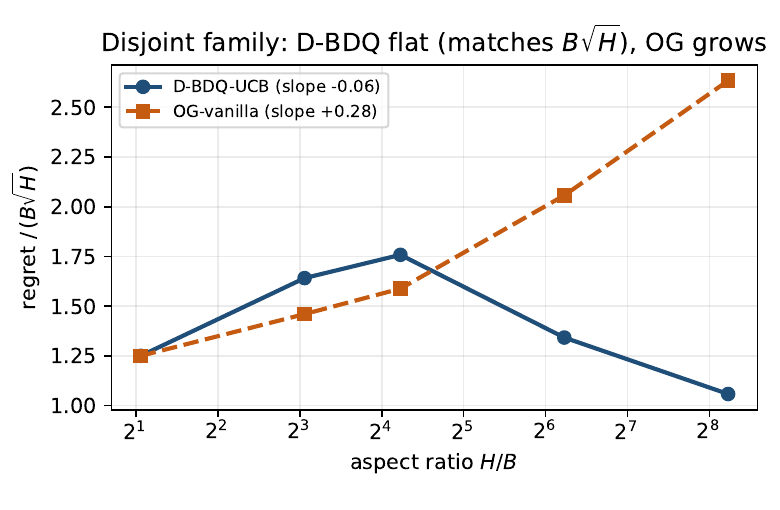}
\caption{Disjoint lower-bound family of Proposition~\ref{prop:adaptivity-lower}: regret normalized by the \(B\sqrt H\) lower-bound rate, versus the aspect ratio \(H/B\) at fixed \(T\). D-BDQ-UCB (Theorem~\ref{thm:disjoint}) stays flat (slope \(-0.06\), coefficient of variation \(0.18\)), matching \(B\sqrt H\) in every regime, while vanilla OG-BDQ-UCB grows (slope \(+0.28\)); the curves cross in the block-sparse regime \(H\gg B\) where OG's one-shot plan is loose.}
\label{fig:dbdq-scaling}
\end{figure}

\subsection{Semi-Real MovieLens Blockwise Table and the E-Scale Battery}
\label{app:escale-detail}

Table~\ref{tab:blockwise-movielens} is the semi-real MovieLens-100k counterpart to the synthetic blockwise Table~\ref{tab:blockwise-all}; the pattern matches in both. Genre reward means \(\widehat\mu_i\) are estimated by normalizing ratings to \([0,1]\) and averaging over the movies carrying each genre, dropping genres with fewer than \(50\) rated movies; the estimated means span \([0.554, 0.730]\), with Film-Noir highest.

\begin{table}[t]
\centering\small
\resizebox{\columnwidth}{!}{%
\begin{tabular}{lrrr}
\toprule
method & blockwise regret \(\pm\) s.e.m. & block viol. & min margin \\
\midrule
BDQ-UCB & \(82.2\pm1.7\) & \(0.0\) & \(0.0\) \\
Global quota+UCB & \(57.3\pm2.6\) & \(4391.8\) & \(-120.0\) \\
Lagrangian, \(\lambda=0.5\) & \(79.3\pm2.3\) & \(16.7\) & \(-4.0\) \\
Lagrangian, \(\lambda=2\) & \(82.8\pm2.9\) & \(0.0\) & \(0.0\) \\
Lagrangian, \(\lambda=8\) & \(77.7\pm1.9\) & \(0.0\) & \(0.0\) \\
BDQ-MOSS & \(103.2\pm5.0\) & \(0.0\) & \(0.0\) \\
Queue pacing, \(V=1\) & \(138.3\pm0.4\) & \(0.0\) & -- \\
\bottomrule
\end{tabular}%
}
\caption{Semi-real MovieLens-100k genre-exposure blockwise experiment (\(K{=}18\), \(B{=}12\), \(H{=}600\), real genre means, \(16\) seeds); companion to the synthetic Table~\ref{tab:blockwise-all}. BDQ-UCB is exactly feasible, Global quota+UCB violates by \(4392\) pulls, the Lagrangian needs \(\lambda\ge2\), and BDQ-MOSS pays a premium for discarding mandatory observations.}
\label{tab:blockwise-movielens}
\end{table}

For the E-Scale battery we generate \(24\) blockwise instances varying jointly in arm count \(K\in\{18,64,256\}\) with Bernoulli means sampled from the MovieLens-calibrated range \([0.554,0.730]\), in gap scale (means compressed toward their median by factors \(\{1,\tfrac12,\tfrac14\}\)), in floor tightness (mandatory fraction \(\{0.3,0.6,0.9\}\) of each block), and in block-length heterogeneity (\(H_b\) drawn log-uniformly from \([200,1800]\), \(B=12\)). For each fixed \(\lambda\) in the grid \(\{0.25,0.5,1,2,4,8\}\) we run the deadline-aware Lagrangian on all \(24\) instances with \(16\) seeds each and report the number of instances with zero block violations, the worst per-instance violation, and mean regret on the feasible subset; an oracle row selects the best feasible \(\lambda\) per instance, with tuning cost counted as grid runs consumed; a BDQ-UCB row reports the same metrics with no tuning. Over the eight-point grid \(\lambda\in\{0.25,\dots,32\}\), feasibility climbs monotonically with \(\lambda\) (\(3,4,6,8,9,16,18,22\) of \(24\)) but no bounded penalty clears more than \(22/24\); the smaller penalties leave worst-case violations in the hundreds to over a thousand missing pulls.

\begin{table}[t]
\centering\small
\resizebox{\columnwidth}{!}{%
\begin{tabular}{lrrr}
\toprule
method & feas.\ / 24 & worst viol. & regret (feas.) \\
\midrule
Lagrangian, \(\lambda=0.25\) & \(3\)  & \(1595\) & \(197.9\) \\
Lagrangian, \(\lambda=0.5\)  & \(4\)  & \(1415\) & \(196.7\) \\
Lagrangian, \(\lambda=1\)    & \(6\)  & \(1103\) & \(160.2\) \\
Lagrangian, \(\lambda=2\)    & \(8\)  & \(851\)  & \(127.4\) \\
Lagrangian, \(\lambda=4\)    & \(9\)  & \(552\)  & \(116.6\) \\
Lagrangian, \(\lambda=8\)    & \(16\) & \(243\)  & \(70.2\)  \\
Lagrangian, \(\lambda=16\)   & \(18\) & \(139\)  & \(64.0\)  \\
Lagrangian, \(\lambda=32\)   & \(22\) & \(96\)   & \(61.5\)  \\
deficit-greedy (\(\lambda\to\infty\)) & \(24\) & \(0.0\) & \(59.0\) \\
oracle \(\lambda\) (\(192\) runs) & \(22\) & \(0.0\) & \(60.6\) \\
BDQ-UCB, oracle-feasible subset & \(22\) & \(0.0\) & \(61.3\) \\
BDQ-UCB (no tuning, all \(24\)) & \(24\) & \(0.0\) & \(58.5\) \\
\bottomrule
\end{tabular}%
}
\caption{E-Scale stress battery (\(24\) instances \(\times\) \(16\) seeds). Over the natural grid no bounded penalty is feasible on more than \(22/24\) (only the \(\lambda\to\infty\) hard schedule reaches all \(24\)), and two instances admit no feasible \(\lambda\le32\). The tuned per-instance oracle edges BDQ-UCB on regret on the common feasible subset (\(60.6\) vs \(61.3\)); BDQ-UCB's advantage is tuning-free exact feasibility on all \(24\), not lower regret. Regret entries average over each method's feasible subset.}
\label{tab:escale}
\end{table}

\begin{figure}[t]
\centering
\placeholdergraphic[width=\columnwidth]{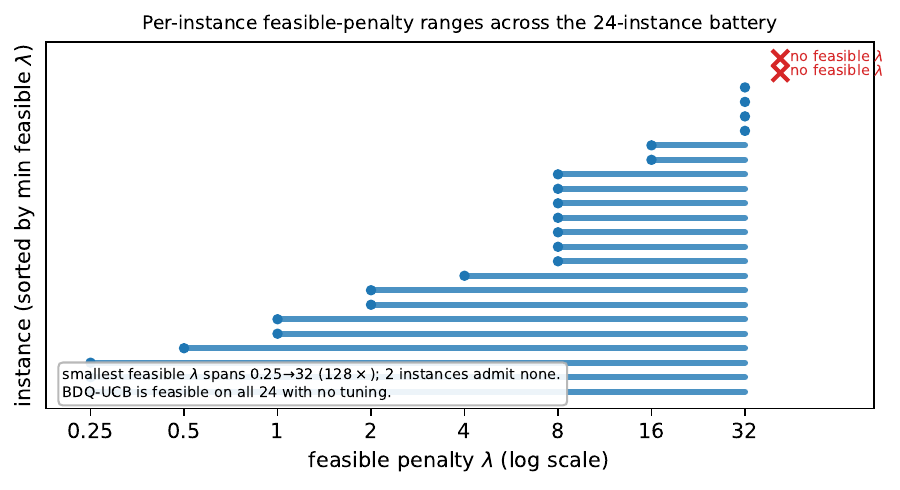}
\caption{Per-instance feasible-penalty ranges across the \(24\)-instance battery (log-\(\lambda\) axis, instances sorted by smallest feasible \(\lambda\)). The smallest feasible penalty spans a \(128\times\) range (\(0.25\) to \(32\)) and two tight, small-gap, large-\(K\) instances (red) admit no feasible \(\lambda\le32\); no single bounded penalty is simultaneously feasible across the battery, whereas BDQ-UCB is feasible on all \(24\) without tuning.}
\label{fig:escale-intervals}
\end{figure}

\section{Proofs}
\label{app:proofs}

\subsection{Proof of Proposition~\ref{thm:reduce}}
\label{app:discrep-proof}

We verify the identity's three claims by telescoping and linearity.

\begin{proof}
For the first claim,
\[
    D_T = \sum_{t=1}^T(e_{A_t}-x_t) = N(T)-S(T),
\]
so \(D_{T,i}=N_i(T)-S_i(T)\). If \(S_i(T)\ge m+B\) and \(\|D_T\|_\infty\le B\), then \(N_i(T)=S_i(T)+D_{T,i}\ge m\). For the second claim,
\[
    \sum_{t=1}^T \mu^\top x_t - \sum_{t=1}^T \mu_{A_t}
    = \mu^\top\!\sum_{t=1}^T x_t - \mu^\top\!\sum_{t=1}^T e_{A_t} = -\mu^\top D_T.
\]
The third claim is an immediate reformulation of the first two.
\end{proof}

\subsection{Proof of Lemma~\ref{lem:block-id}}
\label{app:identities}

The blockwise fair-regret identity rewrites fair regret as a gap-weighted count of pulls beyond each block's quota.

\begin{proof}
Since \(\sum_i N_i(T)=T\) and \(\Delta_{i^\star}=0\), \(\sum_{t=1}^T\mu_{A_t}=\sum_i N_i(T)\mu_i=T\mu_\star-\sum_{i\neq i^\star}\Delta_i N_i(T)\). Substituting into \(\widehat\Reg_{\mathbf m}(T)=\OPT_{\mathbf m}(T)-\sum_t\mu_{A_t}\) and using \(\OPT_{\mathbf m}(T)=T\mu_\star-\sum_b\sum_{i\neq i^\star}m_{b,i}\Delta_i\) gives the identity. Blockwise fairness gives \(N_i(T)=\sum_b N_{b,i}\ge\sum_b m_{b,i}\), so every term is nonnegative.
\end{proof}

\subsection{Proof of Theorem~\ref{thm:blockwise}}
\label{app:blockwise-proof}

We establish exact blockwise fairness of BDQ-UCB and its \(K+4\sqrt{2KRL}\) residual-regret bound, reusing the concentration event of Theorem~\ref{thm:main} on the nonmandatory rounds.

\begin{proof}
The quota schedule in block \(b\) pulls arm \(i\) exactly \(m_{b,i}\) times before any residual UCB pulls, so blockwise fairness holds deterministically.

The same concentration event \(\calE\) as in Theorem~\ref{thm:main} holds with probability at least \(1-\eta\). On \(\calE\), any residual UCB pull of a suboptimal arm \(i\) with \(n\) observations implies \(n\le 8L/\Delta_i^2\). Mandatory pulls are comparator-matched and do not contribute to blockwise fair regret. Hence only residual pulls of suboptimal arms contribute, at most \(1+8L/\Delta_i^2\) per arm beyond the at most one initialization pull absorbed by the additive \(K\), and at most \(R\) in total. Lemma~\ref{lem:block-id} gives
\[
    \widehat\Reg_{\mathbf m}(T)
    \le
    K+
    \sum_{i:\Delta_i>0}
    \min\!\left\{R\Delta_i,\frac{8L}{\Delta_i}\right\}.
\]
The gap-free bound follows by an \(\varepsilon\)-split replacing \(T\) by \(R\); choosing \(\varepsilon=\sqrt{8KL/R}\) gives \(K+4\sqrt{2KRL}\). If \(R=0\), regret is zero.
\end{proof}

\subsection{Proof of Theorem~\ref{thm:block-lower}}
\label{app:lower-proof}

We prove the \(\Omega(\sqrt{KR})\) minimax bound by reducing to a floor-free single-block bandit and applying a KL two-point argument.

\begin{proof}
Reduce to a single block with horizon \(T=R\) and zero floors; every policy is then blockwise fair, and this is a standard stochastic bandit problem. For each \(v\in[K]\), define instance \(P_v\) with arm \(v\) having mean \(1/2+\varepsilon\) and all others mean \(1/2\). Let \(P_0\) be the all-\(1/2\) instance and set \(a_v=\E_0 N_v(R)\). Regret on instance \(v\) is \(\varepsilon\,\E_v[R-N_v(R)]\). Averaging over \(v\),
\[
    \frac1K\sum_{v=1}^K \Reg_v^\pi(R)
    =
    \varepsilon R - \frac{\varepsilon}{K}\sum_{v=1}^K \E_v N_v(R).
\]
By the chain rule for KL divergence, the bound \(\mathrm{kl}(1/2,\,1/2+\varepsilon)\le 4\varepsilon^2\) valid for \(\varepsilon\le1/4\), and Pinsker's inequality,
\[
    \E_v N_v(R)
    \le
    a_v + R\varepsilon\sqrt{2a_v}.
\]
Summing over \(v\) and applying Cauchy's inequality to \(\sum_v\sqrt{a_v}\le\sqrt{KR}\),
\[
    \frac1K\sum_{v=1}^K \Reg_v^\pi(R)
    \ge
    \varepsilon R\!\left(1-\tfrac1K-\varepsilon\sqrt{\tfrac{2R}{K}}\right).
\]
Setting \(\varepsilon=(4\sqrt{2})^{-1}\sqrt{K/R}\) (truncated at \(1/4\)) gives \(c\sqrt{KR}\) for some universal \(c>0\); see also \citet{lattimore2020bandit} for the template.
\end{proof}

\subsection{Proof of Proposition~\ref{prop:naive-group}}
\label{app:naive-proof}

We exhibit a single-group instance on which nearest-integer rounding misses the floor by \(\Omega(s)\).

\begin{proof}
Take a single group \(g\) of size \(s\) and a fractional plan with \(x_i = c_i + (\tfrac12-\epsilon)\) for each \(i\in g\), with integer part \(c_i\) and fractional part \(\tfrac12-\epsilon\). Nearest-integer rounding sends every \(x_i\) down to \(c_i\), so \(\sum_{i\in g}(n_i-x_i)=-s(\tfrac12-\epsilon)=-\Omega(s)\). Setting the floor \(f_g=\sum_{i\in g} x_i\) makes the realized group exposure short by \(\Omega(s)\).
\end{proof}

\subsection{Proof of Theorem~\ref{thm:groupbf}}
\label{app:groupbf-proof}

We give the Beck--Fiala null-space rounding and bound every group's exposure error strictly below the arm degree \(t\) within the reserved budget.

\begin{proof}
Write \(x_i=\lfloor x_i\rfloor+f_i\) with \(f_i\in[0,1)\). It suffices to produce \(z\in\{0,1\}^K\) with \(|\sum_{i\in g}(z_i-f_i)|<t\) for every \(g\in\calG\) and to set \(n_i=\lfloor x_i\rfloor+z_i\): then \(n_i\in\{\lfloor x_i\rfloor,\lceil x_i\rceil\}\) for every \(i\) (coordinates with \(f_i=0\) are frozen at \(z_i=0\) from the start), the totals satisfy \(\sum_i\lfloor x_i\rfloor\le\sum_i n_i\le\sum_i\lceil x_i\rceil\le C\), and \(\sum_{i\in g}(n_i-x_i)=\sum_{i\in g}(z_i-f_i)\).

Initialize \(y\leftarrow f\). Call a coordinate \emph{alive} if \(y_i\in(0,1)\) and \emph{frozen} otherwise, and call a group \emph{active} if it contains more than \(t\) alive coordinates. While at least one coordinate is alive, consider the homogeneous linear system, in the alive coordinates, that fixes \(\sum_{i\in g,\ i\ \mathrm{alive}} y_i\) for every active group \(g\). If \(a\) coordinates are alive, each lies in at most \(t\) groups, so the number of alive-coordinate incidences is at most \(ta\); each active group accounts for more than \(t\) of these incidences, so the number of active groups is strictly less than \(a\). The system therefore has fewer equations than unknowns and admits a nonzero solution \(v\) supported on the alive coordinates. Move \(y\) along \(\pm v\) until some coordinate reaches \(0\) or \(1\), and freeze it there. Each iteration freezes at least one coordinate, so the process terminates after at most \(K\) iterations, each of which solves one linear system, giving polynomial time overall.

An active group's alive-coordinate sum is exactly invariant under every move, and its frozen coordinates never change, so its total \(\sum_{i\in g}y_i\) is invariant while the group is active. Once a group is released it has at most \(t\) alive coordinates, each of which subsequently moves within \((0,1)\) before freezing at \(0\) or \(1\), so each contributes strictly less than \(1\) of drift, and the group total drifts by strictly less than \(t\) in absolute value. Hence at termination \(|\sum_{i\in g}(z_i-f_i)|<t\) for every \(g\in\calG\).

Finally, for every group,
\[
    \sum_{i\in g}n_i
    =
    \sum_{i\in g}x_i+\sum_{i\in g}(z_i-f_i)
    >
    \sum_{i\in g}x_i-t
    \ge
    f_g,
\]
and since \(\sum_{i\in g}n_i\) is an integer it is at least \(\lceil f_g\rceil\ge f_g\). Every group floor is therefore met, and \(\sum_i n_i\le\sum_i\lceil x_i\rceil\le C\) keeps the allocation within the reserved budget.
\end{proof}

\subsection{Proof of Corollary~\ref{cor:groupbdq}}
\label{app:groupbdq-proof}

We show Group-BDQ-UCB is exactly group-feasible via Theorem~\ref{thm:groupbf} and inherits the residual-budget regret bound of Theorem~\ref{thm:blockwise}.

\begin{proof}
Feasibility: by Theorem~\ref{thm:groupbf}, the rounded plan \(n_b\) of every block meets every group floor, and \(\sum_i n_{b,i}\le\sum_i\lceil x_{b,i}\rceil\le H_b\), so the mandatory phase fits inside the block and \(R_b\ge0\).

For regret, let the comparator pull arm \(i\) exactly \(n_{b,i}\) times in block \(b\) and assign all \(R_b\) residual rounds of the block to a best arm. The algebra of Lemma~\ref{lem:block-id}, with \(m_{b,i}\) replaced by \(n_{b,i}\), expresses the realized fair regret against this comparator as \(\sum_{i\neq i^\star}\Delta_i\bigl(N_i(T)-\sum_b n_{b,i}\bigr)\), a gap-weighted count of residual pulls. The concentration event \(\calE\) of the proof of Theorem~\ref{thm:main} holds with probability at least \(1-\eta\); on \(\calE\), any residual selection of a suboptimal arm \(i\) with \(n\) prior observations forces \(n\le 8L/\Delta_i^2\), so arm \(i\) receives at most \(1+8L/\Delta_i^2\) residual pulls beyond the at most one initialization pull absorbed by the additive \(K\), and at most \(R\) residual pulls in total. The \(\varepsilon\)-split with \(\varepsilon=\sqrt{8KL/R}\) then yields the bound \(K+4\sqrt{2KRL}\), exactly as in the proof of Theorem~\ref{thm:blockwise}.
\end{proof}

\subsection{Proof of Proposition~\ref{prop:adaptivity-lower}}
\label{app:adaptivity-proof}

We lower-bound the price of plan adaptivity by \(c\,B\sqrt H\) through an Assouad-type averaging over sign patterns with a per-block Bretagnolle--Huber bound.

\begin{proof}
\emph{The family.} Arms are a star arm \(0\) and \(B\) disjoint pairs \(\{a_b,a_b'\}\), so \(K=2B+1\); groups are the pairs, so the arm degree is \(t=1\). There are \(B\) blocks of even length \(H\), and block \(b\) carries the single floor \(f_{b,g_b}=m:=H/2\) on its own pair, all other floors zero. The star arm is Bernoulli\((3/4)\) in every instance. For a sign vector \(s\in\{\pm1\}^B\), instance \(\nu_s\) sets \(a_b\sim\mathrm{Bernoulli}(\tfrac12+s_b\varepsilon)\) and \(a_b'\sim\mathrm{Bernoulli}(\tfrac12-s_b\varepsilon)\), with
\[
    \varepsilon\ :=\ \frac{1}{8\sqrt{2H}}\ \le\ \frac14 .
\]
Under every \(\nu_s\), the per-block fractional optimum places \(m\) on the good arm of the block's pair and \(H-m\) on the star, so \(\OPT^{\mathrm{frac}}=B\bigl[m(\tfrac12+\varepsilon)+(H-m)\tfrac34\bigr]\).

\emph{Regret is bad-arm mass.} Fix \(s\) and a block \(b\); write \(G,W,O,S\) for the block-\(b\) pulls of the good pair arm, the bad pair arm, all other pairs' arms, and the star, with \(G+W+O+S=H\) and \(G+W\ge m\) by blockwise group fairness. All non-star means are at most \(\tfrac12+\varepsilon\), so the block-\(b\) regret against the fractional optimum is at least
\[
\begin{aligned}
    &(G-m)\bigl(\tfrac14-\varepsilon\bigr)+W\bigl(\tfrac14+\varepsilon\bigr)+O\bigl(\tfrac14-\varepsilon\bigr)\\
    &\quad\ge\ -W\bigl(\tfrac14-\varepsilon\bigr)+W\bigl(\tfrac14+\varepsilon\bigr)\ =\ 2\varepsilon W ,
\end{aligned}
\]
using \(G-m\ge-W\) and \(O\ge0\). Hence \(\mathrm{reg}_b(s)\ge2\varepsilon\,\E_s[W_b]\), where \(W_b\) is the block-\(b\) pull count of the bad arm of pair \(b\), and every block's regret is nonnegative.

\emph{Per-coordinate two-point bound.} Fix \(b\) and \(s_{-b}\), and let \(P_\pm\) denote the trajectory laws over the whole horizon under \(s_b=\pm1\). With \(A:=\{N^{\mathrm{blk}}_{a_b}\ge m/2\}\), where \(N^{\mathrm{blk}}_{a_b}\) counts block-\(b\) pulls of \(a_b\): under \(s_b=-1\) the bad arm is \(a_b\) and \(\mathrm{reg}_b\ge2\varepsilon(m/2)P_-(A)\); under \(s_b=+1\) the bad arm is \(a_b'\) with \(N^{\mathrm{blk}}_{a_b'}\ge m-N^{\mathrm{blk}}_{a_b}\), so \(\mathrm{reg}_b\ge2\varepsilon(m/2)P_+(A^c)\). The Bretagnolle--Huber inequality gives
\[
\begin{aligned}
    \mathrm{reg}_b(s_{-b},-)+\mathrm{reg}_b(s_{-b},+)
    &\ \ge\ \varepsilon m\bigl(P_-(A)+P_+(A^c)\bigr)\\
    &\ \ge\ \frac{\varepsilon m}{2}\,e^{-\mathrm{KL}(P_-,P_+)} .
\end{aligned}
\]
The two laws differ only on the arms of pair \(b\), so by the chain rule and \(\mathrm{kl}(\tfrac12\pm\varepsilon,\tfrac12\mp\varepsilon)\le16\varepsilon^2\) for \(\varepsilon\le\tfrac14\),
\[
\begin{aligned}
    \mathrm{KL}(P_-,P_+)&\ \le\ 16\varepsilon^2\,\E_{(s_{-b},-)}\bigl[M_b\bigr],\\
    M_b&:=\text{total horizon pulls of pair }b .
\end{aligned}
\]

\emph{Markov selection.} Let \(q_b:=2^{-B}\sum_s\E_s[M_b]\). The pairs are disjoint, so \(\sum_bM_b\le T=BH\) pathwise and \(\sum_bq_b\le BH\); by Markov's inequality at most \(B/4\) coordinates have \(q_b>4H\), so at least \(3B/4\) coordinates satisfy \(q_b\le4H\). For such a coordinate, averaging over \(s_{-b}\),
\[
    2^{-(B-1)}\!\!\sum_{s_{-b}}\mathrm{KL}(P_-,P_+)\ \le\ 16\varepsilon^2\cdot2q_b\ \le\ 128\,\varepsilon^2H\ =\ 1 ,
\]
by the choice of \(\varepsilon\), and Jensen's inequality gives \(2^{-(B-1)}\sum_{s_{-b}}e^{-\mathrm{KL}}\ge e^{-1}\).

\emph{Assembling.} Averaging the two-point bound over \(s_b\) and then over \(s_{-b}\), each coordinate \(b\) with \(q_b\le4H\) contributes at least \((\varepsilon m/4)e^{-1}\) to the sign-averaged total regret, so
\[
\begin{aligned}
    2^{-B}\sum_s\bigl[\OPT^{\mathrm{frac}}-\E_s V^\pi\bigr]
    &\ \ge\ \frac{3B}{4}\cdot\frac{\varepsilon m}{4}\,e^{-1}\\
    &\ =\ \frac{3}{256\sqrt2\,e}\,B\sqrt H ,
\end{aligned}
\]
and the supremum over the family dominates the average, proving the claim with \(c=3/(256\sqrt2\,e)\). Finally, on the same instances the fixed plan that places the block's floor mass on \(a_b\) (with unit slack) is admissible for Corollary~\ref{cor:groupbdq}, giving regret \(O(\sqrt{KRL})\) against the fixed-plan comparator.
\end{proof}

\subsection{Proof of Theorem~\ref{thm:disjoint}}
\label{app:disjoint-proof}

We prove exact feasibility of D-BDQ-UCB and a \(4K+8\sqrt{KTL}\) regret bound by decomposing the per-block optimum across disjoint groups.

\begin{proof}
\emph{Feasibility.} Group \(g\) receives exactly \(f_{b,g}\) covering pulls in block \(b\), so every floor is met exactly and deterministically; \(\sum_g f_{b,g}\le H_b\) leaves the surplus nonnegative.

\emph{The comparator decomposes.} Write \(\mu^\star_g=\max_{i\in g}\mu_i\). For any \(y\ge0\) with \(\sum_{i\in g}y_i=G_g\ge f_{b,g}\) and \(\sum_iy_i=H_b\), disjointness gives
\[
\begin{aligned}
    \langle\mu,y\rangle
    &\ \le\ \sum_g\mu^\star_g G_g+\mu_\star\Bigl(H_b-\sum_gG_g\Bigr)\\
    &\ \le\ \sum_g f_{b,g}\,\mu^\star_g+\mu_\star\Bigl(H_b-\sum_g f_{b,g}\Bigr),
\end{aligned}
\]
using \(\mu^\star_g\le\mu_\star\) for the second inequality; the value is attained by covering each floor with a best member of its group and assigning the surplus to a best arm. Hence \(\OPT^{\mathrm{frac}}=\sum_b[\sum_gf_{b,g}\mu^\star_g+(H_b-\sum_gf_{b,g})\mu_\star]\), and with \(\widetilde\Delta_i=\mu^\star_g-\mu_i\) for \(i\in g\),
\[
    \OPT^{\mathrm{frac}}-V\ =\ \sum_{g}\sum_{i\in g}\widetilde\Delta_i\,N^{\mathrm{cov}}_i
    \ +\ \sum_{i\ne i^\star}\Delta_i\,N^{\mathrm{sur}}_i ,
\]
where \(N^{\mathrm{cov}}_i\) and \(N^{\mathrm{sur}}_i\) count covering and surplus pulls of arm \(i\) over the horizon.

\emph{Counting.} Let \(\calE\) be the concentration event of the proof of Theorem~\ref{thm:main}, of probability at least \(1-\eta\). At most one covering selection of any arm occurs with zero prior samples. On \(\calE\), a covering selection of \(i\in g\) with \(n\ge1\) prior samples requires \(U_i\ge U_{i^\star_g}\ge\mu^\star_g\), hence \(\mu_i+2\sqrt{2L/n}\ge\mu^\star_g\) and \(n\le8L/\widetilde\Delta_i^2\); extra samples from surplus pulls or other blocks only increase \(n\), so \(N^{\mathrm{cov}}_i\le1+8L/\widetilde\Delta_i^2\), while \(\sum_{i\in g}N^{\mathrm{cov}}_i=F_g\). Splitting group \(g\) at \(\varepsilon_g=\sqrt{8|g|L/F_g}\),
\[
\begin{aligned}
    \sum_{i\in g}\widetilde\Delta_i N^{\mathrm{cov}}_i
    &\ \le\ \varepsilon_g F_g+\sum_{i\in g:\,\widetilde\Delta_i\ge\varepsilon_g}\Bigl(1+\frac{8L}{\widetilde\Delta_i}\Bigr)\\
    &\ \le\ 2|g|+4\sqrt{2\,|g|F_gL}.
\end{aligned}
\]
The surplus term is the residual argument of Theorem~\ref{thm:blockwise} verbatim: on \(\calE\) it is at most \(2K+4\sqrt{2KR'L}\).

\emph{Assembling.} Summing over groups, using \(\sum_g|g|\le K\) (disjointness) and Cauchy--Schwarz,
\[
\begin{aligned}
    \sum_g\sqrt{|g|F_g}+\sqrt{KR'}
    &\ \le\ \sqrt{\Bigl(\sum_g|g|+K\Bigr)\Bigl(\sum_gF_g+R'\Bigr)}\\
    &\ \le\ \sqrt{2KT},
\end{aligned}
\]
which gives \(\OPT^{\mathrm{frac}}-V\le4K+4\sqrt{2L}(\sum_g\sqrt{|g|F_g}+\sqrt{KR'})\le4K+8\sqrt{KTL}\) on \(\calE\).

\emph{Consequences.} On the family of Proposition~\ref{prop:adaptivity-lower}: \(|g|=2\), \(F_{g_b}=H/2\), \(K=2B+1\), \(R'=BH/2\), so \(\sum_g\sqrt{|g|F_g}=B\sqrt H\) and \(\sqrt{KR'}\le\sqrt{3B\cdot BH/2}\le2B\sqrt H\), giving \(O(\sqrt L\cdot B\sqrt H)\) against the \(\Omega(B\sqrt H)\) lower bound in every regime. With singleton groups, \(\widetilde\Delta_i=0\) annihilates the covering term, the strong comparator coincides with the plan-matched one, and the bound reduces to the surplus term of Theorem~\ref{thm:blockwise}.
\end{proof}

\subsection{Proof of Lemma~\ref{lem:static-id}}
\label{app:identities-static}

The static fair-regret identity is the \(B=1\) specialization of Lemma~\ref{lem:block-id}, rewriting fair regret as a gap-weighted count of pulls beyond the quota.

\begin{proof}
Since \(\sum_i N_i(T)=T\) and \(\Delta_{i^\star}=0\),
\[
    \sum_{t=1}^T\mu_{A_t}=\sum_{i=1}^K N_i(T)\mu_i=T\mu_\star-\sum_{i\neq i^\star}\Delta_i N_i(T).
\]
Substituting into \(\widehat\Reg_m(T)=\OPT_m(T)-\sum_t\mu_{A_t}\) and using \(\OPT_m(T)=T\mu_\star-m\sum_{i\neq i^\star}\Delta_i\) gives the identity. Under \(m\)-fairness every term \(\Delta_i(N_i(T)-m)\) is nonnegative.
\end{proof}

\subsection{Proof of Lemma~\ref{lem:balanced}}
\label{app:quota-proof}

We bound the cyclic-schedule discrepancy by one via a prefix-count argument.

\begin{proof}
After \(m\) complete cycles, each arm has been pulled exactly once per cycle, so \(N_i(Km)=m\). For a prefix \(t\), write \(t=qK+r\) with \(q\ge0\) and \(0\le r<K\). Each arm has been pulled either \(q\) or \(q+1\) times. Since \(t/K=q+r/K\),
\[
    \left|N_i(t)-\frac{t}{K}\right|
    \le
    \max\!\left\{\frac{r}{K},1-\frac{r}{K}\right\}
    \le 1.
\]
\end{proof}

\subsection{Proof of Theorem~\ref{thm:main}}
\label{app:main-proof}

We prove exact fairness deterministically and the \(K+4\sqrt{2KTL}\) regret bound on a Hoeffding concentration event.

\begin{proof}
The exact fairness claim is deterministic. If \(m\ge1\), the first phase pulls every arm exactly \(m\) times, so \(N_i(T)\ge m\) for every arm. If \(m=0\), the constraint is vacuous.

Define the event
\[
    \calE
    =
    \left\{
    \forall i\in[K],\ \forall n\in\{1,\dots,T\}:
    \left|\widehat{\mu}_{i,n}-\mu_i\right|
    \le
    \sqrt{\frac{2L}{n}}
    \right\},
\]
where \(\widehat{\mu}_{i,n}\) is the empirical mean of the first \(n\) rewards from arm \(i\), and \(L=\log(2KT/\eta)\). By Hoeffding's inequality and a union bound over \(KT\) pairs,
\[
    \Prob(\calE^c)
    \le
    2KT\exp(-4L)
    \le
    \eta.
\]

Assume \(\calE\) holds. Fix a suboptimal arm \(i\neq i^\star\) selected during the UCB phase with \(n=N_i(t)\) observations. Since \(i\) maximizes the UCB index and on \(\calE\) we have \(\widehat{\mu}_{i^\star}+\sqrt{2L/N_{i^\star}}\ge\mu_\star\) and \(\widehat{\mu}_i+\sqrt{2L/n}\le\mu_i+2\sqrt{2L/n}\),
\[
    \Delta_i
    \le
    2\sqrt{\frac{2L}{n}},
    \qquad\text{so}\qquad
    n
    \le
    \frac{8L}{\Delta_i^2}.
\]
Let \(M_i=(N_i(T)-m)_+\) be the number of UCB pulls of arm \(i\) beyond the quota. Then \(M_i\le 1+8L/\Delta_i^2\). Using Lemma~\ref{lem:static-id},
\[
    \widehat{\Reg}_m(T)
    \le
    K
    +
    \sum_{i:\Delta_i>0}
    \min\!\left\{
        T\Delta_i,
        \frac{8L}{\Delta_i}
    \right\}.
\]
For the gap-free bound, split arms at threshold \(\varepsilon\):
\[
    \widehat{\Reg}_m(T)
    \le
    K+\varepsilon T+\frac{8KL}{\varepsilon}.
\]
Choosing \(\varepsilon=\sqrt{8KL/T}\) gives \(K+4\sqrt{2KTL}\). The expected-regret bound follows by setting \(\eta=1/T\) and bounding the \(\calE^c\) contribution by \(1\).
\end{proof}

Figure~\ref{fig:regret-vs-T} provides empirical confirmation: fair regret stays inside the \(O(\sqrt{KT\log(KT)})\) envelope and vanishes at large \(T\).

\begin{figure*}[h]
\centering
\placeholdergraphic[width=\textwidth]{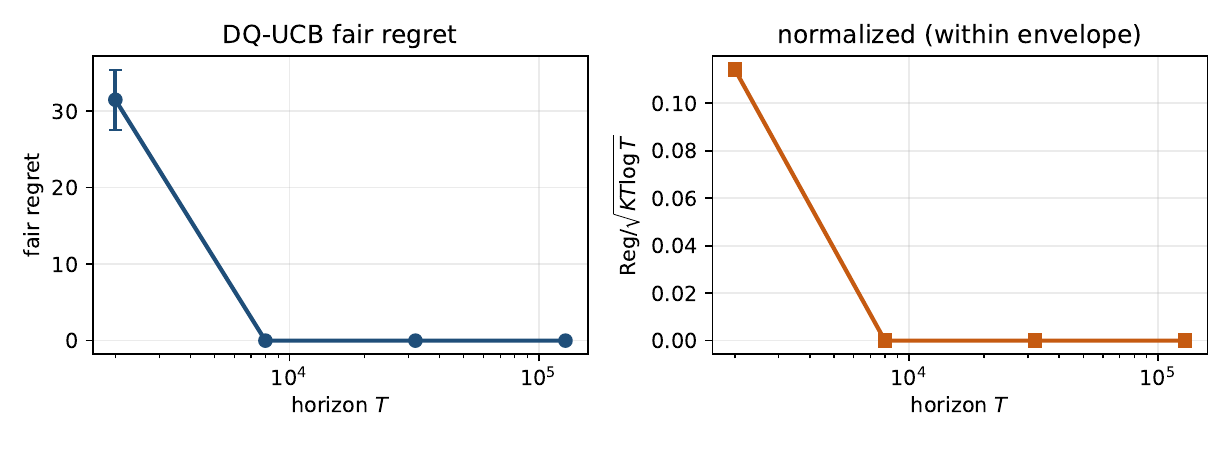}
\caption{E1: Fair regret (left) and normalized regret \(\mathrm{Reg}/\sqrt{KT\log T}\) (right) for DQ-UCB vs.\ horizon \(T\) (\(K=5\), \(\delta=0.1\), 16 seeds). Fair regret drops from \(31.5\) at \(T=2{,}000\) to zero for \(T\ge8{,}000\); the normalized regret stays within the \(O(\sqrt{KT\log(KT)})\) envelope of Theorem~\ref{thm:main} (it is the regret upper bound that is verified, not a matched growth rate, since realized fair regret is zero at large \(T\)).}
\label{fig:regret-vs-T}
\end{figure*}

\subsection{Proof of Corollary~\ref{cor:realized}}
\label{app:realized-proof}

We pass from pseudo-regret to realized reward via an Azuma--Hoeffding martingale bound.

\begin{proof}
Theorem~\ref{thm:main} bounds \(\OPT_m(T)-\sum_t\mu_{A_t}\). Since rewards lie in \([0,1]\), the differences \(X_{A_t}-\mu_{A_t}\) form a bounded martingale difference sequence. By Azuma--Hoeffding, with probability at least \(1-\eta\),
\[
    \sum_{t=1}^T \mu_{A_t}-\sum_{t=1}^T X_{A_t}
    \le
    \sqrt{2T\log(1/\eta)}.
\]
A union bound with the event of Theorem~\ref{thm:main} gives the claim.
\end{proof}

\subsection{Proof of Theorem~\ref{thm:pareto}}
\label{app:pareto-proof}

We identify the value-maximizing allocation at each fairness level and show the resulting curve is Pareto-efficient.

\begin{proof}
Fix \(\beta\in[0,1/K]\). Any \(p\) with \(\phi(p)\ge\beta\) assigns at least \(\beta\) to every arm. Among such allocations, the unique reward maximizer assigns exactly \(\beta\) to every suboptimal arm and the residual \(1-(K-1)\beta\) to \(i^\star\); this is \(p^\beta\). Its value is
\[
    V(p^\beta)
    =
    (1-(K-1)\beta)\mu_\star
    +
    \beta\sum_{i\neq i^\star}\mu_i
    =
    \mu_\star - \beta\sum_{i\neq i^\star}\Delta_i.
\]
Since \(p^\beta\) is the unique value maximizer at fairness level \(\beta\) and \(V(p^\beta)\) is strictly decreasing in \(\beta\) when \(i^\star\) is unique, no point on the curve is dominated. Conversely, any Pareto-efficient \(p\) with \(\phi(p)=\beta\) must equal \(p^\beta\).
\end{proof}

\subsection{Proof of Corollary~\ref{cor:frontier}}
\label{app:frontier-cor}

We match DQ-UCB's guarantee to the integral Pareto frontier and divide by \(T\).

\begin{proof}
The best fair integral allocation for floor \(m\) assigns \(m\) pulls to each suboptimal arm and all remaining pulls to \(i^\star\), which is the integral version of \(p^{m/T}\) on the continuous Pareto frontier. Theorem~\ref{thm:main} bounds DQ-UCB's regret against this allocation; dividing by \(T\) gives convergence of average reward to \(V(p^{m/T})\) at rate \(O(\sqrt{K\log(KT)/T})\).
\end{proof}

\subsection{Proof of Theorem~\ref{thm:bdqmoss}}
\label{app:bdqmoss-proof}

We embed the residual rounds as a MOSS game of horizon \(R\), giving the \(C\sqrt{KR}+K\) bound and, via doubling, the unknown-\(R\) case.

\begin{proof}
Fairness is immediate: the mandatory schedule of BDQ-UCB is unchanged, so Theorem~\ref{thm:blockwise}(i) applies verbatim. If \(R<K\), the fair regret is at most \(R\le\sqrt{KR}\) deterministically and the bound holds with any \(C\ge1\); assume \(R\ge K\).

The positions of the residual rounds are deterministic given the block structure and floors: block \(b\) consists of its mandatory schedule followed by \(R_b\) residual rounds. Let \(\tau(1)<\dots<\tau(R)\) enumerate the residual rounds, and for \(r\le R\) let \(A'_r=A_{\tau(r)}\) and \(X'_r\) denote the arm pulled and the reward observed at the \(r\)th residual round. The residual rule of BDQ-MOSS is a function of \((A'_1,X'_1,\dots,A'_{r-1},X'_{r-1})\) only, and, because rewards are independent across pulls and identically distributed within each arm, the conditional law of \(X'_r\) given the residual past and \(A'_r=i\) is \(\nu_i\), regardless of how many mandatory pulls occurred in between. Hence the process \((A'_r,X'_r)_{r\le R}\) has exactly the law of the MOSS algorithm with horizon parameter \(R\) interacting with the bandit \((\nu_1,\dots,\nu_K)\) for \(R\) rounds. By Lemma~\ref{lem:block-id}, the blockwise fair regret of BDQ-MOSS equals \(\sum_{i\ne i^\star}\Delta_i N^{\mathrm{res}}_i(R)=\sum_{r=1}^{R}\Delta_{A'_r}\), the pseudo-regret of the embedded game. The distribution-free guarantee for MOSS \citep{audibert2009minimax} (see also \citealt{lattimore2020bandit}, Chapter 9) bounds its expectation by \(C\sqrt{KR}\) for a universal constant \(C\), with the \(K\) initialization pulls contributing at most \(K\); this proves the upper bound. The matching \(\Omega(\sqrt{KR})\) is Theorem~\ref{thm:block-lower}, and \(\sqrt{KR}\ge K\) for \(R\ge K\), so the additive \(K\) does not affect the rate.

If the floors are revealed only at block starts, \(R\) is unknown in advance. Run the residual rule in phases \(j=0,1,2,\dots\), where phase \(j\) spans the next \(2^j\) residual rounds, restarting the residual statistics and using horizon parameter \(2^j\) within the phase. Each phase is itself an embedded MOSS game by the argument above, so phase \(j\) contributes expected regret at most \(C\sqrt{K2^j}+K\), and at most \(\lceil\log_2(2R)\rceil\) phases begin, so the total expected fair regret is at most \(C'\sqrt{KR}+K\lceil\log_2(2R)\rceil\) for a universal \(C'\), preserving the rate up to a logarithmic additive initialization term.
\end{proof}

\subsection{Proof of Proposition~\ref{prop:positive-floor-lower}}
\label{app:positive-floor-proof}

We force \(\Omega(\sqrt{KR})\) regret even under positive mandatory exposure, since the mandatory pulls are comparator-matched and carry no information.

\begin{proof}
Fix \(K\ge2\), \(R\ge K\), and an integer \(M\ge0\). The instance has two blocks: block one has length \(H_1=M\) and floor vector \(m_1=(M,0,\ldots,0)\), so its entire length is mandatory mass on arm \(1\); block two has length \(H_2=R\) with zero floors, so the total residual budget is exactly \(R\) and the total mandatory budget is \(M\). Throughout, arm \(1\) is Bernoulli\((1/2)\) under every instance considered. Three observations drive the argument. First, because \(H_1=m_{1,1}\), every blockwise-fair policy pulls arm \(1\) in every round of block one, and so does the blockwise-fair comparator; by Lemma~\ref{lem:block-id}, these \(M\) pulls cancel from fair regret under every instance. Second, the law of those observations is the same under all instances, so by the chain rule they contribute zero to the KL divergence between the trajectory distributions of any two instances; all information comes from residual pulls. Third, under any instance whose unique best arm is \(v\), every residual pull of an arm other than \(v\) costs at least the minimum gap in fair regret.

Consider first \(K\ge3\). Let \(P_0\) give every arm mean \(1/2\), and for \(v\in\{2,\ldots,K\}\) let \(P_v\) give arm \(v\) mean \(1/2+\varepsilon\) and every other arm mean \(1/2\), with \(\varepsilon\in(0,1/4]\) chosen below. Write \(N_v^{\mathrm{res}}\) for the number of residual pulls of arm \(v\) and \(a_v=\E_0[N_v^{\mathrm{res}}]\), so \(\sum_{v=2}^{K}a_v\le R\). Under \(P_v\), fair regret satisfies \(\Reg_v\ge\varepsilon\,\E_v[R-N_v^{\mathrm{res}}]\). The chain rule gives \(\mathrm{KL}(P_0^\pi,P_v^\pi)=a_v\,\mathrm{kl}(1/2,1/2+\varepsilon)\le4\varepsilon^2 a_v\), and Pinsker's inequality yields \(\E_v[N_v^{\mathrm{res}}]\le a_v+R\varepsilon\sqrt{2a_v}\). Averaging over the \(K-1\) alternatives and applying Cauchy--Schwarz, \(\sum_{v=2}^{K}\sqrt{a_v}\le\sqrt{(K-1)R}\),
\[
    \frac{1}{K-1}\sum_{v=2}^{K}\Reg_v
    \ge
    \varepsilon R\!\left(1-\frac{1}{K-1}-\varepsilon\sqrt{\frac{2R}{K-1}}\right).
\]
Take \(\varepsilon=(4\sqrt2)^{-1}\sqrt{(K-1)/R}\), which satisfies \(\varepsilon\le1/4\) because \(R\ge K>(K-1)/2\). Then \(\varepsilon\sqrt{2R/(K-1)}=1/4\), and \(1-1/(K-1)\ge1/2\) for \(K\ge3\), so the average regret is at least
\[
    \frac{\varepsilon R}{4}
    =
    \frac{\sqrt{(K-1)R}}{16\sqrt2}
    \ge
    \frac{\sqrt{KR}}{32},
\]
using \(K-1\ge K/2\). Some instance in the family therefore forces fair regret at least \(\sqrt{KR}/32\).

Consider now \(K=2\). Arm \(1\) remains Bernoulli\((1/2)\) under both instances, and arm \(2\) is Bernoulli\((1/2+\varepsilon)\) under \(P^{+}\) and Bernoulli\((1/2-\varepsilon)\) under \(P^{-}\); the mandatory pulls of arm \(1\) are again uninformative and comparator-matched. Let \(A=\{N_1^{\mathrm{res}}\ge R/2\}\). Under \(P^{+}\) the best arm is \(2\) with \(\Delta_1=\varepsilon\), so \(\Reg_{P^{+}}\ge\varepsilon\,\E^{+}[N_1^{\mathrm{res}}]\ge\varepsilon(R/2)\,P^{+}(A)\); under \(P^{-}\) the best arm is \(1\) with \(\Delta_2=\varepsilon\), so \(\Reg_{P^{-}}\ge\varepsilon\,\E^{-}[R-N_1^{\mathrm{res}}]\ge\varepsilon(R/2)\,P^{-}(A^c)\). By the Bretagnolle--Huber inequality,
\[
    P^{+}(A)+P^{-}(A^c)\ \ge\ \tfrac12\exp\!\bigl(-\mathrm{KL}(P^{+,\pi},P^{-,\pi})\bigr),
\]
and by the chain rule \(\mathrm{KL}(P^{+,\pi},P^{-,\pi})\le\E^{+}[N_2^{\mathrm{res}}]\,\mathrm{kl}(1/2+\varepsilon,\,1/2-\varepsilon)\le16\varepsilon^2R\) for \(\varepsilon\le1/4\), since \(\mathrm{kl}(1/2+\varepsilon,1/2-\varepsilon)=2\varepsilon\log\frac{1+2\varepsilon}{1-2\varepsilon}\le16\varepsilon^2\) in that range. Choosing \(\varepsilon=(16R)^{-1/2}\le1/4\) makes the exponent at most \(1\), so
\[
\begin{aligned}
    \Reg_{P^{+}}+\Reg_{P^{-}}
    &\ge
    \frac{\varepsilon R}{4}\,e^{-1}
    =
    \frac{e^{-1}\sqrt{R}}{16}\\
    &=
    \Omega(\sqrt{R})
    =
    \Omega(\sqrt{KR}).
\end{aligned}
\]

In every case, some instance in the family forces fair regret \(\Omega(\sqrt{KR})\) even though the overall instance contains positive mandatory exposure \(M\), which proves the claim with a universal constant.
\end{proof}

\subsection{Proof of Theorem~\ref{thm:instance-lower}}
\label{app:instance-lower-proof}

We derive the instance-dependent \(\log R\) lower bound by a change of measure against the instance that lifts arm \(i\) above \(\mu_\star\).

\begin{proof}
Fix an instance \(\nu\) of the template family, an arm \(i\ne1\) with \(\Delta_i>0\), and \(\varepsilon\in(0,1-\mu_\star)\). Let \(\nu'\) agree with \(\nu\) except that arm \(i\) has mean \(\mu_\star+\varepsilon\); then \(\nu'\) belongs to the family, and arm \(i\) is its unique best arm. Write \(\Reg_\nu(R)\) for the expected fair regret of \(\pi\) under \(\nu\) at residual budget \(R\).

All mandatory pulls are on arm \(1\), whose law is the same under \(\nu\) and \(\nu'\), so by the chain rule for KL divergence the trajectory laws satisfy
\[
    \mathrm{KL}\bigl(P_\nu,P_{\nu'}\bigr)
    =
    \E_\nu\bigl[N^{\mathrm{res}}_i\bigr]\,\mathrm{kl}(\mu_i,\mu_\star+\varepsilon).
\]
Let \(A=\{N^{\mathrm{res}}_i\ge R/2\}\). Under \(\nu\), by Lemma~\ref{lem:block-id}, the expected fair regret is at least \(\Delta_i\,\E_\nu[N^{\mathrm{res}}_i]\ge\Delta_i(R/2)P_\nu(A)\); under \(\nu'\), every residual pull of an arm other than \(i\) has gap at least \(\varepsilon\) (arm \(1\) has mean \(1/2\le\mu_\star\)), so the expected fair regret is at least \(\varepsilon\,\E_{\nu'}[R-N^{\mathrm{res}}_i]\ge\varepsilon(R/2)P_{\nu'}(A^c)\). With \(c_\varepsilon=\min(\Delta_i,\varepsilon)\), the Bretagnolle--Huber inequality gives
\[
\begin{aligned}
    \Reg_\nu(R)+\Reg_{\nu'}(R)
    &\ \ge\
    \frac{c_\varepsilon R}{2}\bigl(P_\nu(A)+P_{\nu'}(A^c)\bigr)\\
    &\ \ge\
    \frac{c_\varepsilon R}{4}\exp\Bigl(-\E_\nu\bigl[N^{\mathrm{res}}_i\bigr]\,\mathrm{kl}(\mu_i,\mu_\star+\varepsilon)\Bigr),
\end{aligned}
\]
and rearranging,
\[
    \E_\nu\bigl[N^{\mathrm{res}}_i\bigr]
    \ \ge\
    \frac{\log\bigl(c_\varepsilon R/4\bigr)-\log\bigl(\Reg_\nu(R)+\Reg_{\nu'}(R)\bigr)}{\mathrm{kl}(\mu_i,\mu_\star+\varepsilon)}.
\]
Fix \(a\in(0,1]\). Uniform fastness applied to \(\nu\) and to \(\nu'\) gives \(\Reg_\nu(R)+\Reg_{\nu'}(R)\le R^{a}\) for all sufficiently large \(R\), so
\[
    \liminf_{R\to\infty}\frac{\E_\nu[N^{\mathrm{res}}_i]}{\log R}
    \ \ge\
    \frac{1-a}{\mathrm{kl}(\mu_i,\mu_\star+\varepsilon)}.
\]
Letting \(a\downarrow0\) and then \(\varepsilon\downarrow0\), using the continuity of \(\mathrm{kl}(\mu_i,\cdot)\) on \((0,1)\), proves the per-arm claim. The regret consequence follows from Lemma~\ref{lem:block-id}: \(\Reg_\nu(R)\ge\sum_{i\ne i^\star}\Delta_i\,\E_\nu[N^{\mathrm{res}}_i]\ge\sum_{i\ne1:\,\Delta_i>0}\Delta_i\,\E_\nu[N^{\mathrm{res}}_i]\), dropping the nonnegative arm-\(1\) term when arm \(1\) is suboptimal, and summing the per-arm bounds.
\end{proof}

\subsection{Proof of Corollary~\ref{cor:bdqkl}}
\label{app:bdqkl-proof}

We embed the residual rounds as a kl-UCB\(^{++}\) game to attain the exact instance-dependent constant.

\begin{proof}
The mandatory schedule is unchanged, so feasibility is Theorem~\ref{thm:blockwise}(i). Exactly as in the proof of Theorem~\ref{thm:bdqmoss}, the residual rounds with residual-only statistics form a standard \(K\)-armed Bernoulli bandit of horizon \(R\), played here by kl-UCB\(^{++}\) with horizon parameter \(R\), and by Lemma~\ref{lem:block-id} the blockwise fair regret equals the pseudo-regret of that embedded game. kl-UCB\(^{++}\) is simultaneously minimax-optimal and asymptotically optimal for these models \citep{menard2017minimax}: the embedded pseudo-regret is at most \(C\sqrt{KR}\) in expectation for a universal constant \(C\), plus at most \(K\) for initialization, and on every fixed instance the pull counts satisfy \(\limsup_{R\to\infty}\E[N^{\mathrm{res}}_i]/\log R\le1/\mathrm{kl}(\mu_i,\mu_\star)\) for every suboptimal arm \(i\). Multiplying by \(\Delta_i\) and summing gives the displayed limsup. The instance-wise regret is therefore \(O(\log R)\), so BDQ-KL is uniformly fast whenever \(M\) grows at most polynomially in \(R\), and for every arm \(i\ne1\) of the template family the per-arm limsup meets the liminf of Theorem~\ref{thm:instance-lower}: the instance-dependent characterization is exact arm by arm. When \(R\) is not known in advance, the doubling scheme of Theorem~\ref{thm:bdqmoss} preserves the minimax bound; the exact asymptotic constant is specific to known \(R\), since restarting inflates the logarithmic term.
\end{proof}

\subsection{Proof of Theorem~\ref{thm:ogbdq}}
\label{app:ogbdq-proof}

We show OG-BDQ-UCB is exactly group-feasible and bound its regret against the slack-feasible per-block optimum by \(3KB+4H_{\max}\sqrt{2LB}\).

\begin{proof}
Write \(P_b=\{y\ge0:\sum_{i\in g}y_i\ge f_{b,g}+t\ \forall g,\ \sum_iy_i=H_b-2K\}\) for the block-\(b\) plan polytope, assumed nonempty, and \(y_b^\star\in\argmax_{y\in P_b}\langle\mu,y\rangle\).

\emph{Feasibility and accounting.} The plan \(y_b\) has group totals at least \(f_{b,g}+t\) and ceiling budget \(\sum_i\lceil y_{b,i}\rceil\le\sum_iy_{b,i}+K=H_b-K\), so Theorem~\ref{thm:groupbf} produces \(n_b\) with \(n_{b,i}\in\{\lfloor y_{b,i}\rfloor,\lceil y_{b,i}\rceil\}\), every group floor met, and \(\sum_in_{b,i}\le H_b-K\). Together with the \(K\) initialization pulls, the block therefore has \(\ell_b:=H_b-K-\sum_in_{b,i}\) leftover rounds with \(0\le\ell_b\le2K\), the upper bound because \(\sum_in_{b,i}\ge\sum_i\lfloor y_{b,i}\rfloor\ge(H_b-2K)-K\). Every group floor is met in every block deterministically.

\emph{Concentration.} Let \(\calE\) be the event of the proof of Theorem~\ref{thm:main}: for every arm \(i\) and every count \(n\le T\), the empirical mean after \(n\) observations is within \(\sqrt{2L/n}\) of \(\mu_i\); \(\Prob(\calE)\ge1-\eta\). At the block-\(b\) planning step every arm has received one initialization pull in each of blocks \(1,\dots,b\), so \(N_i\ge b\), and on \(\calE\) the clipped index satisfies \(\mu_i\le U_i\) (because \(\widehat\mu_i+\sqrt{2L/N_i}\ge\mu_i\) and \(\mu_i\le1\)) and
\[
    U_i-\mu_i\ \le\ \min\Bigl(1,\ 2\sqrt{2L/N_i}\Bigr)\ \le\ \min\Bigl(1,\ 2\sqrt{2L/b}\Bigr)\ =:\ \rho_b,
\]
uniformly over arms.

\emph{Per-block regret.} On \(\calE\), optimism and LP optimality of \(y_b\) give
\[
\begin{aligned}
    \langle\mu,y_b^\star\rangle
    &\ \le\ \langle U,y_b^\star\rangle
    \ \le\ \langle U,y_b\rangle\\
    &\ \le\ \langle\mu,y_b\rangle+\rho_b\sum_iy_{b,i}
    \ \le\ \langle\mu,y_b\rangle+\rho_bH_b .
\end{aligned}
\]
The realized mean value of block \(b\) is at least \(\langle\mu,n_b\rangle\) (initialization and leftover pulls contribute nonnegatively), and \(\langle\mu,n_b\rangle\ge\langle\mu,y_b\rangle-K\) because \(n_b\) and \(y_b\) differ by less than one in at most \(K\) coordinates and \(\mu_i\le1\). Hence the block-\(b\) contribution to \(\OPT^{\mathrm{ad}}-V\) is at most
\[
    \bigl[\langle\mu,y_b^\star\rangle+2K\mu_\star\bigr]-\bigl[\langle\mu,y_b\rangle-K\bigr]\ \le\ \rho_bH_b+3K .
\]

\emph{Summation.} Since \(\sum_{b=1}^{B}\rho_b\le2\sqrt{2L}\sum_{b=1}^{B}b^{-1/2}\le4\sqrt{2LB}\),
\[
    \OPT^{\mathrm{ad}}-V\ \le\ 3KB+H_{\max}\sum_{b=1}^B\rho_b\ \le\ 3KB+4H_{\max}\sqrt{2LB}
\]
on \(\calE\), which is the claim; with \(H_b=T/B\) the second term is \(4\sqrt2\,T\sqrt{L/B}\).
\end{proof}

\subsection{Proof of Proposition~\ref{prop:slater-gap}}
\label{app:slater-proof}

We bound the cost of the built-in slack by \(O(B)\) using a Slater-type interior point.

\begin{proof}
Write \(V^{\mathrm{frac}}_b\) for the unslacked optimum and \(V^\star_b=\max_{y\in P_b}\langle\mu,y\rangle+2K\mu_\star\). If \(H_b<\max\{t/\sigma_b,2K/\kappa_b\}\) the claim is trivial, because both values lie in \([0,H_b]\); so assume otherwise and set
\[
    \theta\ :=\ \max\Bigl\{\frac{t}{\sigma_bH_b},\ \frac{2K}{\kappa_bH_b}\Bigr\}\ \in\ (0,1].
\]
Let \(y^\ast\) attain \(V^{\mathrm{frac}}_b\) and define \(x:=(1-\theta)y^\ast+\theta z\). For every group, \(\sum_{i\in g}x_i\ge(1-\theta)f_{b,g}+\theta(f_{b,g}+\sigma_bH_b)=f_{b,g}+\theta\sigma_bH_b\ge f_{b,g}+t\); for the budget, \(\sum_ix_i\le(1-\theta)H_b+\theta(1-\kappa_b)H_b=H_b-\theta\kappa_bH_b\le H_b-2K\). Padding \(x\) with additional mass on a best arm until \(\sum_ix_i=H_b-2K\) preserves the floors and does not decrease \(\langle\mu,x\rangle\), so the padded plan lies in \(P_b\) and
\[
    V^\star_b\ \ge\ \langle\mu,x\rangle+2K\mu_\star\ \ge\ (1-\theta)\langle\mu,y^\ast\rangle\ \ge\ V^{\mathrm{frac}}_b-\theta H_b ,
\]
using \(\langle\mu,y^\ast\rangle\le H_b\). Since \(\theta H_b=\max\{t/\sigma_b,2K/\kappa_b\}\), the claim follows.
\end{proof}

\subsection{Proof of Proposition~\ref{prop:dual-ledger}}
\label{app:dual-ledger-proof}

We turn blockwise group-fair regret into an exact per-pull LP-duality ledger and instantiate its disjoint form.

\begin{proof}
For any primal-feasible \(y\) (that is, \(y\ge0\), \(\sum_{i\in g}y_i\ge f_{b,g}\) for every group, and \(\sum_iy_i=H_b\)), dual feasibility and \(\lambda_b\ge0\) give
\[
\begin{aligned}
    \langle\mu,y\rangle
    &\ \le\ \sum_i\Bigl(\omega_b-\sum_{g\ni i}\lambda_{b,g}\Bigr)y_i
    \ =\ \omega_bH_b-\sum_g\lambda_{b,g}\sum_{i\in g}y_i\\
    &\ \le\ \omega_bH_b-\sum_gf_{b,g}\lambda_{b,g},
\end{aligned}
\]
so \(V^{\mathrm{frac}}_b\le\omega_bH_b-\sum_gf_{b,g}\lambda_{b,g}\), with equality for a dual-optimal pair by strong duality, the block program being feasible by assumption and bounded. Define the credit of round \(\tau\in\calB_b\) as \(c_\tau=\sum_{g\ni A_\tau:\,d_g(\tau)>0}\lambda_{b,g}\). A pull of an arm in \(g\) while \(d_g>0\) lowers \(d_g\) by exactly one, and blockwise group fairness drives every deficit from \(f_{b,g}\) to \(0\) inside the block, so group \(g\) is credited on exactly \(f_{b,g}\) rounds and \(\sum_{\tau\in\calB_b}c_\tau=\sum_gf_{b,g}\lambda_{b,g}\). Since \(|\calB_b|=H_b\),
\[
\begin{aligned}
    V^{\mathrm{frac}}_b-\sum_{\tau\in\calB_b}\mu_{A_\tau}
    &\ \le\ \sum_{\tau\in\calB_b}\bigl(\omega_b-\mu_{A_\tau}\bigr)-\sum_{\tau\in\calB_b}c_\tau\\
    &\ =\ \sum_{\tau\in\calB_b}\Bigl[r_{A_\tau}+\sum_{g\ni A_\tau:\,d_g(\tau)=0}\lambda_{b,g}\Bigr],
\end{aligned}
\]
where the equality splits \(\omega_b-\mu_{A_\tau}=r_{A_\tau}+\sum_{g\ni A_\tau}\lambda_{b,g}\) and subtracts the credit. Nonnegativity of every bracket is dual feasibility (\(r_{A_\tau}\ge0\)) together with \(\lambda_b\ge0\). When \((\omega_b,\lambda_b)\) is dual optimal, the first display is an equality, hence so is the ledger, and complementary slackness gives \(r_i=0\) for every arm in the support of an optimal fractional plan.

For the disjoint instantiation used by Theorem~\ref{thm:disjoint}, take \(\omega_b=\mu_\star\) and \(\lambda_{b,g}=\mu_\star-\mu^\star_g\ge0\). Dual feasibility holds because \(\omega_b-\lambda_{b,g}=\mu^\star_g\ge\mu_i\) for \(i\in g\) and \(\omega_b=\mu_\star\ge\mu_i\) for ungrouped arms; the dual value is \(H_b\mu_\star-\sum_gf_{b,g}(\mu_\star-\mu^\star_g)=V^{\mathrm{frac}}_b\), so the pair is optimal; and the reduced costs are \(r_i=\mu^\star_g-\mu_i=\widetilde\Delta_i\) for \(i\in g\) and \(r_i=\Delta_i\) for ungrouped arms, recovering exactly the two columns bounded in the proof of Theorem~\ref{thm:disjoint}.
\end{proof}

\subsection{Proof of Proposition~\ref{prop:greedy-fails}}
\label{app:greedy-fails-proof}

We construct a two-group overlap instance on which any within-group index rule pays \(\Omega(T)\), because the optimal cover accepts one arm's lower mean to satisfy two floors at once.

\begin{proof}
Four arms \(a,b,c,s\) with Bernoulli means \(0.70,0.70,0.65,0.90\); groups \(g_1=\{a,c\}\) and \(g_2=\{b,c\}\), so the arm degree is \(t=2\); \(B\) blocks of length \(H\) divisible by \(3\), each with floors \(f_{b,g_1}=f_{b,g_2}=H/3\). Because each group reads the full mass of its members, placing \(H/3\) on \(c\) satisfies both floors simultaneously, and the fractional optimum puts the remaining \(2H/3\) on \(s\):
\[
    V^{\mathrm{frac}}_b\ =\ \bigl(0.65+2\cdot0.90\bigr)\tfrac{H}{3}\ =\ 2.45\,\tfrac{H}{3}
    \qquad\text{per block.}
\]

Consider any rule of the stated class, with indices \(U_i=\min\{1,\widehat\mu_i+\sqrt{2L/N_i}\}\), and work on the concentration event \(\calE\) of the proof of Theorem~\ref{thm:main}. A covering selection of \(c\) in either group with \(n\ge1\) prior samples forces \(U_c\ge U_a\ge\mu_a\) (or the same with \(b\)), hence \(0.65+2\sqrt{2L/n}\ge0.70\) and \(n\le3200L\); a surplus selection of \(c\) forces \(U_c\ge U_s\ge0.90\), hence \(n\le128L\). Since every selection increments \(N_c\), arm \(c\) is pulled at most \(1+3200L\) times over the entire horizon. Likewise a surplus selection of \(a\) or \(b\) forces \(n\le200L\) (gap \(0.20\) to \(s\)), so each incurs at most \(1+200L\) surplus pulls. Outside these at most \(3+3600L\) exceptional pulls, every block allocates exactly \(H/3\) covering pulls of \(a\) for \(g_1\), \(H/3\) of \(b\) for \(g_2\), and \(H/3\) surplus pulls of \(s\), for a block value of \((0.70+0.70+0.90)H/3=2.30\,H/3\) and a per-block gap of \(0.15\,H/3=H/20\). Every block's regret against its fractional optimum is nonnegative, each exceptional pull distorts value by at most one, and feasibility is exact throughout since each floor is served by exactly \(H/3\) covering pulls; hence on \(\calE\) the total regret is at least \(T/20-(3+3600L)\), which is the claim with a universal constant. The failure is structural rather than statistical: the rule learns every mean correctly and still pays \(\Omega(T)\), because the optimal cover accepts \(c\)'s lower mean in exchange for satisfying two floors with one budget, and no within-group index comparison represents that exchange.
\end{proof}

\subsection{Proof of Proposition~\ref{prop:pbdq-feas}}
\label{app:pbdq-feas-proof}

We show P-BDQ-UCB is pathwise feasible under the initial cover-slack condition, via a one-step slack bound and the terminal cover-rounding guard.

\begin{proof}
Fix a block with \(H_b-\mathrm{mc}(f_b)>2K\), so \(S(0)>2K\) and the guard does not fire at the block start. Every pull lowers \(r\) by one and \(\mathrm{mc}\) by at most one. For the second claim, monotonicity gives \(\mathrm{mc}(d^+(A))\le\mathrm{mc}(d)\); conversely, let \(\pi_d\) be an optimal solution of the covering dual
\[
    \mathrm{mc}(d)=\max\Bigl\{\langle\pi,d\rangle:\ \pi\ge0,\ \textstyle\sum_{g\ni i}\pi_g\le1\ \forall i\Bigr\},
\]
whose feasible region does not depend on the demands. Then
\[
\begin{aligned}
    \mathrm{mc}(d)-\mathrm{mc}(d^+(A))
    &\ \le\ \langle\pi_d,\,d-d^+(A)\rangle
    \ =\!\!\sum_{g\ni A:\,d_g\ge1}\!\!\pi_{d,g}\\
    &\ \le\ \sum_{g\ni A}\pi_{d,g}\ \le\ 1
\end{aligned}
\]
by the dual constraint at the pulled arm, using that \(\pi_d\) remains feasible for the demands \(d^+(A)\). Hence \(S(\tau+1)-S(\tau)\in[-1,0]\) in the sampled phase. Since \(S(\tau)\le r(\tau)\) always, the guard fires no later than the round with \(r(\tau)=2K\); firing cannot occur at the block start, so the firing state has a pre-guard predecessor with slack above \(2K\), whence \(S(\tau_c)>2K-1\ge K\). Rounding an optimal fractional cover \(w(\tau_c)\) up coordinatewise raises at most \(K\) coordinates, so
\[
\begin{aligned}
    \sum_i\lceil w_i(\tau_c)\rceil
    &\ \le\ \mathrm{mc}(d(\tau_c))+K
    \ =\ r(\tau_c)-S(\tau_c)+K\\
    &\ \le\ r(\tau_c)-K+1
    \ \le\ r(\tau_c).
\end{aligned}
\]
The rounded cover therefore fits in the remaining budget; executing it drives every remaining deficit to zero, and leftover rounds, spent on the index argmax, cannot violate one-sided floor constraints. Every group floor of every block is thus satisfied, pathwise and deterministically.
\end{proof}

\subsection{Proof of Proposition~\ref{prop:pbdq-disjoint}}
\label{app:pbdq-disjoint-proof}

We verify the two descent inequalities automatically on disjoint systems, where the covering value is linear.

\begin{proof}
Disjointness makes the covering value linear: no arm serves two groups, so \(\mathrm{mc}(d)=\sum_g d_g\). At a pre-guard state, primal feasibility of \(\widehat y\) gives \(\Prob(A\in g)=\sum_{i\in g}\widehat y_i/r\ge d_g/r\) for every deficient group, so
\[
\begin{aligned}
    \E\bigl[\mathrm{mc}(d^+(A))\bigr]
    &=\sum_{g:\,d_g\ge1}\bigl(d_g-\Prob(A\in g)\bigr)\\
    &\le\Bigl(1-\frac1r\Bigr)\sum_g d_g
    =\frac{r-1}{r}\,\mathrm{mc}(d),
\end{aligned}
\]
which is \eqref{eq:cert-cover}, with no event required.

For \eqref{eq:cert-value}, write \(\mu^\star_g=\max_{i\in g}\mu_i\) and \(\lambda_g=\mu_\star-\mu^\star_g\ge0\). On the pre-guard region \(r>\mathrm{mc}(d)+2K\ge\sum_gd_g\), the residual value is linear:
\[
    \Psi_\mu(r,d)=\mu_\star r-\sum_g\lambda_g d_g ,
\]
because for any feasible \(y\) with group totals \(Y_g\ge d_g\) we have \(\langle\mu,y\rangle\le\sum_g\mu^\star_gY_g+\mu_\star(r-\sum_gY_g)\), the coefficient of \(Y_g\) is \(\mu^\star_g-\mu_\star\le0\), and the bound is attained at \(Y_g=d_g\). Both states \((r,d)\) and \((r-1,d^+(A))\) lie in this region, since \(\sum_gd^+(A)_g\ge\sum_gd_g-1\). Hence, with \(p=\widehat y/r\) and \(\E[d^+(A)_g]=d_g-\Prob(A\in g)\) for deficient groups, linearity gives
\[
\begin{aligned}
    \E\bigl[\Psi_\mu(r-1,d^+(A))+\mu_A\bigr]-\Psi_\mu(r,d)
    &=-\sum_i p_i\,\check r_i,\\
    \check r_i&=\mu_\star-\mu_i-\!\!\sum_{g\ni i:\,d_g\ge1}\!\!\lambda_g .
\end{aligned}
\]
It therefore suffices that \(\check r_i\le\rho_i\) for every support arm of \(\widehat y\). The optimistic program has the explicit optimal dual \(\omega^U=\max_jU_j\) and \(\lambda^U_g=\omega^U-\max_{j\in g}U_j\ge0\): it is feasible because \(\omega^U-\lambda^U_g=\max_{j\in g}U_j\ge U_i\) for \(i\in g\) and \(\omega^U\ge U_i\) for ungrouped arms, and its objective \(\omega^Ur-\sum_g\lambda^U_gd_g=\sum_gd_g\max_{j\in g}U_j+(r-\sum_gd_g)\,\omega^U\) matches the primal optimum, so the pair is optimal. By complementary slackness, every support arm of every optimal plan has zero reduced cost under this dual: a grouped support arm \(i\in g\) satisfies \(U_i=\max_{j\in g}U_j\) when \(\lambda^U_g>0\), and \(U_i=\omega^U\) otherwise; an ungrouped support arm satisfies \(U_i=\omega^U\). On \(\calE_L\), a within-group maximizer \(i\in g\) satisfies \(\mu_i+\rho_i\ge U_i\ge U_j\ge\mu_j\) for the group's best true arm \(j\), so \(\mu^\star_g-\mu_i\le\rho_i\); a global maximizer satisfies \(\mu_\star-\mu_i\le\rho_i\) the same way. In the first case \(\check r_i=\mu^\star_g-\mu_i\) if \(g\) is deficient and \(\check r_i=\mu_\star-\mu_i\) otherwise; in every case \(\check r_i\le\max\{\mu^\star_g-\mu_i,\mu_\star-\mu_i\}\le\rho_i\), because a within-group maximizer also satisfies \(\mu_\star-\mu_i\le\rho_i\) whenever it is a global maximizer, and \(\mu^\star_g-\mu_i\le\mu_\star-\mu_i\) always. This gives \eqref{eq:cert-value} on \(\calE_L\).

The final claim of the proposition is immediate: \(\sum_gf_{b,g}\le(1-\sigma)H_b\) is the initial-slack clause because \(\mathrm{mc}(f_b)=\sum_gf_{b,g}\), and \(\sigma H_b>2K\) is the block-start guard condition of Theorem~\ref{thm:pbdq}.
\end{proof}

\subsection{Proof of Theorem~\ref{thm:pbdq}}
\label{app:pbdq-proof}

We prove the conditional \(\widetilde O(\sqrt{KT})\) guarantee for P-BDQ-UCB by combining slack concentration with the true-value descent inequality.

\begin{proof}
Fix a block \(b\). Write \(r(\tau)\) and \(d(\tau)\) for the remaining budget and integer deficits before round \(\tau\), \(\mathrm{mc}\) and \(S(\tau)=r(\tau)-\mathrm{mc}(d(\tau))\) as in the main text, and
\[
\begin{aligned}
    \Psi(\tau)=\Psi_\mu(r(\tau),d(\tau))
    =\max\Bigl\{\langle\mu,y\rangle:\ y\ge0,\
    &\textstyle\sum_{i\in g}y_i\ge d_g(\tau)\ \forall g,\\
    &\sum_i y_i=r(\tau)\Bigr\},
\end{aligned}
\]
so that \(\Psi(0)=V^{\mathrm{frac}}_b\). The margin clause gives \(S(0)\ge\sigma H_b\), and with \(\sigma H_b>2K\) the hypothesis \(H_b-\mathrm{mc}(f_b)>2K\) of Proposition~\ref{prop:pbdq-feas} holds. Feasibility, the one-step bound \(S(\tau+1)-S(\tau)\in[-1,0]\) during the sampled phase, and the deterministic guard fit are therefore available throughout; let \(\tau_c\) be the guard's firing round, so \(S(\tau_c)\in(2K-1,2K]\) and the residual polytope is nonempty at every sampled round.

\emph{Stopped-process convention.} The descent condition asserts \eqref{eq:cert-cover}--\eqref{eq:cert-value} at pre-guard states reached on \(\calE_L\). Formally, every conditional-expectation step below is applied to the process stopped at the first pre-guard time whose state violates either inequality; on \(\calE_L\) that time is not before \(\tau_c\), and every path outside \(\calE_L\) is charged to the failure budget at the end, at cost at most \(T\).

\emph{Slack concentration.}
Before \(\tau_c\), let \(p_i(\tau)=\widehat y_i(\tau)/r(\tau)\). By the cover-contraction inequality~\eqref{eq:cert-cover},
\[
    \E[\mathrm{mc}(d(\tau+1))\mid\mathcal F_\tau]
    \le
    \frac{r(\tau)-1}{r(\tau)}\mathrm{mc}(d(\tau)),
\]
so \(M(\tau)=S(\tau)/r(\tau)\) is a submartingale before the guard: \(\E[S(\tau+1)]\ge r(\tau)-1-\frac{r(\tau)-1}{r(\tau)}\mathrm{mc}(d(\tau))=\frac{r(\tau)-1}{r(\tau)}S(\tau)\). Since \(|S(\tau+1)-S(\tau)|\le1\) by Proposition~\ref{prop:pbdq-feas} and \(0\le S(\tau)\le r(\tau)\), its one-step increments satisfy \(|M(\tau+1)-M(\tau)|\le2/(r(\tau)-1)\) whenever \(r(\tau)>1\). With \(M(0)\ge\sigma\), Azuma's inequality and
\(\sum_{s\ge r}1/(s-1)^2\le2/r\) imply that for every pre-guard state with
\(r(\tau)\ge r^\star=\lceil64L/\sigma^2\rceil\),
\[
    \Prob(M(\tau)\le\sigma/2)
    \le
    \exp(-\sigma^2 r(\tau)/64)
    \le
    \frac{\eta}{2KT}.
\]
A union bound over all rounds and blocks gives, with probability at least \(1-\eta\), \(M(\tau)>\sigma/2\) for every such pre-guard state. On this event, whenever additionally
\(r(\tau)\ge(4K+2)/\sigma\), we have \(S(\tau)>2K\), so the guard cannot yet fire. Hence
\[
    r(\tau_c)
    \le
    \frac{64L}{\sigma^2}+\frac{4K+2}{\sigma}+1.
\]

\emph{Regret in the sampled phase.}
On \(\calE_L\), the clipped indices satisfy \(\mu_i\le U_i\le\mu_i+\rho_i\) with
\(\rho_i=\min\{1,2\sqrt{2L/(N_i\vee1)}\}\), and the true-value descent inequality~\eqref{eq:cert-value} gives, for each pre-guard sampled round,
\[
    \E[\Psi(\tau)-\Psi(\tau+1)-\mu_{A_\tau}\mid\mathcal F_\tau]
    \le
    \E[\rho_{A_\tau}\mid\mathcal F_\tau].
\]
Telescoping over \(\tau<\tau_c\), using \(\Psi(\tau_c)\le r(\tau_c)\), and noting that rewards in the committed tail are nonnegative, yields the pathwise inequality
\[
    V^{\mathrm{frac}}_b-\sum_{\tau\in\calB_b}\mu_{A_\tau}
    \le
    \sum_{\tau<\tau_c}\bigl(\Psi(\tau)-\Psi(\tau+1)-\mu_{A_\tau}\bigr)+r(\tau_c)
\]
for the block. Taking expectations with the stopped-process convention, summing over blocks, and using the tail bound above, the only remaining term is the sum of confidence radii over actually pulled arms. For each arm, \(\sum_{n=1}^{N_i(T)}\min\{1,2\sqrt{2L/n}\}\le1+4\sqrt{2LN_i(T)}\), so by Cauchy--Schwarz,
\[
    \sum_{\tau=1}^T\rho_{A_\tau}
    \le
    K+4\sqrt{2L}\sum_i\sqrt{N_i(T)}
    \le
    K+4\sqrt{2KTL}.
\]
The reward-concentration event \(\calE_L\) fails with probability at most \(\eta\), the slack-concentration event fails with probability at most \(\eta\), and the regret on either failure is at most \(T\). Thus
\[
\begin{aligned}
    \OPT^{\mathrm{frac}}-\E[V]
    &\le
    K+4\sqrt{2KTL}+2\eta T\\
    &\quad+B\left(\frac{64L}{\sigma^2}+\frac{4K+2}{\sigma}+1\right),
\end{aligned}
\]
which is the claimed bound after absorbing constants into \(C\).
\end{proof}

\subsection{Proof of the Static Best-Fair Comparator}
\label{app:comparator-static}

\begin{proposition}
\label{prop:static-comparator}
Let \(m\in\{0,1,\dots,\lfloor T/K\rfloor\}\). Among all integer vectors \(n=(n_1,\dots,n_K)\) satisfying \(n_i\ge m\) and \(\sum_{i=1}^K n_i=T\), the reward \(\sum_{i=1}^K n_i\mu_i\) is maximized by assigning \(m\) pulls to every arm \(i\neq i^\star\) and assigning all remaining pulls to \(i^\star\).
\end{proposition}

\begin{proof}
Let \(n\) be any feasible fair allocation. Since \(n_i\ge m\), define \(r_i=n_i-m\ge0\), so that \(\sum_{i=1}^K r_i=T-Km\). The expected reward of \(n\) is
\[
    \sum_{i=1}^K n_i\mu_i
    =
    m\sum_{i=1}^K \mu_i
    +
    \sum_{i=1}^K r_i\mu_i.
\]
The first term is fixed across all feasible allocations, so maximizing reward is equivalent to maximizing \(\sum_i r_i\mu_i\) subject to \(r_i\ge0\) and \(\sum_i r_i=T-Km\). This is maximized by placing all residual mass on an arm with largest mean, namely \(i^\star\): \(r_{i^\star}=T-Km\) and \(r_i=0\) for \(i\neq i^\star\), equivalently \(n_{i^\star}=T-(K-1)m\) and \(n_i=m\) for \(i\neq i^\star\).
\end{proof}

\subsection{Proof of the Blockwise Comparator}
\label{app:comparator-block}

\begin{proposition}
\label{prop:block-comparator}
For each block \(b\), among all integer allocations satisfying \(n_{b,i}\ge m_{b,i}\) and \(\sum_i n_{b,i}=H_b\), reward is maximized by assigning \(m_{b,i}\) pulls to every arm and all residual pulls to a best arm.
\end{proposition}

\begin{proof}
For a fixed block, write \(n_{b,i}=m_{b,i}+r_{b,i}\) with \(r_{b,i}\ge0\). The block reward is \(\sum_i m_{b,i}\mu_i + \sum_i r_{b,i}\mu_i\). The first term is fixed by the block floors, and the residual mass \(\sum_i r_{b,i}=H_b-\sum_i m_{b,i}\) is maximized by placing it on a best arm. Summing over blocks gives \(\OPT_{\mathbf m}(T)\).
\end{proof}

\subsection{Proof of Corollary~\ref{cor:static-special}}
\label{app:static-special}

\begin{corollary}[Static case as a special case]
\label{cor:static-special}
Taking \(B=1\), \(H_1=T\), and \(m_{1,i}=m\) for every arm reduces Theorem~\ref{thm:blockwise} to Theorem~\ref{thm:main}, up to replacing \(R=T-Km\) by the looser bound \(T\).
\end{corollary}

\begin{proof}
With \(B=1\), \(H_1=T\), and \(m_{1,i}=m\) for every \(i\), the blockwise constraint collapses to \(N_i(T)\ge m\), and BDQ-UCB executes DQ-UCB exactly. The blockwise comparator (Lemma~\ref{lem:block-id}) becomes \(\OPT_m(T)\) (Lemma~\ref{lem:static-id}), and the residual budget is \(R=T-Km\). Substituting into Theorem~\ref{thm:blockwise}(ii)--(iii) yields the bounds of Theorem~\ref{thm:main}(ii)--(iii), loosened by replacing \(R\) with \(T\). The expected-regret bound follows by setting \(\eta=1/T\).
\end{proof}

\subsection{Proof of Proposition~\ref{prop:separation}}
\label{app:separation-proof}

\begin{proposition}[Separation from block-independent floors]
\label{prop:separation}
Fix \(\alpha\in(0,1/2]\), an even number of blocks \(B\), and a block length \(H\) such that \(\alpha H\) is an integer. Consider two arms with \(\mu_1=1\) and \(\mu_2=1-\Delta\), where \(\Delta>0\). In odd blocks, require \(m_{b,2}=\alpha H\) and \(m_{b,1}=0\); in even blocks, require \(m_{b,1}=m_{b,2}=0\). Then: (i) no single final-horizon global floor can encode these blockwise constraints, since an allocation may satisfy the exact aggregate count of arm 2 while violating every odd-block requirement; (ii) any block-independent per-block lower bound \(g_2H\) for arm 2 either violates the odd-block requirement if \(g_2<\alpha\), or, if \(g_2\ge\alpha\), incurs at least \(\tfrac{\alpha\Delta}{2}T\) additional blockwise regret relative to the best blockwise-fair comparator, where \(T=BH\); and (iii) BDQ-UCB satisfies all blockwise constraints exactly and has regret \(O(\sqrt{KR\log(KT)})\).
\end{proposition}

\begin{proof}
The aggregate arm-2 requirement is \((B/2)\alpha H=\alpha T/2\). A global final-count constraint enforcing this total cannot distinguish odd-block placements from even-block placements, proving (i). For (ii), a block-independent surrogate with per-block lower bound \(g_2H\) either under-enforces odd-block requirements (if \(g_2<\alpha\)) or forces at least \(\alpha H\) pulls of the suboptimal arm in each of the \(B/2\) even blocks (if \(g_2\ge\alpha\)), incurring at least \((B/2)\alpha H\Delta=\alpha\Delta T/2\) additional regret. Part (iii) is Theorem~\ref{thm:blockwise}.
\end{proof}

\end{document}